\documentclass[lettersize,journal]{IEEEtran}
\usepackage{algorithmic}
\usepackage{algorithm}
\usepackage{array}
\usepackage[caption=true,font=normalsize,labelfont=sf,textfont=sf]{subfig}
\usepackage{textcomp}
\usepackage{stfloats}
\usepackage{url}
\usepackage{verbatim}
\usepackage{graphicx}
\usepackage{cite}
\usepackage{afterpage}

\usepackage{multirow}
\usepackage[table,xcdraw]{xcolor}

\usepackage{comment}

\usepackage{amsmath,amsfonts,amssymb,mathtools,url,enumerate,verbatim}
\usepackage{amsthm}
\usepackage[german,french,british]{babel}
\usepackage{ucs}
\usepackage[utf8]{inputenc}
\usepackage[T1]{fontenc}
\usepackage{pgf,tikz}
\usepackage{float}
\usepackage{mathrsfs}
\usepackage{mathtools}
\usepackage{multirow}
\usepackage{yfonts}
\usetikzlibrary{arrows}
\usepackage{longtable}
\usepackage{cleveref}
    \newcommand\newdot{{\kern.8pt\cdot\kern.8pt}}
\newcommand\nbull{{\kern.8pt\raise1.5pt\hbox{\small\bf .}\kern.8pt}}
\newcommand\1{\hbox{\kern.375em\vrule height1.57ex depth-.1ex
		width.05em\kern-.375em \rm 1}}

\newcommand\E{\mathbb{E}}

\newcommand\R{\mathbb{R}}
\renewcommand\P{\mathbb{P}}

\newcommand\rbb{\mathbb{R}}

\DeclareMathOperator{\gbay}{G^{\text{Bayes}}}

\DeclareMathOperator{\fix}{u} %  Possible ratings
\DeclareMathOperator*{\op}{Op} % Filter ==> Convolutional operation (matrix)
\DeclareMathOperator{\fixx}{\tilde{u}} %Max Span of Ratings
\DeclareMathOperator{\ktwo}{K_2} % (i,j) level embedding before final conv
\DeclareMathOperator{\elc}{\ell_c} %Classification Loss Function
\DeclareMathOperator{\elr}{\ell_r} %Regression Loss Function 
\DeclareMathOperator{\Elc}{\mathcal{L}} % Classification Loss (aggregated)
\DeclareMathOperator{\rmseemp}{\mathcal{P}_{Empirical}} % Empirical RMSE
\DeclareMathOperator{\rmsepop}{\mathcal{P}_{Population}} % Population RMSE
\DeclareMathOperator{\kap}{E_k} %Set of one hot encodings of ratings plus non observed
\DeclareMathOperator{\nn}{Sum}

\DeclareMathOperator{\con}{\mathcal{K}}
\DeclareMathOperator{\imploss}{\mathcal{L}_{\text{impl}}}
\DeclareMathOperator{\uS}{\underline{S}}

\DeclareMathOperator{\rad}{\mathfrak{R}}

\DeclareMathOperator{\Fr}{Fr}

\makeatletter
\newcommand*{\addFileDependency}[1]{% argument=file name and extension
	\typeout{(#1)}% latexmk will find this if $recorder=0
	\@addtofilelist{#1}
	\IfFileExists{#1}{}{\typeout{No file #1.}}
}\makeatother

\usepackage{xr}

\newtheorem{theorem}{Theorem}
\newtheorem{lemma}[theorem]{Lemma}
\newtheorem{proposition}[theorem]{Proposition}
\newtheorem{corollary}[theorem]{Corollary}
\theoremstyle{definition}

\theoremstyle{definition}

\newtheorem{remark}{Remark}
\usepackage{mathabx}
\makeatletter
\newcommand{\thickhline}{%
	\noalign{\ifnum0=`}\fi\hrule height 1pt
	\futurelet\reserved@a\@xhline}
\newcolumntype{"}{@{\hskip\tabcolsep\vrule width 1pt\hskip\tabcolsep}}
\makeatother

\DeclareMathOperator{\Gall}{\mathcal{G}}

\DeclareMathOperator{\down}{down}
\DeclareMathOperator{\upp}{up}

\title{Conv4Rec: A 1-by-1 Convolutional AutoEncoder for User Profiling through Joint Analysis of Implicit and Explicit Feedbacks}

\author{Antoine Ledent, Petr Kasalický, Rodrigo Alves, and Hady W. Lauw}

\begin{document}
	
	\maketitle
	
	\begin{abstract}
		We introduce a new convolutional AutoEncoder architecture for user modelling and recommendation tasks with several improvements over the state of the art. Firstly, our model has the flexibility to learn a set of associations and combinations between different interaction types in a \textit{way that carries over to each user and item}.  Secondly, our model is able to learn jointly from both the explicit ratings and the implicit information in the sampling pattern (which we refer to as `implicit feedback'). It can also make separate predictions for the probability of consuming content and the likelihood of granting it a high rating if observed.  This not only allows the model to make predictions for both the implicit and explicit feedback, but also increases the informativeness of the predictions:  in particular, our model can identify items which users would not have been likely to consume naturally, but would be likely to enjoy if exposed to them.  Finally, we provide several generalization bounds for our model, which to the best of our knowledge, are among the first generalization bounds for auto-encoders in a Recommender Systems setting; we also show that optimizing our loss function guarantees the recovery of the exact sampling distribution over interactions up to a small error in total variation. In experiments on several real-life datasets, we achieve state-of-the-art performance on both the implicit and explicit feedback prediction tasks despite relying on a single model for both, and benefiting from additional interpretability in the form of individual predictions for the probabilities of each possible rating.
	\end{abstract}

	\begingroup
	\renewcommand\thefootnote{}\footnotetext{This preprint is accepted for publication at Transactions on Neural Networks and Learning Systems (TNNLS).}
	\addtocounter{footnote}{-1}
	\endgroup
	
	\section{Introduction}
	
	Recommender Systems (RSs) rely on previous 	user-item interaction to create a mathematical model of their preferences on which they base future recommendations \cite{koren09,softimpute,he2017neural,sedhain2015autorec}.  Typically, one organises the available information in a matrix of observations, of dimension $m\times n$, representing the known interactions between the $m$ \textit{users} %(e.g. persons) 
		and the $n$ \textit{items} (e.g. movies, books, etc.).  The $i,j$ the entry of the matrix records a possible interaction between user $i$ and item $j$. For instance, the entry $R_{i,j}$ could be a number in $\{1,2,\ldots,5\}$ representing the number of 'stars' rating ascribed to item $j$ by user $i$, or simply be a binary variable indicating whether or not user $i$ has consumed (watched, viewed, bought, rated) the content $j$.

Accordingly, traditional methods in recommender systems are generally divided into to two broad approaches depending on whether (1) they focus on extracting information from numerical \textit{ratings} given by users to the items they have consumed (the so-called \emph{explicit feedback}) or (2) the decisions made by the users to consume certain content rather than others in the first place (\emph{implicit feedback}) \cite{hu2008collaborative,vanvcura2022scalable,sedhain2015autorec}. More concretely, the former approach views the values of the observed entries $R_{i,j}$ (for $(i,j)\in\Omega$  where $\Omega$ is the set of observed entries) as the outputs in a supervised learning problem where the inputs are the (user, item) pairs $(i,j)$, whereas the second approach attempts to extract structural information from the pattern of interactions. In many contexts, this corresponds to the matrix $\widebar{R}=1_{\Omega}$, which simply contains the combined IDs of all the previous user-item interactions, i.e. the matrix with $\widebar{R}_{i,j}=1$ if user $i$ has consumed (e.g., purchased, liked, clicked) item $j$ and $\widebar{R}_{i,j}=0$ otherwise~\cite{hu2008collaborative}. In other words, where we define the inputs $x$ as the entries $(i,j)$ and the outputs $y$ are the ratings $R_{i,j}$, the explicit feedback corresponds to the conditional distribution of $y$ given $x$, and the implicit feedback corresponds to the sampling distribution $\mathcal{D}$ from which $x$ is drawn.
	
	\begin{figure*}
		\centering
		\includegraphics[width=.87\textwidth]{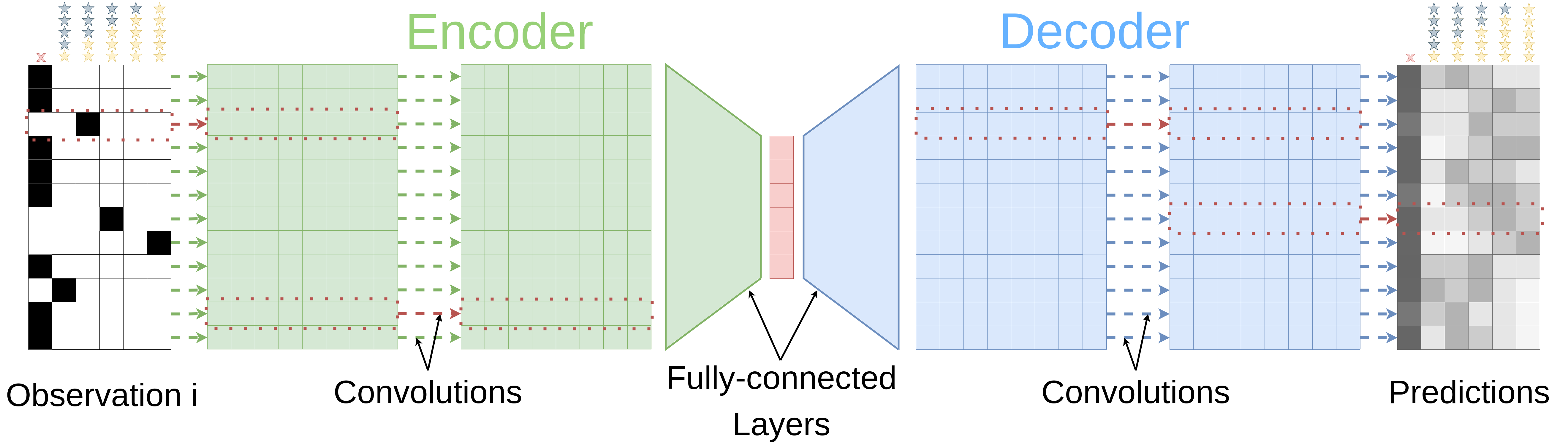}
		\caption{Illustration of our architecture. The input and output represented correspond to a single user, but the weights are shared between users. The first column of the preprocessed representation corresponds to the "unobserved" category.}
		\label{MainPicture}
	\end{figure*}
	
	In recent years, it has become increasingly evident that  the information contained in the implicit feedback is relevant to the explicit feedback prediction task, and vice versa~\cite{corating,10.1145/3341981.3344225}. This is intuitively clear since implicit feedback and explicit feedback are both closely related to the personal tastes of each user.

		Arguably, it is one of the best explanations for the recent success of Graph Neural Network (GNN) methods~\cite{he2020lightgcn,wang2019neural,Zhang2020Inductive}, which \textit{indirectly} incorporate both implicit and explicit feedback into the learning procedure.  \textcolor{black}{For instance, in~\cite{Zhang2020Inductive}, the propagation of embeddings through the interaction graph implicitly uses the information contained in the implicit feedback, whilst the loss function signal arises from the explicit feedback. However, those methods usually only evaluate performance on either the explicit or the implicit feedback. }

		 \textcolor{black}{ Recently, several methods have attempted to \textit{explicitly} model both the implicit and the explicit feedbacks in combined approaches~\cite{mandal2018explicit,WADMF,corating}. For instance, ~\cite{WADMF} assume access to a large amount of implicit feedback in addition to a more moderate amount of explicit feedback. To the best of our knowledge~\cite{WADMF} and~\cite{corating}  are the only existing works which are able to \textit{evaluate their method} on both implicit and explicit feedback prediction tasks. Thus, there remains a relative dearth of combined methods which incorporate both types of feedbacks during both training and prediction. Furthermore, in both~\cite{WADMF} and~\cite{corating}, the model's output is a single $m\times n$ matrix $\Phi$ which is used both to predict explicit feedback (interpreting the entry $\Phi_{i,j}$ as the predicted rating) and rank items to perform the retrieval task. Thus, both models implicitly assume that highly ranked items are more likely to be interacted with, and are unable model the likelihood of interaction separately from the likelyhood of a high rating. We believe this is a significant limitation: indeed, identifying \textit{serendipitious items} which have a low probability of interaction but a high expected rating conditional on interatction is a promising direction to improve exploration ability in online recommender systems.  }

	In this paper, we propose a convolutional autoencoder architecture  \textcolor{black}{with two key characteristics. Firstly, our method utilises the information from both the implicit and the explicit feedback both at the data representations stage with the encoder and through the loss function applied to the decoder's output. Secondly, our method is able to simultaneously make separate predictions for both the implicit and the explicit feedback: the output of our model is a \textit{distributional prediction} which provides both an estimate of the total probability of interaction and the conditional distribution over ratings assuming an interaction takes place. This differs from popular works with distributional outputs such as Probabilistic Matrix Factorization~\cite{PMF} or Bernouilli Matrix Factorization~\cite{BeMF}, which only estimate the distribution over ratings (conditional on interaction). }
	
	 \textcolor{black}{Concretely, our} autoencoder takes as input the row of all known and unknown interactions of a given user $i$ with all items and tries to reconstruct the known ratings \textcolor{black}{or lack thereof.} %in particular, the model must \textit{both} provide an estimate of the likelihood of an item being interacted with \textit{and} estimate the most likely ratings}.

	 \textcolor{black}{To allow our model to freely learn any set of relationships or associations between the ratings $1,2,3,4,5$  in a way that carries over through all the interactions, we rely on weight sharing in the form of convolutions of window size $1$ at the first encoder layers and the last decoder layers. Our model is trained with the cross entropy loss as a six-class classification problem with classes "1", "2", "3", "4", "5" and "no interaction".} As we demonstrate in our ``Generalization bounds" Section and in Appendix~\ref{subsec:B}, the use of this specific loss function  \textcolor{black}{guarantees that as the number of observations tends to infinity, our model's excess risk approaches zero in the case where the ground truth is realizable. This is true both in terms of the square loss and in terms of the L1 distance between the probability distribution over interactions and its estimate.}   % In particular, we demonstrate that successfully minimizing our loss function guarantees a close recovery of the sampling distribution over entries up to a small total variation distance.
%	 \textcolor{black}{The main novelty in our idea is two fold: on the one hand, we use convolutions of window size $1$ to allow the model to learn various associations between the different ratings $\{1,2,3,4,5\}$ (together with the presence or absence of the rating) in a way that carries over to all users and items through weight sharing; and on the other hand, we train the decoder with a cross-entropy loss which enforces the complete recovery of each possible rating or lack thereof. As we demonstrate in our ``Generalization bounds" Section and in Appendix~\ref{subsec:B}, the use of this specific loss function enables us to obtain generalization bounds which hold both for implicit and for explicit feedback. In particular, we demonstrate that successfully minimizing our loss function guarantees a close recovery of the sampling distribution over entries up to a small total variation distance.}
	
	Our architecture is illustrated in Figure~\ref{MainPicture}. The input is an $n\times 6$ matrix where $n$ is the number of items. The channel direction of length $6$ corresponds to a one-hot encoding of the ratings (on a scale from $1$ to $5$) with an extra slot for unobserved entries. For instance, if item $j$ was not rated by user $i$, the $j$th row of the $i$th datapoint will be $(1,0,0,0,0,0)^\top$. If item $j'$ was given a rating of $2$ (resp. $5$) by user $i$, the $j'$th row of the $i$th datapoint will be $(0,0,1,0,0,0)^\top$ (resp. $(0,0,0,0,0,1)^\top$). Our first few layers are 1D convolutions of window size $1$, which means that each convolutional filter has shape $(1,6)$ and contains only $6$ weights. This allows the convolution layers to learn, in intermediate layers, associations such as ``ranking equals $4$ or ranking equals $5$'' (with a filter with weights such as $(0,0,0,0,1,1)$) or ``either not observed or ranked $1$'' (with a filter with weights such as $(1,1,0,0,0,0)$) using only a single set of $6$ weights for each association, in a way that carries over to all items (and of course, users). \textcolor{black}{It is also possible for the model to learn a rescaled form of the ratings: for instance, the filter $(0,1,2,3,4.5,4.5)$ learns a compact representation of the rating which blurs all $4$ or $5$ ratings together as $4.5$. }
	
	We train the model as an autoencoder to enforce recovery of the original one-hot encoded input: our loss function is the cross entropy loss aggregated over all items. 
	After training, our model output offers insights into a specific (user, item) pairing denoted as $(i,j)$. \textcolor{black}{Indeed, for each (user, item) pair, our model outputs a probability for each possible rating ($\{1,2,\ldots,5\}$), as well as a probability that no interaction is recorded at all. We use the same architecture for the decoder as for the encoder. In particular, the model can easily represent certain natural functions without too many parameters. For instance, if the ground truth matrix is rank $r$ after replacing the ratings $\{1,2,3,4,5\}$ by $\{1,1.5,3,5,10\}$, the model can represent this using only $6+r(m+n)$ parameters each for the decoder and encoder (where $m,n$ are the number of users and items respectively). Thus, the simplicity of the model contributes not only to its empirical success, but to its favorable theoretical properties: indeed, we show that the sample complexity of the model scales like the number of parameters, which can be quite low and comparable to matrix completion methods.}
	
	\textcolor{black}{In addition, we also prove generalization bounds which bound the \textit{total variation distance} between the ground truth distribution and the recovered probabilities based on the assumption that our custom-loss function has been minimized during training: to the best of our knowledge, this one of the first generalization bounds for recommendation systems that applies to the implicit feedback task of predicting the likelihood of interactions, rather than the ratings prediction task. 
		Appendix~\ref{sec:implicitbounds} explains in more detail how to relate the model's output to specific interaction probabilities for a dataset of size $N$, whilst Appendix~\ref{subsec:B}  provides more details on the type of generalization bounds that can be obtained. }

	%To provide a clearer understanding, let us consider a scenario where the ratings matrix demonstrates a mere $2$\% observance rate. In such cases, a $G_{i,j,0}$ value of $0.98$ would be deemed as normal. This notion is supported by the visual representation in Figure~\ref{MainPicture}, where unobserved squares consistently appear darker. Conversely, a value like $G_{i,j,0}\approx 0.995$ would suggest an unusually low probability of the item being consumed or observed by the user. However, our interest doesn't solely pertain to the probability of user interaction with an item; we are equally concerned about whether the user would appreciate the interaction. To tackle this, we implement a re-scaling mechanism for the less probable classes, i.e., item being interacted with ratings $1$ through $5$. This re-scaling ensures their probabilities sum up to $1$, thereby allowing the computation of an expectation. For illustration, consider a plausible example: $G_{i,j}$ is represented as $\{0.95,0,0.01,0.01,0.02,0.01\}$. In this case, there is a 95\% chance that the user won't interact with the item, a $0$\% chance of interaction and giving it a rating of $1$, a 2\% chance of interaction with a 4-star rating, and so forth. In our example, $\{0,0.2,0.2,0.4,0.2\}$ emerges as the re-scaled probabilities given the item is interacted, leading to an expected value (and prediction) of 3.6-stars.}

Our contributions can be summarised as follows: 

\begin{enumerate}
	\item We propose Conv4Rec, a simple 1-by-1 convolutional autoencoder architecture for recommender systems which learns from both implicit and explicit feedback to perform the recommendation task.
	\item Our predictions come in the form of  \textcolor{black}{6 probabilities $p_0,p_1,\ldots,p_5$ for each (user, item pair): one probability is associated to associated to each possible rating ($1,2,3,4$ or $5$)  and an additional probability $p_0$ which models the probability of a lack of any interaction. This naturally allows the model to make predictions on the implicit feedback (by ranking the items in decreasing order of $p_0$) and on explicit feedback (via the values of $p_1,\ldots,p_5$).}
	\item  \textcolor{black}{This distributional output also allows our model to identify \textit{serendipitous items}: items with a small probability of being observed but a high probability of being rated highly if observed. }
	%This implies we are also able to model the probability of sampling each entry, allowing us to make more refined and interpretable predictions about both implicit and explicit feedback. 
	\item We prove generalization bounds for explicit feedback which show that the sample complexity scales like the sum of the number of parameters of the model and the embedding dimension. We also prove norm-based generalization bounds which exploit the proximity to initialization.
	\item \textcolor{black}{In what we understand as one of the first generalization bounds for implicit feedback in the recommendation task, we prove that minimizing our custom cross-entropy loss ensures the ability to recover the sampling distribution over interactions up to a small error in total variation. Cf. Theorem~\ref{cor:explicitfromimplicit}.}
	\item We experimentally evaluate our model on a wide range of datasets and benchmarks and show that we can achieve \textit{competitive performance} at both the ratings prediction and implicit feedback prediction tasks.  \textcolor{black}{Importantly, this performance is achieved despite the fact that our method consists in a single model for both implicit and explicit feedbacks, unlike most other baselines. This demonstrates the quality of the derived distributional predictions and their potential use in other applications such as identifying serendipitous items.}    % Importantly, we illustrate how to interpret the Conv4Rec predictions.
\end{enumerate}

\section{Related Works}

\noindent\textbf{Recomender Systems:} Recommender systems are traditionally bifurcated along the lines of explicit vs. implicit feedback, each with their respective assumptions, objectives, metrics, and datasets.  On the one hand, classical models for explicit feedback based on matrix factorization \cite{koren09} attempt to fit the ground-truth ratings, based on various assumptions on the data spaces (e.g., Gaussian \cite{PMF} or Poisson \cite{shan10,gopalan15}).  Such models cannot easily accommodate implicit feedback as it would be inappropriate to employ simplistic assumptions that unobserved interactions should fit to ratings of 0. On the other hand, implicit feedback models \cite{hu2008collaborative,rendle09} make use of signals of greater abundance but of lower exactness or reliability (e.g., clicks). They are more amenable to ranking, rather than score fitting, objectives. Early attempts at leveraging both feedback types~\cite{corating,10.1145/3341981.3344225} used a linear combination of both types of loss functions, an approach that was less than principled as the two loss functions may be in different spaces and the determination of the relative weights was left unaddressed. In contrast, our proposed approach represents a seamless framework that could learn from both types of feedback simultaneously, allowing predictions of the likelihood of adoption (implicit) as well as the degree of preference (explicit).

Recent literature on collaborative filtering has also largely moved from matrix factorization to neural models of varying architectures, particularly to model implicit feedback. \cite{he2017neural} and~\cite{dziugaite2015neural} can be viewed as non-linear extensions of matrix factorization, employing the popular scalar product in combination with multilayer
perceptron. 	
 \textcolor{black}{In~\cite{ComparativeConvolutionsTNNLS22}, a convolutional architecture is designed to extract meaningful representations from both user and item side information (such as age, occupation, genre) and leverage it to perform the recommendation task. The evaluated performance metrics are purely implicit feedback related (recall and AUC). } 	 \textcolor{black}{In~\cite{KernelDeepLearningImplicitExplicitTNNLS2024}, a deep neural network architecture is constructed which, similarly to our work,  relies on both the implicit and explicit feedback one hot encodings. However, the model is trained with the square loss and evaluated only with RMSE and MAE on the explicit feedback prediction task. }

A variety of  \textcolor{black}{\textbf{autoencoders}} have also been considered, from AutoRec \cite{sedhain2015autorec} relying on standard (deterministic) autoencoders, CDAE \cite{wu16} relying on denoising autoencoder (DAE), to VAE-CF relying on variational autoencoders \cite{Liang18,truong2021bilateral}. \textcolor{black}{In~\cite{R3_3_autoencoder_side_info}, an autoencoder architecture is designed to perform the recommendation task with the particularity that side information is not just used in the encoding phase but also incorporated in the reconstruction loss.}  Many of these models focus on implicit feedback alone, with their differences essentially boiled down to the underlying neural architecture (in the generic sense), rather than in how explicit feedback signals could be modelled jointly over and above the implicit feedback.

We also draw attention to how this work focuses on modelling the collaborative filtering signals themselves, as opposed to the preponderance of works that rely on neural architectures to model auxiliary data (e.g., text or images) while still using matrix factorization for the user-item interactions. For instance, ConvMF \cite{kim2016convolutional} employs convolutional neural network to model the text of reviews, while CDL \cite{wang2015collaborative} employs a denoising autoencoder for the same. Similarly, \cite{LLMexploringuserretrieval} relies on LLMs to to incorporate Movie information in the recommendation process whilst~\cite{TNNLS_knowledge_graph} does the same whilst incorporating knowledge graph information.  In turn, VBPR \cite{he16} uses pretrained image models to extract features from product images, used as content-based aid to an implicit feedback model.

\noindent  \textcolor{black}{\textbf{Self-supervised Methods in Recommender Systems:}}

Independently, there are explorations on using \textbf{graph-based models} to aggregate neighborhood information in the context of recommendation (e.g., a user likely prefers what other similar users also prefer) \textcolor{black}{\cite{he2020lightgcn,wang2019neural,deng2022graph,dengliangraph,gao2024rethinking}. For instance, the popular method LightGCN~\cite{he2020lightgcn} creates user and item representations which are propagated through the graph of interactions, whilst~\cite{TNNLS_GNN1} and~\cite{TNNLS_GNN2} encode different forms of positive and negative interactions in the graph. Furthermore, some GNN-based research utilizes graph information not merely to improve accuracy, but to achieve other aims such as diversification~\cite{TNNLS_diversifying} or trustworthy recommendation~\cite{TNNLS_trust_GNN,TNNLS_trust_aware}, which complements uncertainty estimation techniques~\cite{IEEE_uncertainty,TNNLS_alves}.
	In~\cite{R3_1_federated_Friendship_prediction}, the authors rely on GNNs and federated learning to predict friendship links between users. In~\cite{R3_2_session_music} attention-based GNNs are used to perform the task of session-based music recommendation. In~\cite{R3_4_intent_aware}, a recommender system is developed for Point of Sale recommendation by modelling each user's check-in behavior as graphs. }  \textcolor{black}{Various other methods employ graphs or hypergraphs in similar self-supervised approaches to the recommendation task~\cite{xu2024multi,dong2024multi,ye2025fuxi,knowledgeGuidedGraphTNNLS2022,musicTNNLS23,emojis}.}
 \textcolor{black}{In~\cite{wang2024category}, an ingenious self-supervised learning method is designed for session based recommendation with item side information being provided in the form of categories. Specifically, a session hypergraph (HG) is constructed whose nodes are items and item categories and whose hyperedges are defined by each session, and an additional session graph (SG) is constructed whose nodes are the sessions and whose edges connect sessions with items or categories in common. The session embeddings produced by both the session hypergraph and the session graph are pulled closer together with SSL (self-supervised learning), and the session based recommendation task is performed based on a similarity measure between the HG session embeddings and the item embeddings. 
Whilst this approach provides impressive results, we do not include it in our baselines because we are concerned with the implicit and explicit feedback prediction based on all available interactions, rather than session based recommendation. In addition, the proposed method does not aim to predict explicit feedback and we do not assume the presence of side information about items or users.}

 \textcolor{black}{In~\cite{liu2023dynamically}, a GNN based approach to \textit{streaming recommendation } is constructed which estimates the changes in user preferences between each time stamp and updates the learned graph embeddings accordingly. Since this is a pure implicit feedback method which studies the streaming recommendation setting relying on timestamp information, we do not include it in the baselines. There are also several recent works which similarly exploit the time dimension in streaming recommendation or reinforcement learning based recommendation~\cite{TNNLS24_reinforcement,knowledgeGuidedGraphTNNLS2022,Plug_play_TNNLS2025,Dynamic,Time_Interval,TemporallyEvolving}.}

%There are also many exciting recent works which also rely on self supervised information through graphs or transformers to perform implicit feedback prediction in various contexts~\cite{xu2024multi,dong2024multi,ye2025fuxi,knowledgeGuidedGraphTNNLS2022,musicTNNLS23,emojis}. Those works provide invaluable progress in the field of retrieval at large and realistic scales. However, they rely on timestamp and sometimes other side information and only output predictions for the implicit feedback alone. 

We also mention the important emerging direction of federated recommender systems~\cite{FedRec,FedRec2}, where multiple clients must be able to train and obtain recommendations with some local training without submitting all of their previous interaction data to the server. We refer the interested reader to the review~\cite{FedRec}.

Still, such approaches are orthogonal to ours. \textcolor{black}{Indeed, the function class represented by GNNs is radically different from that of our autoencoder based method.  \textcolor{black}{Furthermore, whilst these methods may rely on both feedback types when passing embeddings through the interaction graphs, they} typically can only be applied to \textit{predict} either implicit or explicit feedback, whilst we aim to provide a single model that incorporates both. }

%\noindent \textbf{Matrix completion methods:}
%There is a wealth of literature on \textit{low-rank} matrix completion. In this case, there is no non-linearity and the function classes considered are easier to define mathematically, yie lding favorable theoretical guarantees~\cite{softimpute,CandesRecht,CandesTao2010,ReallyUniform1,Foygel_Max_Norm,IMCtheory1}  On the other hand, such approaches lack representational capacity the ability to incorporate implicit feedback either when extracting information from the observations or at prediction time.  

\noindent \textcolor{black}{\noindent \textbf{Joint Modelling of Implicit and Explicit Feedbacks}}

 \textcolor{black}{To the best of our knowledge, there are very few models which explicitly attempt to model explicit and implicit feedback jointly.  Furthermore, as explained in the introduction, none simultaneously provide predictions and evaluations of their methods on both implicit and explicit feedback without \textit{enforcing} a correlation between them. }

In~\cite{mandal2018explicit}, the authors propose a joint model for implicit and explicit feedback prediction in a context where many different forms of feedbacks (including "views" and "helpfulness scores") are provided. Each type of feedback is rescaled to the scale $[-1,1]$ and modelled via a matrix $\text{tanh}(AB^\top)$ for factor matrices $A\in\R^{m\times r},B\in\R^{n\times r}$, and regularization is applied to force the user based factor matrices to be close to each other.

However, we note (cf. Tables 3 and 4) that the authors \textbf{only evaluate their method on RMSE} (explicit feedback).

 \textcolor{black}{	In~\cite{corating}, a single low-rank matrix $UV^\top$ is used to estimate both the explicit feedback (rating) and the implicit feedback (interaction or lack thereof) resulting in the following loss function:  
\begin{align}
	&	\min_{U,V} \sum_{(i,j)\in\Omega} [[UV^\top]_{i,j}-R_{i,j}]^2   \label{eq:corating}\\&+ \lambda_1 \sum_{i=1}^m \sum_{j=1}^n [[UV^\top]_{i,j}-I_{i,j}]^2+  \lambda_2 \left[ \|U\|_{\Fr}^2 +\|V\|_{\Fr}^2 \right],\nonumber
\end{align}
where $R\in\R^{m\times n}$ and $I\in \{-1,1\}^{m\times n}$ denote the explicit and implicit feedbacks respectively where $-1$ indicates the lack of any interaction in the implicit feedback case. }
 \textcolor{black}{In~\cite{WADMF}, the authors propose a new approach to tackle situations where a partially annotated dataset is provided: it is assumed that an explicitly annotated set of ratings $D_e$ is provided together with an additional pure implicit feedback dataset $D_i$  of unannotated interactions. The method initially trains an explicit feedback model (matrix factorization) on $D_e$  and then uses the predictions made by the model on $D_i$ to create a \textit{weakly annotated dataset} and train a second matrix factorization model using it. The resulting model is a single matrix $\Phi$ which can naturally be used to predict explicit feedback, and is also used to rank items in implicit feedback tasks.}

 \textcolor{black}{To the best of our knowledge, these two works~\cite{corating,WADMF} are the only ones which evaluate their method on both implicit and explicit feedback: this is done by using the recovered matrix $UV^\top$ (or $\Phi$) both for ratings prediction and as a ranking mechanism. In particular, both methods assume by construction that items which are likely to be highly ranked are also likely to be interacted with, and is unable to identify serendipitous items (items with low interaction probability but high expected rating conditioned on interaction).}

\noindent  \textcolor{black}{\textbf{Uncertainty Estimation and Recommender Systems with Distributional Predictions:}}

 \textcolor{black}{Uncertainty estimation has received substantial interest in the recommender systems community~\cite{coscrato2023estimating,OrdRec}. In particular, many models estimate uncertainty through the use of a distributional modelling of the predictions over ratings. This allows one to use the variance of the predictions as a proxy for uncertainty. For instance, in the celebrated work~\cite{PMF} (Probabilistic Matrix Factorization), the ratings ascribed to each user item combination are assumed to follow a normal distribution, with both the variance and the mean being inferred form data. 
Similarly, in~\cite{URP} (URP, User Ratings Model), a probabilistic model is constructed for estimating the probabilities of ratings within a discrete set of possible ratings where an additional latent variable $Z$ is introduced to determine the `user attitude' responsible for the particular rating being considered. }

 \textcolor{black}{In all cases, it is assumed that the probability distribution of the rating given for each user/item combination belongs to a parametric class $\mathcal{P}_\theta$ of probability distributions over possible ratings, with the parameters $\theta_{i,j}$ of the distribution varying with each user/item combination $(i,j)$ in a predetermined way. For instance, in the PMF model~\cite{PMF}, the distribution $\mathcal{P}$ is  a normal distribution with variance $\theta^1_{{i,j}}=\sigma^2$  and mean $\theta^2_{i,j}=\mu_{i,j}=\langle u_i,v_i\rangle$  for some user and item embeddings $u_i$ and $v_i$. Furthermore, only the conditional distribution of the ratings given the presence of interaction (i.e., the explicit feedback) is modelled, not the probability of interaction itself (the implicit feedback). This can be seen from the expression of the likelihood used to optimize each model (see equation 5 in~\cite{PMF} and equation 1 in~\cite{URP} (or equations 10,  and 14 respectively in the survey~\cite{coscrato2023estimating}), which are both expressed as a product of likelihoods of each individual rating over each observed entry, without involving any additional factor for the likelihood of having observed those entries in particular). Thus, both models are restricted to the probabilistic modelling of explicit feedback, whist our model estimates probabilities for both the implicit and explicit feedback. }

 \textcolor{black}{Similarly, in~\cite{BeMF}, a separate low-rank model (trained with FunkSVD~\cite{FunkSVD}) is trained to independently estimate probabilities for all possible discrete ratings (e.g. $\{1,2,3,4,5\}$), resulting in $5$ distinct matrices $\Phi_1,\ldots,\Phi_5$ (cf. bottom of Page 7:11 in~\cite{coscrato2023estimating}). Each implicit feedback matrix is represented as $\text{logistic}(AB^\top)$ for some matrices $A\in\R^{m\times r},B\in \R^{n\times r}$ .  Then, each entry's explicit feedback is estimated via $\sum_{s=1}^5 \frac{\Phi_s}{\sum_{s' } (\Phi_s)_{i,j}}s $.  Although the calculation is analogous to our calculation in equation~\eqref{eq:asinBeMF} the implicit feedback is neither predicted nor incorporated in the loss function in~\cite{BeMF}. Thus, the method remains a pure explicit feedback method and is not directly comparable to ours. 
}

 \textcolor{black}{	In~\cite{OrdRec}, an estimate of the probability associated to each rating is also calculated by postprocessing a single low-rank model $\Phi_{i,j}$ via $\P(r_{i,j}\leq \fix_{\kappa})=\frac{1}{1+\exp(\Phi_{i,j})-s_{i,\fix_{\kappa}}}$ where the thresholds $s_{i,\fix_{\kappa}}$ only depend on the user $i$ and the candidate rating value $\fix_\kappa$. Note that as in~\cite{BeMF,URP,PMF}, only the conditional probability of the rating given the entry is estimated, not the probability of interaction. }

 \textcolor{black}{Lastly, we note that much research attempts to provide interpretable and explainable predictions for recommender systems in ways that go beyond a simple distributional output.  For instance, \cite{pan2020explainable} propose a feature mapping approach that aligns uninterpretable general features with interpretable `aspect features' of items, achieving both satisfactory accuracy and explainability in recommendations.  In~\cite{Davidson2010}, recommendations are explained based on which previously watched item was used the most in the recommendation. Such approaches achieve a better explanation of the rationale behind recommendations, whilst we focus on providing a more detailed distributional prediction of all possible interactions or lack thereof. Thus, they are orthogonal to our approach. \cite{zhang2020explainable} provides a comprehensive survey on explainable recommendation techniques. 
}

\noindent \textbf{Statistical learning theory and GNNs:}
%\textbf{Generalization bounds:} 
Many generalization bounds have been shown for matrix completion models~\cite{ReallyUniform2,IMCtheory1}. However, they rely heavily on the geometric properties of the inner product and cannot accommodate deeper architectures. \textcolor{black}{In~\cite{deng2022graph}, generalization bounds were proved for Graph Neural Networks for the recommendation task which, ignoring logarithmic terms, scale like the number of parameters of the model. In addition, the authors also introduce an innovative training module, IMix, which improves the performance of the GNN model considered. Other similar generalization bounds were proved in~\cite{within1,within2,within3,within4,huang2024foundations} for general GNN tasks. However, such works assume that the graph is fixed in advance rather than part of the sampling procedure as in Recommender Systems settings. Furthermore, GNN based models are fundamentally different from our autoencoder model in terms of the function class expressed. }  There is a very rich literature on generalization bounds for neural networks~\cite{spectre,LongSed,graf,ledent2021normbased}, with recent emphasis on the optimization procedure~\cite{NTK,Aroratwolayers,gradientover}. However, our neural network takes as input the representation of all of a user's many interactions, thus, the inputs are not i.i.d. when taking the problem at face value. This requires extra care and covers a number of arguments. There is also a limited amount of work on generalisation bounds for autoencoders~\cite{generoder}. However, the problems are fundamentally different: the inputs (user features) to our auto-encoder do not correspond to "inputs" in the supervised learning task of explicit feedback prediction. In fact, perfect reconstruction of each input is not the most desirable aim. For each \textit{unobserved} entry, the information which our model uses for explicit feedback prediction lies in the scores ascribed to the possible \textit{observed} ratings, not in the score ascribed to the correct ``unobserved'' category. Thus, the generalization behavior of the reconstruction loss is a poor indicator of the generalization performance of the recommender systems model. This calls for a different analysis.

\section{Methodology}
\label{sec:methodology}
\textbf{Basic notation:} 
We have $m$ users and $n$ items, and the ground truth ratings matrix is denoted by $R\in \R^{m\times n}$. We observe $N$ entries in an i.i.d. fashion. We assume that there are $k$ possible rankings and the possible ratings are denoted as $\fix_1<\fix_2<\ldots<\fix_k$. For instance, if the rating scale is $\{1, 2, \cdots,5\}$, then $k=5$ and $\fix_1=1,\fix_2=2,\ldots,\fix_5=5$. \textcolor{black}{In the interest of full generality, the theoretical results are applicable to arbitrary real values for $\fix_1,\ldots,\fix_k$, though in most real-life examples, the ratings would be integers.} %{\color{blue}We assume the  $R_{i,j}=\fix_0$ if an entry $(i,j)$ is not observed.}
For the reader's convenience, a table of the notations used in this paper is provided in the appendix. 

\noindent \textbf{Preprocessing step and input to model:}
As an autoencoder, our model replicates the input as the output. To build the input for a user $i\leq m$, we use a sparse matrix $U_i\in \R^{n\times (k+1)}$ with the $j^{th}$ row corresponding to item $j\leq n$.  Each row is modeled as a one-hot encoding vector $U_{i,j,\nbull} \in \R^{(k+1)}$, where $U_{i,j,\kappa}=1$ if $(i,j)$ is observed and $R_{i,j}=\fix_{\kappa}$, $U_{i,j,0}=1$ if $(i,j)$ is not observed and $U_{i,j,\kappa}=0$ for all other values of $(i,j,\kappa)$.

%Our model's first step is to encode each user $i\leq m$ as a sparse matrix $U_i\in \R^{n\times (k+1)}$ in such a way that for each $j$, the $j$th row of $U_i$ represents a one-hot encodings of the "ranking or lack of observation" given by $i$ to item $j\leq n$. More precisely, $(U_i)_{j,\kappa}=1$ if $(i,j)$ is not observed and  $(U_i)_{j,\kappa}=1$ if $(i,j)$ is observed and $R_{i,j}=\kappa+1$. 

\noindent \textbf{Forward pass:}
Our model then feeds each user feature $U_i$ through an \textit{encoder network} $\phi_{\theta_1}:\R^{n\times (k+1)}\rightarrow \R^r$ with $D_1$ parameters $\theta_1\in\R^{D_1}$  to produce a low-dimensional bottleneck representation $\phi_{\theta_1}(U_i)\in\R^{r}$ of dimension $r\ll m,n$. Then, this low-dimensional representation is fed through a \textit{decoder} network $g_{\theta_2}:\R^r\rightarrow \R^{n\times (k+1)}$ to obtain a matrix of scores $g_{i,\nbull,\nbull}\in\R^{n\times (k+1)}$.  Note that $g_{i,j,0}$ represents the score associated to the lack of observation, and $g_{i,j,\kappa}$ represents the score associated to the rating $\fix_{\kappa}, \forall 1 \leq \kappa  \leq k$. The scores $g_{i,j,\kappa}$ are translated into probabilities $G_{i,j,\kappa}$ via a softmax: $G_{i,j,\kappa}=\text{softmax}(g_{i,j,\nbull})_\kappa$.   Therefore $G_{i,j,\kappa}$ is the probability of observing $\fix_{\kappa}$ at entry $(i,j)$.  This means that the probabilities of each rating, conditionally given that the entry is observed, can be computed as $\widetilde{G}_{i,j,\kappa}=\sum_{\kappa=1}^{k} \frac{  G_{i,j,\kappa}}{\sum_{\kappa'=1}^{k}  G_{i,j,\kappa'}}$\footnote{Note that this is equivalent to applying a softmax to all $\kappa$ (including $0$) and then conditioning on the rating being observed. Indeed $\frac{\text{softmax}(g_0,\ldots,g_k)_o}{\sum_{i=1}^k \text{softmax}(g_0,\ldots,g_k)_i }= \text{softmax}(g_1,\ldots,g_k)_o$}.

%Our model then feeds the each user feature $U_i$ through an \textit{encoder network} $\phi_{\theta_1}:\R^{n\times (k+1)}\rightarrow \R^r$ with $D_1$ parameters $\theta_1\in\R^{D_1}$  to produce a low-dimensional bottleneck representation $\phi_{\theta_1}(U_i)\in\R^{r}$ of dimension $r\ll m,n$. Then, this low-dimensional representation is fed through a \textit{decoder} network $g_{\theta_2}:\R^r\rightarrow \R^{n\times (k+1)}$ to obtain a matrix of scores $g_{i,\nbull,\nbull}\in\R^{n\times (k+1)}$ where for $\kappa\geq 1$, $g_{i,j,\kappa}$ represents the score associated to the rating  $\kappa+1$ and $g_{i,j,\kappa}$ represents the score associated to the lack of observation. The scores $g_{i,j,\kappa}$ are translated into probabilities $G_{i,j,\kappa}$ via a softmax: $G_{i,j,\kappa}=\text{softmax}(g_{i,j,\nbull})_\kappa$.   Therefore $G_{i,j,\kappa}$ is the "probability" of observing rating $\kappa+1$ at entry $(i,j)$.  This means that the probabilities of each rating, conditionally given that the entry is observed, can be computed as $\widetilde{G}_{i,j,\kappa}=\sum_{\kappa=1}^k \frac{\fix_\kappa  G_{i,j,\kappa}}{\sum_{k=1}^5 G_{i,j,\kappa}}$\footnote{Note that this is equivalent to applying a softmax to all $\kappa$ (including $0$) and then conditioning on the rating being observed. Indeed $\frac{\text{softmax}(g_0,\ldots,g_k)_o}{\sum_{i=1}^k \text{softmax}(g_0,\ldots,g_k)_i }= \frac{\frac{e^{g_o}}{\sum_{i=1}^ke^g_i}}{\sum_{j=1}^k\frac{e^{g_j}}{\sum_{i=1}^ke^g_i}}= \frac{e^{g_0}}{\sum_{j=1}^ke^{g_j}}=\text{softmax}(g_1,\ldots,g_k)_o$}.

The outputs  $G_{i,j,\kappa}$ can be interpreted on their own to provide information about which interactions %ratings
are the most likely in relation to each other.
Alternatively, a hard and fast prediction of the final rating can also be made via the formula   \textcolor{black}{
\begin{align}
	\label{eq:asinBeMF}
	F_{i,j}=\sum_{\kappa=1}^{k} \frac{\fix_\kappa  G_{i,j,\kappa}}{\sum_{\kappa'=1}^{k} G_{i,j,\kappa'}}=\sum_{\kappa=1}^k \fix_\kappa  \widetilde{G}_{i,j,\kappa}
\end{align}}
%\footnote{Note that this is equivalent to applying a softmax to all $\kappa$ (including $0$) and then conditioning on the rating being observed. Indeed $\frac{\text{softmax}(g_0,\ldots,g_k)_o}{\sum_{i=1}^k \text{softmax}(g_0,\ldots,g_k)_i }= \frac{\frac{e^{g_o}}{\sum_{i=1}^ke^g_i}}{\sum_{j=1}^k\frac{e^{g_j}}{\sum_{i=1}^ke^g_i}}= \frac{e^{g_0}}{\sum_{j=1}^ke^{g_j}}=\text{softmax}(g_1,\ldots,g_k)_o$}

\noindent \textbf{Loss function and training:} We train our model to reproduce the one-hot encoding input $U_i$ for each user $i\leq m$, i.e. we enforce $\R^{n\times (k+1)}\ni U_i\simeq G_{i,\nbull,\nbull}$ for all $i$. In practice, our loss function is an average over all users and items of the cross-entropy loss across the last dimension of $U_i$: 
\begin{align}
	\Elc(G)&:=\frac{1}{mn} \sum_{i=1}^m\sum_{j=1}^n \elc(U_{i,j,\nbull},G_{i,j,\nbull}) \\
	&=\frac{1}{mn} \sum_{i=1}^m\sum_{j=1}^n -\log(G_{i,j,[R_{\Omega}]_\kappa}),\label{zerobased}
\end{align}
where $R_{\Omega}$ is the matrix of known interactions\footnote{Note that to simplify notation, the entries of $G_{i,j,\nbull}$ are indexed with a zero-based convention so that the logarithm in Equation~\eqref{zerobased} is $-\log(G_{i,j,0})$ if the entry is unobserved}: 
\begin{align}
	[R_{\Omega}]_{i,j} &=\kappa \quad \text{if} \quad (i,j) \in\Omega \>\land \> R_{i,j}=u_\kappa \\
	&= 0 \quad \text{if} \quad  (i,j)\notin \Omega.
\end{align}

Our algorithm minimizes the reconstruction loss $\Elc$ over the choice of parameters $\theta_1\in\R^{D_1}$ and $\theta_2\in\R^{D_2}$ via classic gradient-based methods implemented in TensorFlow.

\noindent \textbf{Architecture:} We employ a combination of $1\times 1$ convolutions and fully-connected layers. We explain the architecture for the decoder first, and then straightforwardly extend it to the encoder. Let $X^{\ell-1}_i$ (resp. $X^{\ell}_i$) be the input to (resp. output of) the $\ell^{th}$ layer. 
The first $L_0$ layers are fully connected, whilst the  \textit{last}  $L-L_0$ layers (of the decoder) are convolutional. The \textit{inputs} to each convolutional layer have a "spacial dimension" corresponding to the items: if $\ell\geq L_0$, then $X^{\ell}_i\in \R^{n\times K_\ell}$, where each row of $X^{\ell}_i$ corresponds to an item $j\leq n$. On the other hand, the inputs to fully-connected layers only have a channel dimension: if $\ell\leq L_0-1$, $X^{\ell}_i\in\R^{K_\ell}$.

For the first (fully-connected) $L_0-1$ layers the input $X^{\ell-1}_i$ lies in $\R^{K_{\ell-1}}$ and the operation consists in
\begin{align}
	\R^{K_\ell}\ni X_i^{\ell}=\sigma(W^\ell X_{i,\nbull}^{\ell-1})
\end{align}
where $W^\ell\in\R^{K_\ell}\times \R^{K_\ell-1}$ and $\sigma$ is the (component-wise) non-linearity, which we assume to be \textbf{1-Lipschitz} in theoretical results and set as \textbf{Relu} in the experiments. For the fully-connected layer $L_0$, the operation consists in 
\begin{align}
	\R^{n\times K_\ell}\ni X_i^{\ell}=\sigma(W^\ell X_{i,\nbull}^{\ell-1}),
\end{align}
where $W^\ell\in \R^{nK_{\ell}\times K_{\ell-1}}$.

The convolutional operation consists in applying the weights $W^\ell\in \R^{K_\ell\times K_{\ell-1}}$ to each of the rows of $X^{\ell}_{i}$. Thus, for $\ell\geq L_0+1$, the output of layer $\ell$ will satisfy: 
\begin{align}
	X^{\ell}_{i,j,\nbull}=	\sigma(W^\ell  X^{\ell-1}_{i,j,\nbull}) \quad \quad \forall i\leq m, j\leq n,
\end{align}
where the weights $W^\ell\in \R^{K_\ell\times K_{\ell-1}}$ are tied between each item $j\leq m$ (and between each user $i\leq n$). 

%The first $L_0$ layers are fully-connected, whilst the  \textit{last}  $L-L_0$ layers of the decoder convolutional. 

The encoder is defined similarly: the first several layers are convolutional, whilst the last layers leading up to the embedding layer are fully connected. To avoid notational clutter, we do not explicitly write the equivalent equations for the encoder. For illustration, see Figure~\ref{MainPicture}.

\noindent \textbf{Remarks:}

\begin{itemize}
	\item  The notation $\theta_2\in\R^{D_2}$ stands for all parameters in the decoder network. In particular,  $D_2= \sum_{\ell\neq L_0} K_{\ell}K_{\ell-1} +K_{L_0-1}n K_{L_0}$. 
	\item The probabilities $G_{i,j,0}$ assigned to the lack of observation of entry $(i,j)$ are typically very high: they are calibrated assuming a total sample size of $N$. Thus, if the model performed perfectly, $1-G_{i,j,0}$ would be the probability that a sample of $N$ i.i.d. observations from the sampling distribution $\mathcal{D}$ does not hit $(i,j)$. For instance, for a uniform sampling distribution with $N\ll mn$, we should have $1-G_{i,j,0}\simeq \frac{N}{mn}$. Likewise, in a scenario where the ratings matrix demonstrates a mere $2$\% observance rate a $G_{i,j,0}$ value of $0.98$ would be deemed as normal. This notion is supported by the visual representation in Figure~\ref{MainPicture}, where unobserved squares consistently appear darker. Conversely, a value such as  $G_{i,j,0}\approx 0.995$ would suggest an unusually low probability of the item being consumed or observed by the user. However, the relative magnitudes of those probabilities for different entries afford us a practical way of estimating the sampling distribution, i.e., of predicting the so-called implicit feedback.
	%\item We also experimented with training  via the square loss appended to the aggregated predictions $F_{i,j}\in[\fix_1,\fix_2]$ and convex combinations of both losses. We found that the regression loss made the training much more difficult with marginal performance improvements at best. 
\end{itemize}

\section{Generalization Bounds}

We provide results in terms of the square loss for explicit feedback where we bound the \textit{generalization gap}. \textcolor{black}{In addition, we characterize the implicit feedback performance of the model trained with our loss function by also providing bounds on the total variation distance between the estimated probability distribution and the ground truth.}
\textcolor{black}{Before delving into the results, let us clarify the assumption regarding the \textbf{sampling procedure}: as mentioned in the methodology section, we assume that $N$ entries are sampled \textit{i.i.d.}. Each entry $(i,j)$ and its corresponding observation $R_{i,j}+\zeta_{i,j}$ (where $R_{ij}$ is the ground truth and $\zeta_{i,j}$ is the noise) is sampled i.i.d. from a joint distribution over $([m]\times [n])\times \R$. Since each observed entry $R_{i,j}+\zeta_{i,j}$ must belong to the set of possible ratings $\{\fix_1,\ldots,\fix_{\kappa}\}$ (i.e. $\{1,2,3,4,5\}$ in Douban, Movielens, etc ), this is equivalent to assuming there is an underlying probability distribution $p\in \R^{[m]\times [n]\times [\kappa] }$  with such that the probability of a sample being $((i,j),\fix_\kappa)$ is  $p_{i,j,\kappa}$ and $\sum_{i=1}^m\sum_{j=1}^n\sum_{\kappa=1}^k p_{i,j,\kappa}=1$. Thus, for any $i,j$, the ground truth $R_{i,j}$ is defined as the expected rating for entry $i,j$: $ \frac{\sum_{\kappa=1}^kp_{i,j,\kappa} \fix_{\kappa}}{\sum_{\kappa=1}^kp_{i,j,\kappa} }$. In particular, if there is no noise, the only value $\kappa$ such that $p_{i,j,\kappa}\neq 0$ is the one defined by  $R_{i,j}=\fix_{\kappa}$.  %For convenience, we also write $p_{i,j,0}$ for the quantity $1-\sum_{\kappa=1}^k p_{i,j,\kappa}=\sum_{(i',j')\neq (i,j)}\sum_{\kappa=1}^k p_{i,j,\kappa}$.  
}
\textcolor{black}{We define the empirical \textcolor{black}{(or training)} error	$\rmseemp=\frac{1}{N}\sum_{(i,j)\in\Omega}\elr(R_{i,j}+\zeta_{i,j},F_{i,j})$  where as usual, $\elr$ is the square loss and $F_{i,j}= \sum_{\kappa=1}^{k} \frac{\fix_\kappa  G_{i,j,\kappa}}{\sum_{\kappa'=1}^{k} G_{i,j,\kappa'}}$ is the explicit feedback prediction of the model for entry $i,j$. 
	Similarly, the population error, \textcolor{black}{which corresponds to the expected error on an i.i.d. test set, is defined as} 
	$\rmsepop=\mathbb{E}_{\mathcal{D}}\left(\elr(R_{i,j}+\zeta_{i,j},F_{i,j})\right)$. We provide bounds on the generalization gap, which is the difference $\rmsepop-\rmseemp$. 
	%As with most results in learning theory, the bounds are uniform over all functions in our chosen function class, which means that the bounds also imply excess risk bounds for the global minimizer of the square loss. However, we optimize with the cross entropy loss instead, as this allows us to . Still, we also show in the appendix that the
}

We provide bounds of two types:
\begin{itemize}
	\item \emph{Parameter counting} bounds, which rely on counting the parameters and estimating Lipschitz constants, and 
	\item  \emph{Norm-based} bounds, which are more sensitive to the geometry of weight space and approach zero when the weights approach initialization.
\end{itemize}
There is a trade-off between the appearance of factors of the norms of the weights and the number of parameters: depending on the proof strategy, one of those quantities ends up in the logarithmic factors, and the other ends up as the dominant term of the bound. The results are not directly comparable, and the dichotomy between those two approaches has been explained in various other works, including most notably the comprehensive work of~\cite{graf}, as well as earlier works~\cite{Germ1,Germ2,Shallow2}. 
%\cite{Germ1,Germ2,Shallow2,antoine}. 

It is worth noting that our bounds mostly rely on the existence of a low-dimensional representation present in the embeddings $\phi_{\theta_1}(U_i)\in\R^r$, rather than the specific way in which it was obtained via $\phi_{\theta_1}$. This explains why the dimension or the norms of $\theta_1$ do not appear explicitly in the bounds. In practice, the use of the encoder network greatly improves training efficacy and may afford statistical advantages, although such advantages are not rigorously incorporated in the bounds. The bounds apply to the square loss, evaluated on the predictions given by $F_{i,j}=\sum_{\kappa=1}^{k} \frac{\fix_\kappa  G_{i,j,\kappa}}{\sum_{\kappa'=1}^{k} G_{i,j,\kappa'}}$  $=\sum_{\kappa=1}^k \fix_\kappa  \widetilde{G}_{i,j,\kappa}$, but similar results could be obtained for the classification loss. Although the techniques bear similarities to existing methods from~\cite{ledent2021normbased,LongSed,graf}, there are several differences: (1) We use architecturally specific norms specially designed for our $1\times 1$ convolutional architecture to bound the weights and the Lipschitz constants of each layer. %Whilst a similar approach was taken in~\cite{antoine}, the specific architecture of this problem makes the results especially tight and satisfying. 
(2) There is a need to provide a cover for the embedding space and to propagate this choice to the other layers: we use an $L^\infty$ covering number at the decoder level which treats our cover of the embedding space as an `auxiliary dataset'. In particular, it is not, to the best of our knowledge, possible to rely on $L^2$ covering numbers for the norm-based bounds as was done in~\cite{graf}. (3) As a result of point (2), an unexpected factor of $r$ appears inside the square root in the dominating third term in Equation~\eqref{eq:normmain}. It is worth noting that although the results are expressed in terms of explicit norm constraints, it is trivial to extend them (via a union bound) to \textit{post-hoc} bounds where such quantities as $a_\ell, \chi,s_\ell$ are replaced by empirically-evaluated analogues at the modest cost of extra constant and logarithmic factors.

\subsection{Architecture-specific norms}

In order to obtain tighter bounds in terms of the norms of the weights and the relevant Lipschitz constants, we must define the following architecture-dependent norms on the parameter and activation spaces of each layer $\ell$ of the decoder: 
we will write $\|\nbull\|_{\ell}$ for an architecture-dependent norm at layer $\ell$ which is equal to the maximum Frobenius norm of a convolutional patch, except at the last layer where we use the $L^\infty$ norm. More precisely:
\begin{itemize}
	\item $\|g(\phi(U))\|_{\ell}=\|g(\phi(U))\|_{\Fr}$  if $\ell\leq  L_0-1$ ($\ell+1^{\textcolor{black}{th}}$ layer is fully connected);
	\item $\|g(\phi(U))\|_{\ell}=\max_{j\leq n}\|g(\phi(U))_{j,\nbull}\|_{\Fr}$  if $ L_0\leq \ell\leq L-1$ (i.e. the $\ell+1^{\textcolor{black}{th}}$ layer is convolutional) and;
	\item $\|g(\phi(U))\|_{L}=\|g(\phi(U))\|_{\infty}$ at the last layer.
\end{itemize} 
Similarly, we will write $\|W^\ell\|_{\ell}$ for the Lipschitzness constant of layer $\ell$ with the weights $W^\ell$ with respect to the norms $\|\nbull\|_{\ell-1}$ and $\|\nbull\|_{\ell}$. The norms $\|W^\ell\|_\ell$ can be explicitly computed in terms of classic matrix norms via a different formula depending on whether $\ell<L_0$, $\ell=L_0$ or $\ell>L_0$, cf. Appendix~\ref{ap:notation}.

\begin{table*}[]
	\centering
	\caption{Performance comparison of Conv4Rec vs baselines on the real datasets regarding implicit-feedback metrics}
	\begin{tabular}{ll|c|c|}
		\cline{3-4}
		&          & \textbf{Recall@$50$}             & \textbf{Recall@$100$}           \\ \hline
		\multicolumn{1}{|l|}{}                                                         &  \textcolor{black}{\textbf{CoRating}}      & \color{black}$ \mathbf{0.1655 \pm 0.0012}$   & \color{black} $\mathbf{0.2524 \pm 0.0026}$ \\ \cline{2-4} 
		\multicolumn{1}{|l|}{}                                                         & \color{black} \textbf{WADMF}      &\color{black}  $0.0727 \pm 0.0019$  & \color{black} $0.1194 \pm 0.0026$ \\ \cline{2-4} 
		\multicolumn{1}{|l|}{}                                                         & \textbf{NCF}      & $0.1530 \pm 0.0012$  & $0.2302 \pm 0.0021$ \\ \cline{2-4} 
		\multicolumn{1}{|l|}{}                                                         & \textbf{LightGCN} & $0.1324 \pm 0.0043$  & $0.2065 \pm 0.0047$ \\ \cline{2-4} 
		\multicolumn{1}{|l|}{}                                                         & \textcolor{black}{\textbf{XSimGCL}} & \textcolor{black}{$0.1069 \pm 0.0022$}  & \textcolor{black}{$0.1734 \pm 0.0031$}  \\ \cline{2-4} 
		\multicolumn{1}{|l|}{\multirow{-6}{*}{\textbf{Douban}}}                                 & \textbf{Conv4Rec} & ${0.1610 \pm 0.0022}$ & ${0.2487 \pm 0.0053}$  \\ \hline
		\rowcolor[HTML]{DDDDDD} 
		\multicolumn{1}{|l|}{}                                                         & \color{black} \textbf{CoRating}      &\color{black} $0.5437 \pm 0.0097$  & \color{black} $0.6810 \pm 0.0076$ \\ \cline{2-4} 
		\rowcolor[HTML]{DDDDDD} 
		\multicolumn{1}{|l|}{}                                                         & \color{black} \textbf{WADMF}      & \color{black} $0.3005 \pm 0.0079$  & \color{black} $0.4528 \pm 0.0086$ \\ \cline{2-4} 
		\rowcolor[HTML]{DDDDDD} 
		\multicolumn{1}{|l|}{\cellcolor[HTML]{DDDDDD}}                                 & \textbf{NCF}      & $0.5492 \pm 0.0053$  & $0.6971 \pm 0.0067$ \\ \cline{2-4} 
		\rowcolor[HTML]{DDDDDD} 
		\multicolumn{1}{|l|}{\cellcolor[HTML]{DDDDDD}}                                 & \textbf{LightGCN} & $\mathbf{0.5546 \pm 0.0061}$  & $0.7032 \pm 0.0068$  \\ \cline{2-4} 
		\rowcolor[HTML]{DDDDDD} 
		\multicolumn{1}{|l|}{\cellcolor[HTML]{DDDDDD}}                                 & \textcolor{black}{\textbf{XSimGCL}} & \textcolor{black}{${0.5184 \pm 0.0167}$}  & \textcolor{black}{$0.6818 \pm 0.0104$}  \\ \cline{2-4} 
		\rowcolor[HTML]{DDDDDD} 
		\multicolumn{1}{|l|}{\multirow{-6}{*}{\cellcolor[HTML]{DDDDDD}\textbf{MovieLens 100K}}} & \textbf{Conv4Rec} & $0.5276 \pm 0.0137$  & $\mathbf{0.7073 \pm 0.0085}$  \\ \hline
		\multicolumn{1}{|l|}{}                                                         & \color{black} \textbf{CoRating}      & \color{black} $0.3206 \pm 0.0030$  & \color{black} $0.4286 \pm 0.0033$ \\ \cline{2-4} 
		\multicolumn{1}{|l|}{}                                                         & \color{black} \textbf{WADMF}      & \color{black} $0.2337 \pm 0.0017$  & \color{black} $0.3127 \pm 0.0015$ \\ \cline{2-4} 
		\multicolumn{1}{|l|}{}                                                         & \textbf{NCF}      & $0.3987 \pm 0.0010$  & $0.5274 \pm 0.0008$  \\ \cline{2-4} 
		\multicolumn{1}{|l|}{}                                                         & \textcolor{black}{\textbf{LightGCN}} & \textcolor{black}{$0.3511 \pm 0.0014$}  & \textcolor{black}{$0.4682 \pm 0.0010$}  \\ \cline{2-4} 
		\multicolumn{1}{|l|}{}                                                         & \textcolor{black}{\textbf{XSimGCL}} & \textcolor{black}{$0.3154 \pm 0.0026$}  & \textcolor{black}{$0.4238 \pm 0.0002$}  \\ \cline{2-4} 
		\multicolumn{1}{|l|}{\multirow{-6}{*}{\textcolor{black}{\textbf{MovieLens 25M}}}}                          & \textbf{Conv4Rec} & $\mathbf{0.4653 \pm 0.0011}$           & $\mathbf{0.5860 \pm 0.0013}$           \\ \hline
		
		\rowcolor[HTML]{DDDDDD} 
		\multicolumn{1}{|l|}{}                                                         & \color{black} \textbf{CoRating}      & \color{black} $0.0801 \pm 0.0102$  & \color{black} $0.0972 \pm 0.0107$ \\ \cline{2-4} 
		\rowcolor[HTML]{DDDDDD} 
		\multicolumn{1}{|l|}{}                                                         & \color{black} \textbf{WADMF}      & \color{black} $0.0391 \pm 0.0007$  & \color{black} $0.0650 \pm 0.0008$ \\ \cline{2-4} 
		\rowcolor[HTML]{DDDDDD} 
		\multicolumn{1}{|l|}{\cellcolor[HTML]{DDDDDD}}                                 & \textcolor{black}{\textbf{NCF}}      & \textcolor{black}{$0.0973 \pm 0.0011$}  & \textcolor{black}{$0.1256 \pm 0.0014$} \\ \cline{2-4} 
		\rowcolor[HTML]{DDDDDD} 
		\multicolumn{1}{|l|}{\cellcolor[HTML]{DDDDDD}}                                 & \textcolor{black}{\textbf{LightGCN}} & \textcolor{black}{$0.0964 \pm 0.0014$} & \textcolor{black}{$0.1395 \pm 0.0031$}  \\ \cline{2-4} 
		\rowcolor[HTML]{DDDDDD} 
		\multicolumn{1}{|l|}{\cellcolor[HTML]{DDDDDD}}                                 & \textcolor{black}{\textbf{XSimGCL}} & \textcolor{black}{$0.0084 \pm 0.0006$}  & \textcolor{black}{$0.1214 \pm 0.0016$}  \\ \cline{2-4} 
		\rowcolor[HTML]{DDDDDD} 
		\multicolumn{1}{|l|}{\multirow{-6}{*}{\cellcolor[HTML]{DDDDDD}\textcolor{black}{\textbf{Amazon Electronics}}}} & \textcolor{black}{\textbf{Conv4Rec}} & \textcolor{black}{$\mathbf{0.0998 \pm 0.025}$}  & \textcolor{black}{$\mathbf{0.1437 \pm 0.0047}$}  \\ \hline
		\multicolumn{1}{|l|}{}                                                         &  \color{black} \textbf{CoRating}      & \color{black} 
 \color{black}$0.1974 \pm 0.0026$  & \color{black} $0.2425 \pm 0.0018$ \\ \cline{2-4} 
		\multicolumn{1}{|l|}{}                                                         & \color{black} \textbf{WADMF}      & \color{black} $0.0634 \pm 0.0005$  & \color{black} $0.0988 \pm 0.0002$ \\ \cline{2-4} 
		\multicolumn{1}{|l|}{}                                                         & \textcolor{black}{\textbf{NCF}}      & \textcolor{black}{$0.2154 \pm 0.0056$}  & \textcolor{black}{$0.2900 \pm 0.0061$}  \\ \cline{2-4} 
		\multicolumn{1}{|l|}{}                                                         & \textcolor{black}{\textbf{LightGCN}} & \textcolor{black}{$\mathbf{0.2803 \pm 0.0062}$}  & \textcolor{black}{$\mathbf{0.3781 \pm 0.0055}$}  \\ \cline{2-4} 
		\multicolumn{1}{|l|}{}                                                         & \textcolor{black}{\textbf{XSimGCL}} & \textcolor{black}{$0.2540 \pm 0.0048$}  & \textcolor{black}{$0.3417 \pm 0.0050$} \\ \cline{2-4} 
		\multicolumn{1}{|l|}{\multirow{-6}{*}{\textcolor{black}{\textbf{Amazon Games}}}}                          & \textcolor{black}{\textbf{Conv4Rec}} & \textcolor{black}{${0.2377 \pm 0.0284}$}          & \textcolor{black}{$0.3245 \pm 0.0382$}          \\ \hline
	\end{tabular}
	\label{tab:implicit}
\end{table*}

\begin{table*}[]
	\centering
	\caption{Performance  for explicit feedback (RMSE)}
	\begin{tabular}{l|c|c|c|c|c|}
		\cline{2-6}
		& \textbf{Douban}                                & \textbf{MovieLens 100K}                        & \textbf{MovieLens 25M}    & \textcolor{black}{\textbf{Amazon Electronics}}  & \textcolor{black}{\textbf{Amazon Games}}                      \\ \hline
		\multicolumn{1}{|l|}{\color{black} \textbf{CoRating}} & \color{black} $1.1141\pm 0.0154$   & \color{black} $0.9969 \pm 0.0107$  & \color{black} $1.1270 \pm 0.1086$ & \color{black} $1.2887 \pm 0.0062$  & \color{black} $1.4021 \pm 0.0149$ \\ \hline
		\multicolumn{1}{|l|}{\color{black} \color{black} \textbf{WADMF}} & \color{black} $1.0823\pm 0.0334$   & \color{black} $1.1000 \pm 0.01890$  & \color{black} $1.1547 \pm 0.0500$ & \color{black} $1.0970 \pm 0.0049$  & \color{black} $1.1809 \pm 0.0005$ \\ \hline
		\multicolumn{1}{|l|}{\textbf{SoftImpute}} & $0.8706\pm 0.0061$   & $0.9832 \pm 0.0181$  & $0.8343 \pm 0.0094$ & \textcolor{black}{$1.1589 \pm 0.0334$}  & \textcolor{black}{$1.1721 \pm 0.0399$} \\ \hline
		\multicolumn{1}{|l|}{\textbf{BOMIC}}       & $0.8583 \pm 0.00355$ & $0.9702 \pm 0.0256$  & $0.8273 \pm 0.0015$ & \textcolor{black}{$1.1470 \pm 0.0376$}  & \textcolor{black}{$1.1679 \pm 0.0406$} \\ \hline
		\multicolumn{1}{|l|}{\textbf{IGMC}}       & $0.7836 \pm 0.0047$  & $0.9556 \pm 0.02025$ & $\mathbf{0.8176 \pm 0.0037}$ & \textcolor{black}{$1.1809 \pm 0.0578$}  & \textcolor{black}{$1.2211 \pm 0.1163$} \\ \hline
		\multicolumn{1}{|l|}{\textbf{NNMF}}       & $0.7855 \pm 0.0058$  & $0.9582 \pm 0.0145$  & $0.8224 \pm 0.0038$   & \textcolor{black}{$0.9521 \pm 0.0112$}  & \textcolor{black}{$1.0661 \pm 0.0823$}                          \\ \hline
		\multicolumn{1}{|l|}{\textbf{AutoRec}}    & $0.8176 \pm 0.0029$        & $0.9691 \pm 0.0055$        & $0.8773 \pm 0.0063$ & \textcolor{black}{$0.9723 \pm 0.0319$}   & \textcolor{black}{$1.0768 \pm 0.0764$}      \\ \hline
		%\multicolumn{1}{|l|}{\textcolor{black}{\textbf{SAPR}}}    & \textcolor{black}{$0.0000 \pm 0.0000$}       & \textcolor{black}{$0.0000 \pm 0.0000$ }      & \textcolor{black}{$0.0000 \pm 0.0000$} & \textcolor{black}{$0.0000 \pm 0.0000$}  & \textcolor{black}{$0.0000 \pm 0.0000$}      \\ \hline
		\multicolumn{1}{|l|}{\textbf{Conv4Rec}}   & $\mathbf{0.7707 \pm 0.0059}$ & $\mathbf{0.9347 \pm 0.0138}$  & $0.8863 \pm 0.0050$  & $\textcolor{black}{\mathbf{0.9366 \pm 0.0183}}$  & $\textcolor{black}{\mathbf{1.0542 \pm 0.0777}}$                            \\ \hline
	\end{tabular}
	\label{tab:explicit}
\end{table*}

\subsection{Parameter-counting bounds}
\begin{theorem}
	\label{thm:generalization_bound_main}
	% \footnote{As is standard in the literature, we  greatly simplify the exposition by using notation that assumes that each entry cannot be sampled more than once. However, since the entries are i.i.d., the same entry can be sampled several times, in which case the loss must be computed independently and summed for all occurrences of the entry. More rigorously the observations are i.i.d of the form $(\xi^o,\bar{\xi}^o)$ with $\xi^o\in \{1,2,\ldots,m\}\times \{1,2,\ldots,n\}$  and $\bar{\xi}^o\in [\fix_1,\ldots,\fix_k]$  and write  $\sum_{o=1}^N \Theta(\bar{\xi}_o)$ instead of $\sum_{(i,j)\in \Omega} \Theta(R_{i,j} \zeta_{i,j})$ for any function $\Theta$, and it should be assumed that the "ground truth" values $R$ satisfy $\E(\zeta_{i,j})=0$.}.
	
	Fix $\beta,\nu,\chi>0$ and initialized weights $M^1,\ldots,M^L$. With probability greater than $1-\delta$ over the draw of the training set $\Omega$, for every set of parameters $\theta_1,\theta_2$ satisfying the following conditions:
	\begin{align}
		\label{cond:MW_main}
		\sum_{\ell\leq L}\|W^\ell-M^\ell\|_\ell \leq  \beta  \nonumber \\
		\left\|M^\ell\right\|_\ell\leq 1+\nu	\\
		\max_{i\leq m} \|\phi_{\theta_1}(U_i)\|\leq \chi  
	\end{align}
	we also have:
	\begin{align}
		&\left| \rmsepop -\rmseemp \right|\leq \label{eq:paracount}  \\
		&	  O\Bigg([\fixx^2] \sqrt{\frac{\log(1/\delta)}{N}} + \frac{\fixx^2 }{N} +  \fixx^2 \sqrt{ \frac{[mr+D_2] [\beta+\nu L]}{N}}\Bigg)\nonumber \\ &+O\Bigg(\fixx^2\sqrt{\frac{[mr+D_2]\log\left( N\fixx (\chi +\beta) (\chi+1) +1 \right)  }{N}}\Bigg),  \nonumber 
	\end{align}
	where $\fixx:=\fix_k-\fix_1$ denotes the maximum span between the ratings.
\end{theorem}

\textcolor{black}{\textbf{Interpretation}: as can be seen from the above results, for matrices close to initialization and for fixed $L,\beta,\nu$, the dominant terms scale like $\widetilde{O}\left(\fixx^2 \sqrt{\frac{[mr+D_2]}{N}} \right)$. Thus, treating as constant the scaling quantity $\fixx^2$ (which is the constant $25=5^2$ in datasets with ratings in the set $\{1,2,\ldots,5\}$), the number of samples required for learning to be statistically efficient roughly scales like $mr+D_2$, which is the sum of the number of parameters in the decoder and the total embedding dimension over all users. Since the use of convolutional layers drastically reduces the number of parameters, this implies that the bound captures the statistical advantage of the convolution. If the number of fully connected layers $L_0$ is $1$ and the width of the convolutions in the decoder is constant and equal to $w$, then $D_2=rnK_{L_0}+(L_1-1)w^2+w(k+1)$, so the dominant sample complexity scales like $D_2=mr+ rnK_{L_0}+(L_1-1)w^2+w(k+1)$. The dependencies on $w,L_1,k$ and $m,n$ are not multiplicative, which can help guide architecture search for various datasets: the expressivity of the convolutions in the decoder can be decided roughly based on the total number of observed entries (independently of the number of users or the number of items), whilst the learnable range of $r$ depends on both $m$ and $n$, and the learnable range of $K_{L_0}$ depends on $n$ only. Thus, roughly the same values for $w,L_1$ can be used for various datasets, whilst $K_{l_0}$ must be chosen taking the number of items into account and $r$ must depend on both the number of users and the number of items. From this it can be seen that in general, the tuning of $K_{L_0}$ and $r$ is much more decisive than that of $w$, $L_0$. Nevertheless, we also expect some of this effect to be mitigated by the implicit rank restriction imposed by the weight decay imposed on the convolutions, which has an analogous effect in standard DNNs~\cite{BottleNeckRank,ledent2024generalization,wangimplicit}. However, precisely characterizing such effects is outside the scope of this work.} %This is especially true if taking into account the low rank bias of deep networks at the overparametrization regime, which suggests that the architecture of the convolutions isn't particularly 

\noindent \textbf{Remark:}
The explicit constants are provided in the appendix, and the result holds for all values of $N$. In particular, we provide the proof from first principles, in line with~\cite{graf} and differing from~\cite{LongSed}. In addition, differently from both~\cite{graf,LongSed}, we take advantage of the special architecture to improve the bounds by better management of the norms involved in the Lipschitz constants. %This is in line with~\cite{antoine}, but the architecture of this paper makes the spectral norms $\|W^\ell\|_{\ell}$ particularly easy to compute and minimize unnecessary dependencies. 
This is even more pronounced in the norm-based bounds, which we approach in the next subsection.

\textcolor{black}{
\textbf{Comparison with bounds on GNNs:} as explained in the related works section, several works have shown generalization bounds for GNN architectures, including in the recommendation setting~\cite{deng2022graph,within1,within2,within3,within4}. Since the models are radically different, the generalization bounds are not directly comparable.  \textcolor{black}{In general, if we \textit{fix the depth $L$} and the norms of the weights, in both our result in Thm~\ref{thm:generalization_bound_main} and the main result in~\cite{deng2022graph}, the \textit{sample complexity} scales roughly as the number of parameters in the model. In our case, this is clear from equation~\eqref{eq:paracount} and for $\nu=0$, holds up to logarithmic factors of norm based quantities even if we take the dependence in $L$ into account. In the case of~\cite{deng2022graph}, this can be seen from equation (4) on page 5, where the dominant term is a factor of $d$ (the embedding size), which becomes $\widetilde{O}(\sqrt{\frac{d^2}{N}})$ if we treat both the norm factors as constant. Here, $d^2$ is the number of parameters in one fully connected layer. If we take into account the dependence on depth $L$ (whilst still ignoring logarithmic factors), then the bound in~\cite{deng2022graph} scales like $\widetilde{O}(\sqrt{\frac{d^2 C_L^{2L}}{N}})$ contains an additional weakness in the form of the factor $C_L^L$, which introduces \textit{exponential depth dependence. In contrast, the bound in~\eqref{eq:paracount} scales like $\widetilde{O}\left(\frac{mr+D_2}{N}\right)$ or $\widetilde{O}\left(\frac{L[mr+D_2]}{N}\right)$, which is at worst an additional multiplicative dependence in $L$ if we do not assume $\nu=0$. } In general, we also note that even the number of parameters as a function of the architectural choices of Conv4Rec and other methods such as~\cite{deng2022graph} can differ. For instance,} consider a graph neural network model with initial user and item embedding size $r_1$ with a $L$ message passing layers of width $r_1$ and a final output taking the form of an inner product between the concatenated user and item features. This would have $Lr_1^2+(m+n)r_1$ parameters and outputs a matrix of rank $Lr_1$. A version of Conv4Rec with $L$ encoder and decoder convolutions with width $k=6$ and bottleneck user representations of size $r$ would have an internal representation of rank $r$ and a number of parameters equal to $(L-1)k^2+kL+ r(m+n)$. If we take $r=r_1$, Conv4rec will generally have fewer parameters, whilst if we take $r=Lr_1$, it will have more. 
}
\subsection{Norm-based Bounds}

\begin{theorem}
\label{thm:NormBased_main}

Fix $a_1,\ldots,a_L>0$, $\chi>0$ and $s_1,\ldots,s_L>0$. With probability greater than $1-\delta$ over the draw of the training set $\Omega$, for every set of parameters $\theta_1,\theta_2$ satisfying the following conditions: 
\begin{align}
	\|W^\ell\|_\ell & \leq s_\ell \quad \quad \forall \ell \leq L \\ 
	\left\|\left(W^\ell-M^\ell\right)^\top\right \|_{2,1} &\leq a_l    \quad \quad \forall \ell \leq L-1,\\
	\left\|\left(W^L-M^L\right)^\top\right \|_{\Fr} &\leq a_L 	\label{eq:definenorms}\\
	\max_{i\leq m} \|\phi_{\theta_1}(U_i)\|&\leq \chi 
\end{align}
we also have $\left| \rmsepop-\rmseemp\right|   \leq  $
\begin{align}
	\label{eq:normmain}
	&O\left([\fixx^2] \sqrt{\frac{\log(1/\delta)}{N}}\right)    + \widetilde{O}\left(\fixx^2 \sqrt{\frac{mr}{N} }\right)  \\ &     +  \widetilde{O}\left(\fixx^2 \left[\prod_{\ell=1}^L s_\ell\right]   \left[  \sum_{\ell =1}^L \left(\frac{a_\ell }{s_\ell }\right)^{\frac{2}{3}} \right]^\frac{3}{2}\sqrt{\frac{r  }{N}}  \right)  \nonumber 
\end{align}

\end{theorem}

\noindent \textbf{Remark:} It is worth noting the dependence on $k$ in Theorems~\ref{thm:NormBased_main} and~\ref{thm:generalization_bound_main}.  While the parameter-counting bound from Theorem~\ref{thm:generalization_bound_main} explicitly depends on $k$ through the number of parameters $D_2=\sum_{\ell\neq L_0} K_{\ell}K_{\ell-1} +K_{L_0-1}n K_{L_0}$, which includes the term $ K_{L}K_{L-1}=K_{L-1}k$, the norm-based bound~\eqref{thm:NormBased_main} only involves $k$ indirectly through the norm $a_\ell$. 
More precisely, ignoring all logarithmic factors, the contribution to the bound from the last layer is $\sqrt{\frac{K_{L-1}k}{N}}$ for the parameter counting bound and $\left[\prod_{\ell\leq L}s_\ell\right]\sqrt{\frac{ra_L^2}{N}}$ for the norm based version. 
In the case where the observations are well-balanced between all possible ratings $\fix_1,\ldots,\fix_k$, both bounds show a sample complexity that is roughly linear in $k$. This dependence on $k$ alone, rather than the product of $k$ and $m$ or $n$, is remarkable:  the convolutional layers only need to learn associations between the possible ratings $\fix_1,\ldots,\fix_k$ once, for the whole dataset. This provides partial theoretical justification for the success of our method.

%In addition, there are interesting subtleties between the finer dependence on $k$ in norm-based and parameter counting bounds: when certain values of $\fix_1,\ldots,\fix_k$ are rarely observed, the norm-based bound will be much tighter, scaling like the effective number of different possible observations which occur frequently. 
Note the norm-based bounds are much more sensitive to scaling factors, as can be readily observed from the presence of the factor of $\left[\prod_{\ell<L}s_\ell\right]$. Although our results apply to the square loss, these considerations bear some similarities to other works on norm-based bounds in multi-class classification contexts~\cite{lei2015multi,ledent2021normbased,wu2021fine,structured}.

\textbf{Remark:} Note that although the theoretical results presented apply to explicit feedback only (giving guarantees in terms of generalization error with the square loss), this is only for convenience: very similar proof techniques would achieve similar generalization bounds as those of Theorems~\ref{thm:NormBased_main} and~\ref{thm:generalization_bound_main} in terms of the excess risk of our loss function, which implies that under realisability assumptions we can closely approximate the ground truth distribution over all (entry, rating) pairs with a number of samples which only depends on the architecture of the network and can be less than the total number of possible entries $mn$. Furthermore, the dependence on $k$ is similar, reaching the same key conclusion that the weight sharing allows our model to learn the associations between the ratings jointly over the whole dataset, with negligible sample complexity cost for both implicit and explicit feedback recovery. 
\subsection{Total Variation Guarantees on the Recovery of the Ground Truth Distribution}
\textcolor{black}{In this section, we show that training with our cross entropy loss provably leads to a recovery of the ground truth distribution over the entries. This can be interpreted as a guarantee in terms of implicit feedback. Furthermore, we also show in the zero noise, realizable case that a guarantee on the test MSE can be obtained for a model trained with our cross entropy loss function. Formally, the theorem applies to the loss function defined in Equations~\eqref{eq:therealloss} and~\eqref{eq:thereallosspop}, but they coincide with the loss function described here when there are no duplicate entries. See Subsection~\ref{subsec:B} for details. }
\textcolor{black}{
\begin{theorem}[Cf. Corollaries~\ref{cor:firstcorrr} and~\ref{cor:explicitfromimplicit}]
	\label{thm:allthisnewstuff}
	Fix $\beta,\nu,\chi>0$ and initialized weights $M^1,\ldots,M^L$.  Let $g^*$ (resp $\widehat{g}$ ) be the minimizer of the loss~\eqref{eq:thereallosspop} (resp~\eqref{eq:therealloss}) subject to the following conditions: 
	\begin{align}
		\label{cond:MWimplnewnewnew}
		\sum_{\ell\leq L}\|W^\ell-M^\ell\|_\ell &\leq  \beta  \nonumber \\
		\left\|M^\ell\right\|_\ell&\leq 1+\nu   \nonumber \\
		\max_{i\leq m} \|\phi_{\theta_1}(U_i)\| &\leq \chi  \nonumber   \\
		\|g_{\theta_2}(\phi_{\theta_1}(U_i))_{j,\kappa}\|&\leq  B \quad \forall i,j,\kappa 
	\end{align}
	Let $\widehat{p}_{\nbull,\nbull,\nbull}=\frac{\widehat{G}_{\nbull,\nbull,\nbull}}{\mathcal{G}}$ denote the normalized distribution over $[m]\times [n]\times [k]$ obtained from $\widehat{G}$. 
	Then, if $g^*=\gbay$ (i.e., the ground truth is realizable),  we have with probability greater than $1-\delta$ over the draw of the training set $\Omega$
	\begin{align}
		\left| p_{\nbull,\nbull}- \widehat{G}_{\nbull,\nbull}\right|_1=	\sum_{i,j=1}^{m,n} \left|p_{i,j}-\widehat{G}_{i,j}\right|\leq \sqrt{\frac{1}{2}\mathcal{Q}},
	\end{align}
	where $\mathcal{Q}:= $
	\begin{align}
		\nonumber
		O\left[B \sqrt{\frac{\log(2/\delta)}{2N}}\right] +\widetilde{O}\left[\frac{B}{\sqrt{N}}\sqrt{ [mr+D_2]\left[ [\beta+\nu L] \right]}\right],
	\end{align}	
	where the $\widetilde{O}$ notation hides logarithmic factors in $N,\chi,\beta$. 
	Furthermore, assume that there is no noise: i.e. for each $i,j$, there exists a single $\kappa_{i,j}\leq k$ such that $p_{i,j,\kappa_{i,j}}\neq 0$ and $p_{i,j,\kappa'}=0$ for all $\kappa'\neq \kappa_{i,j}$. Then we have on the same high-probability event: 
	\begin{align}
		\rmsepop\leq 2\fixx^2 \sqrt{\frac{1}{2}\mathcal{Q}}.
	\end{align}
	\end{theorem}}

	\textcolor{black}{In conclusion, the above result establishes the validity of the use of our loss function as a proxy for both implicit and explicit feedback. Indeed if the ground truth is realizable, minimizing our loss function provably guarantees that  \begin{enumerate}
	\item The recovered distribution approaches the ground truth distribution in total variation. This can be interpreted as a guarantee on the implicit feedback performance. 
	\item In the noiseless case, minimizing our  \textcolor{black}{entry-wise Cross-Entropy} loss function also guarantees good performance in terms of the square loss.  \textcolor{black}{This is in contrast to existing works which prove generalization guarantees in terms of the loss function used for training. } This is an additional result on the performance in terms of explicit feedback. 
\end{enumerate}
}
\textcolor{black}{\textbf{Remark:} Note that the dependence on the sample size $N$ is more favorable in Theorem~\ref{thm:NormBased_main}  (i.e. $N^{-\frac{1}{2}}$) than in Theorem~\ref{thm:allthisnewstuff} ($N^{-\frac{1}{4}}$). However, the results are fundamentally different: Theorem~\ref{thm:NormBased_main} bounds the \textit{generalization gap} for any predictor learned by the model. Thus, it doesn't concern the training procedure and only gives a favorable result if the empirical error is low. On the other hand, Theorem~\ref{thm:allthisnewstuff} shows that minimizing our loss function will indeed implicitly minimize both the RMSE and total variation distance in some cases. }

\section{Experiments}

To compare our model with baselines, we conducted experiments using real-world datasets under two distinct user feedback settings. In the first setting, we examined recommendation models that rely on the user's explicit ratings (e.g., a user-$i$ rates an item-$j$ on a scale between one and five stars). The second setting deals with implicit feedback. For this implicit feedback component, an "interaction" is recorded whenever a recommender detects that user-$i$ interacts with item-$j$ from, regardless of user-$i$'s eventual ranking preference. We demonstrate below that our model can handle both scenarios simultaneously with state-of-the-art performance. We split the baselines based on the feedback settings.

The following baselines were evaluated for the explicit feedback scenario.  \textcolor{black}{The hyperparameter tuning range was usually chosen as close as possible to the original references, and some important details are described below each specific baseline. }
\begin{itemize}
\item \textbf{SoftImpute:} a classic matrix completion method that uses nuclear-norm regularization~\cite{softimputeALS}. The parameter $\lambda$ was tuned by examining the following set of values $\{10^{-6},10^{-5}, \cdots, 10^{3}\}$.
\item \textbf{BOMIC:} An inductive matrix completion approach based on a sum of multiple mutually orthonormal components, together with nuclear-norm regularization. We the form that only considers the presence o biases. We set the hyperparametes $\lambda_0 = 0$ and $\lambda_2=\lambda_3$, as in Section~$II.B$ of~\cite{omic}. $\lambda_3$ and $\lambda_4$ were selected by considering the same range of \mbox{\textbf{SoftImpute}}.
\item \textbf{IGMC:} this method trains a graph neural network (GNN) based on 1-hop subgraphs built from the ratings matrix. The hyperparameters were tunned according to Section~5 of ~\cite{Zhang2020Inductive}. 
\item \textbf{NNMF:} this method replaces the inner product in traditional matrix factorization with a learned multi-layer feed-forward neural network. We tunned the hyperparameters according to Section~5 of~\cite{dziugaite2015neural}. 
\item \textbf{AutoRec} is an autoencoder framework for collaborative filtering. We tuned the regularisation strength $\lambda$ and the appropriate latent dimension $k$ according to Section~3 of~\cite{sedhain2015autorec}. 
\end{itemize}

The following are the baselines for the implicit feedback:

\begin{itemize}

\item \textbf{NCF:} this is a neural network-based method that addresses collaborative filtering in recommendation systems by replacing traditional inner product-based interactions between users and items with a neural architecture capable of learning arbitrary functions directly from data. We validate this model by considering the Section~4.1 (item \textit{Parameter Settings}) from~\cite{he2017neural}.

\item \textbf{LightGCN:} the model consists of a concise Graph Convolution Network (GCN) that retains only the most essential component of GCN — neighborhood aggregation — for collaborative filtering. The %validation procedure was
 \textcolor{black}{hyperparameters were} adapted from Section 4.1.2 of~\cite{he2020lightgcn}  \textcolor{black}{and cross-validated to take into account the characteristics of the datasets evaluated in this work. Namely, a grid search with the following values was performed: embedding size $\in \{32, 64, 128\}$, number of layers $K \in \{2, 3, 4\}$, regularization coefficient $\lambda \in \{1e^{-3}, 1e^{-4}, 1e^{-5}\}$, and learning rate $\in \{1e^{-2}, 1e^{-3}, 1e^{-4}\}$.} %, taking into account the characteristics of the datasets that are evaluated in this work.

\item \textcolor{black}{\textbf{XSimGCL:} this model consists of a contrastive learning-based graph neural network approach to enhance recommendation systems by leveraging noise-based embedding augmentation instead of traditional graph augmentations. The hyperparameters were selected based on the guidelines provided in Section 4.1 of~\cite{XSimGCL}, with adjustments made to suit the specific characteristics of the datasets evaluated in this study.}  \textcolor{black}{Specifically, the following values were evaluated in a grid-search: coefficient of the contrastive loss $\lambda \in \{0.01, 0.05, 0.1, 0.2\}$, number of layers $\in \{2, 3\}$, temperature $\tau \in \{0.1, 0.15, 0.2\}$, the magnitude of noise $\epsilon \in \{0.01, 0.05, 0.1, 0.2\}$, embedding size $\in \{32, 64, 128\}$ and learning rate $\in \{1e^{-2}, 1e^{-3}, 1e^{-4}\}$}.

\item  \textcolor{black}{\textbf{WADMF:} The model from~\cite{WADMF} one of the very few existing works which perform predictions for both explicit and implicit feedbacks. Both predictions are based on a single matrix $\Phi$ which is learnt with a two-stage matrix factorization approach using both implicit and explicit feedback datasets. The implicit feedback dataset used in pretraining is referred to as a `weakly annotated dataset'. We thus refer to the method as WADMF (Weakly Annotated Dataset Matrix Factorization) in the tables.}

\item  \textcolor{black}{\textbf{Co-rating:} The Co-rating model~\cite{corating}, which also produces a single matrix which is used for both explicit feedback prediction and ranking. This is a matrix factorization based method, cf. equation~\eqref{eq:corating} in the introduction.}
\end{itemize}

\begin{figure*}
\centering
\includegraphics[width=0.85\textwidth]{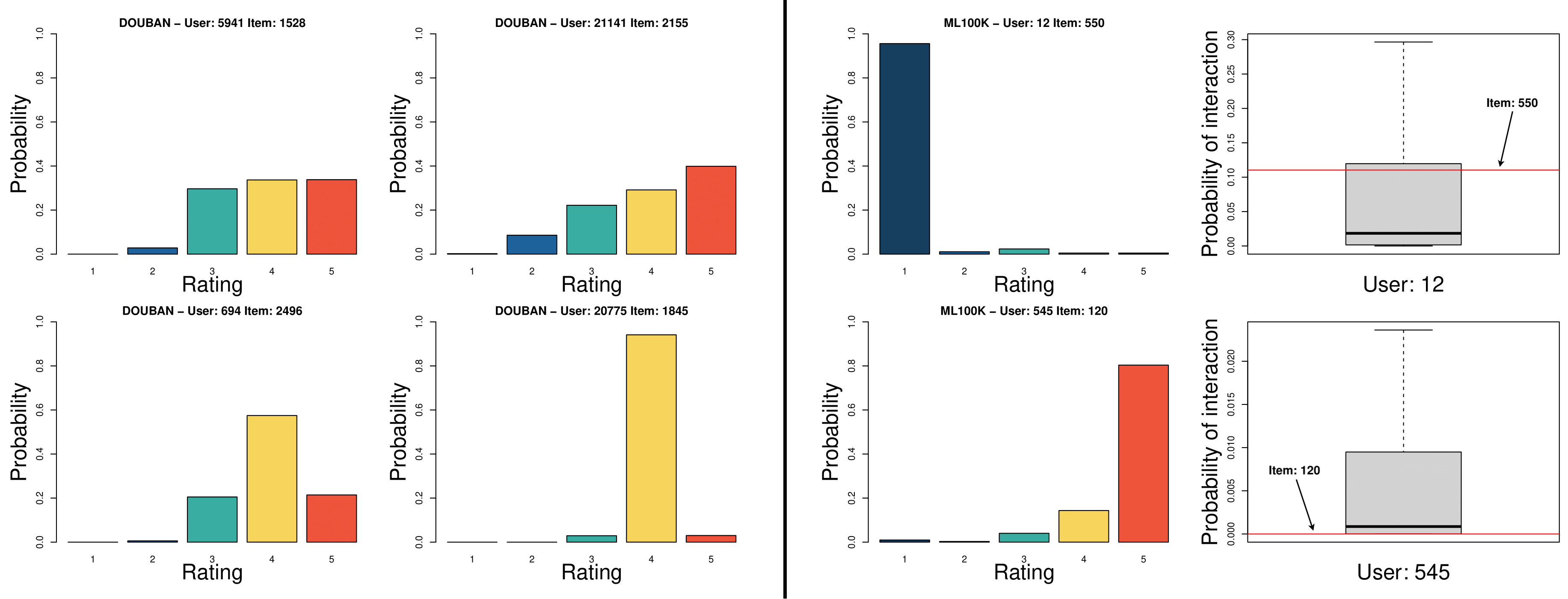}
\caption{ \textcolor{black}{Illustration of the information that can be extracted from our model's output}. The four graphs on the \emph{left} represent output probabilities ($\widetilde{G}_{i,j,\kappa}$ for $0\leq \kappa\leq 5$) for four user-item combinations sampled from the \textbf{Douban} dataset's test set. The four graphs on the \emph{right} illustrate the output probabilities $\tilde{G}_{i,j,\kappa}$ for $1\leq \kappa\leq k$ and the probability of an interaction between the corresponding user $i$ and item $j$, in comparison to the distribution of such probabilities for the same user $i$ over all items $j$. The samples belong to the test set of \textbf{ML100K}.}
\label{fig:interpretability}
\end{figure*}

\begin{figure*}
\centering
\includegraphics[width=0.85\textwidth]{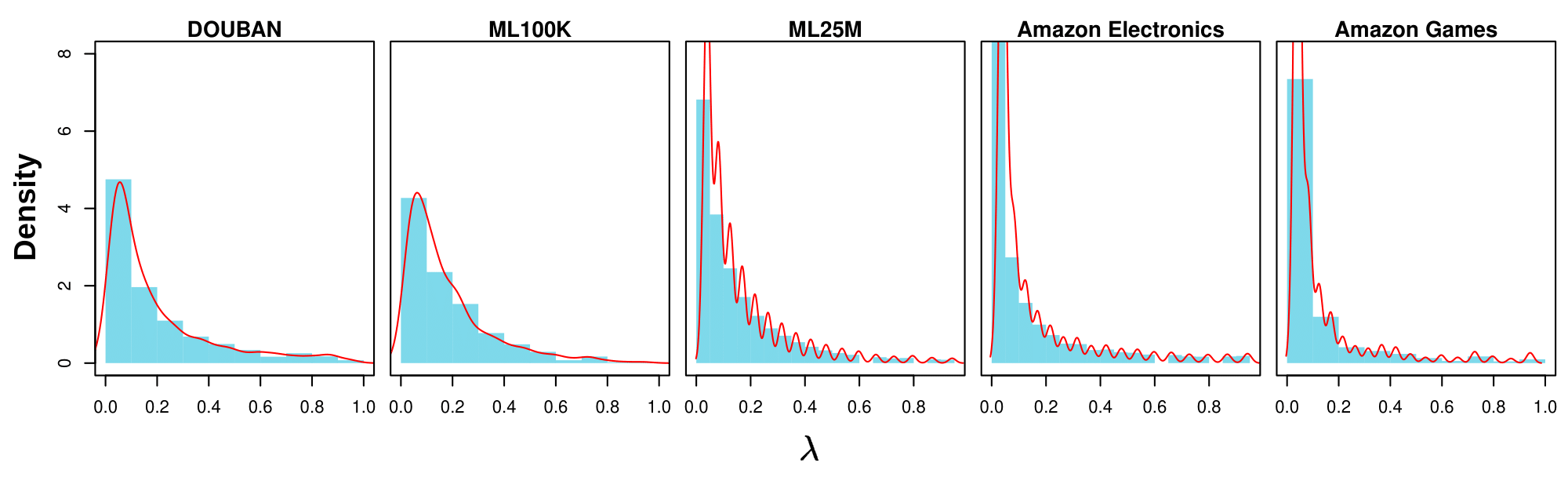}
\caption{\textcolor{black}{Density of indication $\lambda$ across the datasets. Further explanation in the main text.}} 
\label{fig:lambdas}
\end{figure*}

In all cases, the cited parameters follow the notation of the respective papers. For the explicit feedback baselines we only consider the contribution of observed ratings in the loss function. In the case of implicit feedback counterparts we set an entry $(i,j)$ to $1$ if user-$i$ has interacted with item-$j$ and to $0$, otherwise. {\color{black}The validation procedure was performed, with RMSE utilized for explicit feedback tasks and Recall$@50$ for implicit feedback tasks.}

{\color{black}
\noindent  \textcolor{black}{\textbf{Architectural and implementation details for Conv4Rec (\textcolor{black}{including hyperparameter ranges}):} We employ a simple architecture where the encoder $\phi_{\theta_1}$ consists of two layers, one of which is convolutional ($L=1$, and $L_0=1$). Regarding the decoder $g_{\theta_2}$, we validate $L \in \{2,3,5,9\}$ and keep $L_0=1$. We tune the embedding dimension $r \in \{8,16,32,\ldots,128\}$ and the size of the convolutional filter $K_2 \in \{6,12,16\}$. We run our method in batches of $40$ epochs (in batches of $10$ epochs each) and select the best set of hyperparameters. With the best set of hyperparameters, we run our method in batches of $10$ epochs and stop when the validation error increases. For the prediction, we average the predictions of all batches of $10$ epochs. Our implementation uses Keras with Tensorflow as backend and Adam with Nesterov momentum as the optimizer.}

}

\noindent \textbf{Real-world datasets}: for our experiments we selected the following three recommender systems datasets. 
\begin{itemize}
\item \textbf{Douban}: in Douban~\cite{zhu2019dtcdr}, users are social network members, items are movies; 
\item \textbf{MovieLens}
\footnote{Available at: \url{https://grouplens.org/datasets/movielens/}}
: a standard collection of RS datasets. We consider the following versions:
\begin{itemize}
	\item MovieLens~100K (\textbf{ML100K}) and 
	\item MovieLens~25M (\textbf{ML25M}).
\end{itemize} 
\textcolor{black}{\item \textbf{Amazon}: Amazon is as a large benchmark dataset~\cite{ni2019justifying} for recommendation systems in the e-commerce domain. We have selected the following popular subsets:
	\begin{itemize}
		\item \textbf{Amazon Electronics} and 
		\item \textbf{Amazon Games}.
\end{itemize} }
\end{itemize}
In all datasets (explicit feedback case), entry  $(i,j)$ is the rating user $i$ gives to movie $j$ on a scale between $1$ and $5$ stars. For both datasets, we split the set of observed entries three times into a training set (90 \%), a validation set (5\%) and a test set (5\%). {\color{black}The splitting procedure was repeated independently three times, resulting in three distinct dataset splits for evaluation in our experimental procedure. The reported results are averaged over all splits.}

%\begin{itemize}

%	\item \textbf{Douban}:In Douban~\cite{zhu2019dtcdr}, users are social network members, items are movies.

%   \item \textbf{MovieLens}\footnote{Available in: \url{https://grouplens.org/datasets/movielens/}}: it is a standard collection of RS datasets. We consider the versions MovieLens~100K (\textbf{ML100K}) and MovieLens~1M (\textbf{ML1M}). 

%\end{itemize}

{\color{black}
\noindent \textbf{Results for the explicit feedback:} For the explicit feedback evaluation metric, we select the Root Mean Squared Error (RMSE). Lower RMSE indicates better performance. Our method outperforms all the baselines in \textbf{DOUBAN}, \textbf{ML100K} \textcolor{black} {and both \textbf{Amazon} datasets}, as shown in Table~\ref{tab:explicit}. 

\noindent \textbf{Results for the Implicit Feedback:} \textcolor{black}{As metrics for the implicit feedback assessment, we selected the standard versions of Recall$@50$ and Recall$@100$. For both metrics, higher measurements indicate better performance. As seen in Table~\ref{tab:implicit}, our method outperforms the baselines in the majority of the analyzed scenarios. We highlight that our method particularly excels in \textbf{MovieLens 25M} and \textbf{Amazon Electronics}, the largest datasets of our study.} 
%in terms of all metrics, except Recall$@50$ for ML100K, where it ranks second, \emph{slightly} behind the \textbf{\mbox{LightGCN}} method. 
}

\noindent \textbf{Remark:}  \textcolor{black}{As seen in Tables~\ref{tab:implicit} and~\ref{tab:explicit}, our method achieves competitive performance at both the implicit and explicit feedback prediction tasks. This is especially noteworthy given that most baselines cannot be evaluated on both tasks.} % For instance, although GNNs could theoretically provide scores which can be used to predict implicit feedback, the computational complexity of the forward pass and the search through all items makes this impractical.  
%In addition, our model was trained using only the explicit feedback loss (RMSE): none of our results were cross-validated based on implicit feedback performance, which makes our implicit feedback performance particularly noteworthy. 

\noindent  \textcolor{black}{\textbf{Advantages of our Distributional Predictions and Applications:} Our model's output includes probabilities for each possible rating, providing additional qualitative information without the need for postprocessing. For instance, each of the four graphs in the left panel of Figure~\ref{fig:interpretability} presents the probabilities of possible ratings ($k \in \{1,2,\cdots, 5\}$), conditional on observing a given entry $(i,j)$. These entries correspond to estimated ratings $F_{i,j}$ that are close to the ground truth $R_{i,j}=4$ and were sampled from the test set of the \textbf{Douban} dataset. Therefore, these four samples represent excellent predictions in terms of explicit feedback.}

\textcolor{black}{However, these predictions vary in their degrees of certainty and provide different qualitative insights. The top right prediction indicates that the user is likely to enjoy the item, with higher ratings being more probable. Conversely, the bottom right graph shows a prediction with a high degree of certainty that the rating will be \textit{exactly} 4, rather than a lower or higher rating.  \textcolor{black}{As an application, }this information allows the recommendation system provider (or the user) to select which item to recommend based on the current strategy. For instance, if the recommender is in an \textit{exploration} phase and needs to gather more information about users in an active learning fashion, recommendations corresponding to less certain predictions, like those in the top right, would be more suitable as they would provide more informative feedback. On the other hand, if the aim is user retention, the bottom-right recommendation is the most important among the four examples, as it aligns closely with the user’s taste with a high degree of certainty.}

%\noindent \textbf{Interpretability and applications:} Our model's output includes probabilities for each possible rating. This is an advantage in itself, since it provides additional qualitative information without any postprocessing. For instance, each of the four graphs in the left Figure~\ref{fig:interpretability} presents probabilities of the possible rating ($\fix_k \in \{1,2,\cdots 5\}$), conditionally that a given entry $(i,j)$  is observed. The entries correspond to estimated ratings $F_{i,j}$ of around close to the ground truth $R_{i,j}=4$, and they were sampled from the test set of the \textbf{Douban} dataset. Therefore, these four samples correspond excellent predictions in terms of explicit feedback. 

%Note that, however, they have different degrees of certainty and correspond to different qualitative predictions. The top right prediction tells us that the user is likely to enjoy the item and that higher ratings are more likely. The bottom right graph, on the other hand,  corresponds to a prediction with a high degree of certainty that the rating will be \textit{exactly} $4$ rather than a lower rating or a rating of $5$.

The four graphs on the right correspond to predictions for two entries sampled from \textbf{ML100K}'s test set. The boxplots show, for a given user-$i$, the distribution of the probabilities of interaction for all items. The horizontal red line shows the probability of interaction for item-$j$ displayed on the graph immediately to the left. The top half graphs correspond to an item with a very high probability of interacting with the user, where the most likely rating is $1$ by a large margin. This would typically correspond to a blockbuster movie that a significant part of the users is likely to interact with but may not be well adapted to user-$12$'s particular tastes. In the bottom two graphs, we have the opposite situation, which is much more interesting in terms of generating value for RS users and recommendation with novelty: item-$120$ is very unlikely to be observed by user-$545$ naturally, but very likely to be enjoyed by them if they are given the chance to watch it. 
\textcolor{black}{This would be an ideal item to recommend in scenarios where the content provider aims to offer content that is novel to the user in an unexpected manner, enhancing user engagement and satisfaction by introducing them to items they are likely to enjoy but would not have discovered on their own. This approach not only broadens the user’s exposure to the content catalog but also increases the likelihood of uncovering hidden quality content, thereby fostering a more personalized and enriching user experience. Additionally, such recommendations can help in diversifying the consumption patterns, reducing the focus on popular items and promoting a wider array of content.}

\textcolor{black}{  \textcolor{black}{\textbf{Experimental Study on Modified Ranking Strategies}}: Finally, we will discuss how different users relate to implicit and explicit feedback in our model, which is capable of optimizing both types of feedback jointly. To achieve this, we perform additional experiments and compute an indicator $\lambda$, which is related to the ranking strategy. To better understand this, consider the following toy example where our model outputs, for a given user, the vectors $[0.2, 0.0, 0.8, 0.0, 0.0, 0.0]$ and $[0.7, 0.0, 0.0, 0.0, 0.0, 0.3]$ for items $a$ and $b$, respectively. Item $a$ has a high probability of being interacted with (implicit prediction $\mathcal{I}_a = 1 - 0.2 = 0.8$) but with a low predicted rating ($\mathcal{E}_a = 2$, after normalization). Conversely, the probability of the user liking item $b$ is high since we estimate a rating of five ($\mathcal{E}_b = 5$) with high certainty if the user interacts, but with a low probability of interaction ($\mathcal{I}_b = 0.3$). As discussed in this section, the decision on how to perform the recommendation is related to the strategy of the providers (and possibly of the users). However, it is reasonable that for any item, the final prediction score $F_i$ could be obtained by varying a parameter $\lambda$ in}  \textcolor{black}{$F_i(\lambda)  = \mathcal{I}_i + \lambda \mathcal{E}_i.$} \textcolor{black}{Note that if $\lambda$ is high, more attention is given to the \textit{explicit} rating behavior of the user. With $\lambda = \infty$, only explicit feedback is analyzed for every user. On the other hand, when $\lambda = 0$, only implicit feedback is considered. To better understand the dynamics of the datasets, we cross-validated $\lambda$ for each user across a spectrum of 100 values exponentially distributed in the interval $[0, 10^3]$. The density of the parameter distribution is shown in Figure~\ref{fig:lambdas}. We only plot the results between $[0,1]$ since it encompasses more than $95\%$ of the users in all datasets.  Some conclusions can be made: first, although the majority of users have better relevant items retrieved via implicit scores, there exists a non-negligible percentage of users for whom explicit feedback is also relevant for the retrieval procedure. Second, larger and sparser datasets, such as \textbf{Amazon} and \textbf{ML25M}, rely more on implicit feedback, since they are concentrated with values of $\lambda$ close to zero. Third, explicit taste is more important for movie recommendations than for consumption in e-commerce, where items with high popularity and trending can make huge impact on the dynamics of recommenders.}

\noindent \textbf{Conclusion:} We introduced Conv4Rec, a convolutional autoencoder which incorporates both implicit and explicit feedback prediction into a single training procedure. The output of our model consists in a set of probabilities for each of the possible ratings as well as for the absence of an observation, providing  \textcolor{black}{ more information than a simple rating prediction. Indeed, our model can estimate likelihood of interaction and expected rating separately. Beyond estimating the uncertainty in ratings estimation, the predictions can also identify items a user may enjoy but would be unlikely to interact with naturally}. We prove both parameter-counting and norm-based generalization bounds which illustrate the statistical gains from the convolutional structure of our model.  \textcolor{black}{We also show how to bound the total variation distance between the estimated and ground truth distributions over sampled interactions.  Experiments demonstrate} that our method consistently performs at  \textcolor{black}{a competitive level} in both implicit and explicit feedback prediction tasks. \textcolor{black}{Future directions include studying more complicated feedback types such as text, which would require substantial architectural modifications in the form of other layer types to replace our convolutions. } %Conv4Rec also provides interpretable predictions. For instance, our method is uniquely able to identify items that are both likely to be enjoyed by the user and unlikely for the user to encounter naturally.

% TEST equation~\eqref{eq:therealloss}
%Several considerations can be explored in future work.  This provides an  avenue to direct the exploration phase of a reinforcement learning strategy. 

\section*{Acknowledgements}

AL and PK's research was supported by the Singapore Ministry of Education (MOE) Academic Research Fund (AcRF) Tier 1 grant.  RA and PK thank Recombee for supporting their research. PK's  research was also supported by the Grant Agency of the Czech Technical University in Prague under grant SGS23/210/OHK3/3T/18.
%\bibliographystyle{ieeetr}
%\bibliography{bibliography}

\numberwithin{equation}{section}
\numberwithin{theorem}{section}
\numberwithin{figure}{section}
\numberwithin{table}{section}
\renewcommand{\thesection}{{\Alph{section}}}
\renewcommand{\thesubsection}{\Alph{section}.\arabic{subsection}}
\renewcommand{\thesubsubsection}{\Roman{section}.\arabic{subsection}.\arabic{subsubsection}}
\setcounter{secnumdepth}{-1}
\setcounter{secnumdepth}{3}

 %		\onecolumn
 %\clearpage 
 %\afterpage{\onecolumn}
 \begin{center}
 	\begin{table*}
 		\centering
 		\begin{tabular}{|c|c|} \hline
 			Notation & Meaning \\ \thickhline 
 			$\|\nbull\|$ & Spectral Norm \\ \hline 
 			$\|\nbull\|_{\Fr}$  &Froebenius Norm \\ \hline 
 			$m$ (resp. $n$)& Height (resp. width) of Ground Truth Matrix \\ \hline 
 			$[m]$ (resp. $[n]$) &  $\{1,2,\ldots,m\}$ (resp. $\{1,2,\ldots,n\}$)\\ \hline 
 			$\mathcal{D}$ & Sampling distribution of entries over $[m]\times [n]$ \\ \hline 
 			$R\in\mathbb{R}^{m\times n}$ & Ground Truth Matrix  \\\hline
 			$N$ & Number of Observed Samples\\ \hline 
 			$\Omega=\{\xi_1,\xi_2,\ldots,\xi_N\}$ & Set of observed entries\\\hline
 			$R_\Omega \in \R^{m\times n} $  & Matrix of observed Interactions (Ratings) \\
 			$(R_{\Omega})_{i,j}=\kappa$ if $R_{i,j} (+\zeta_{i,j})=\fix_\kappa$ & \\
 			$(R_{\Omega})_{i,j}=0$ \quad if unobserved & \\ \hline 
 			$U_1,U_2,\ldots,U_m\in \R^{n\times (k+1)}$ & One-hot encoded rows of $R_\Omega$, (input to encoder) \\ \hline 
 			$k$ & Number of Possible Ratings (e.g. 5)\\ \hline
 			$\fix_1<\fix_2<\ldots<\fix _k$ & Possible ratings (e.g. $\{1,2,3,4,5\}$)\\ \hline 
 			$\fixx=\fix_k-\fix_1$ & Maximum span between ratings (e.g. $5-1=4$)\\ \hline 
 			$\phi_{\theta_1}:\R^{n\times (k+1)}\rightarrow \R^{r}$ & Encoder CNN \\ \hline 
 			$D_1$ & Number of Parameters of encoder CNN\\ \hline
 			%	 $\phi^1_{\theta_1}:\R^{n\times (k+1)}\rightarrow \R^{n\times \kone}$ & First, convolutional component of the encoder network \\ \hline  
 			%  $\phi^2_{\theta_1}: \R^{n\times \kone}\rightarrow \R^{r}$ & Second, fully-connected component of the encoder network \\ \hline
 			$\theta_1\in\R^{D_1}$ & Parameters of encoder CNN\\\hline
 			$g_{\theta_2}:\R^{r}\rightarrow \R^{n\times (k+1)}$ & Decoder CNN\\ \hline 
 			$D_2$ & Number of Parameters of decoder CNN $g_{\theta_2}$\\ \hline
 			$L$ & Depth of decoder CNN $g_{\theta_2}$ \\ \hline
 			$L_0$ & Number of fully connected layers of decoder CNN $g_{\theta_2}$ \\ \hline
 			%	$g^1_{\theta_2}:\R^r\rightarrow \R^{n\times \ktwo}$  &  First, fully connected component of $g_{\theta_2}$\\ \hline  
 			%	$g^2_{\theta_2}:\R^{n\times \ktwo}\rightarrow \R^{n\times (k+1)}$ & Second, Convolutional component of  $g_{\theta_2}$ \\ \hline 
 			$\theta_2\in\R^{D_2}$ & Parameters of decoder Network\\ \hline
 			$\R^{m\times r} \ni X  = \phi_{\theta_1}(U) $ & Matrix of user embeddings\\ \hline
 			$\R^{m\times n \times (k+1)}\ni g_{\nbull,\nbull,\nbull}$ &  Predicted scores  \\ \hline
 			$\R^{m\times n \times (k+1)} \ni G=\text{Softmax}(g_{\theta_2}(\phi_{\theta_1}(U)))$ & Predicted Probabilities  \\ \hline 
 			$F_{i,j}=f(g_{i,j,\nbull})=\sum_{\kappa =1}^k \frac{\fix_{\kappa} G_{i,j,\kappa}}{\sum_{\kappa'=1}^k G_{i,j,\kappa'}}$ & Predicted ratings \\\hline
 			$f:\R^{n\times (k+1)}\rightarrow \R^{n}:g_{i,j,\nbull}\mapsto \sum_{\kappa =1}^k \frac{\fix_{\kappa} \exp(g_{i,j,\kappa})}{\sum_{\kappa'=1}^k \exp(g_{i,j,\kappa'})}$  &  Row-wise softmax + expectation \\ \hline 
 			$K_{\ell}$ & Channel width at layer $\ell$\\\hline 
 			$X^{\ell}_i$ & $\ell$th layer activation \\
 			&  of decoder network for user $i\leq m$ \\\hline 
 			$\widetilde{G}_{i,j,\kappa}$ &  $ \sum_{\kappa=1}^k \frac{\kappa  G_{i,j,\kappa}}{\sum_{\kappa'=1}^k G_{i,j,\kappa'}}$  \\ \thickhline 
 			$W^\ell\in\R^{K_{\ell}}\times \R^{K_{\ell-1}}$ ( for $\ell=1,2,\ldots,L$) & Weights (filter) at layer $\ell$\\ 
 			(if $\ell\neq L_0$)  &  \\ \hline 
 			$W^\ell\in\R^{n\times K_{\ell}}\times \R^{K_{\ell-1}}$ (for $\ell=1,2,\ldots,L$) & Weights (filter) at layer $\ell$\\ 
 			(if $\ell= L_0$)  &  \\ \hline 
 			$M^\ell$ ( for $\ell=1,2,\ldots,L$) & Initial weights  at $\ell$\\ \hline 
 			$\op(W)$ & Matrix Representing the convolutional \\
 			($\op(W^\ell)=W^\ell$ if $\ell$ fully-connected) & operation performed by the filter $W$  \\ \hline
 			$\kap =\{e_0,e_1,e_2,\ldots,e_{k}\}\subset \{0,1\}^{(k+1)}$ & Set of one-hot encodings of ratings (or lack thereof) \\ \thickhline 
 			$\elc: \kap \times [0,1]^{(k+1)}\rightarrow \R  : (e_\kappa,p)\rightarrow -\log(p_{\kappa })$ & Cross-Entropy Loss function   \\ \hline 
 			$\elr: [\fix_1,\fix_2]^2 \rightarrow [0,\fixx^2]: (y_1,y_2)\rightarrow \left|y_1-y_2\right|^2$ & Square loss function  \\ \hline 
 			$\Elc=\frac{1}{mn}\sum_{i=1}^m\sum_{j=1}^n \elc((R_\Omega)_{i,j},G_{i,j,\nbull})$ & Aggregated Cross-entropy loss \\ \hline 
 			$\rmseemp=\frac{1}{N}\sum_{(i,j)\in\Omega}\elr(R_{i,j}+\zeta_{i,j},F_{i,j})$ & Empirical RMSE \\ \hline 
 			$\rmsepop=\mathbb{E}_{\mathcal{D}}\left(\elr(R_{i,j}+\zeta_{i,j},F_{i,j})\right)$ & Population RMSE \\ \thickhline 
 			\textcolor{black}{$\widehat{	\mathcal{L'}}(g)$}& $\frac{1}{N}\sum_{(i,j)\in\Omega} -\left[\log(G_{i,j,	[R_{\Omega}]_{i,j}})-\log(G_{i,j,0})  \right] -\frac{1}{N}\sum_{(i,j)\in[m]\times [n]}\log(G_{i,j,0})$  \\
 			&  (Loss in Theorem~\ref{thm:implicitparam}, equals $\mathcal{L}(g)$ if no repeated entries)\\ \hline 
 			\textcolor{black}{$	\mathcal{L}'(g)$}& $ \E_{(i,j,\kappa)\sim \mathcal{D}} -\left[\log(G_{i,j,\kappa	})-\log(G_{i,j,0})  \right]- \frac{1}{N}\sum_{(i,j)\in[m]\times [n]}\log(G_{i,j,0})$ \\ & (Population version of $\widehat{	\mathcal{L'}}(g)$)\\ \thickhline 
 			$\|(x,\theta_2)\|_{\nn} =\|x\|+\sum_{\ell=1}^L \left\|W^\ell-W^\ell\right\|_{\ell}$ &  Witness norm to Lipschitz parametrization of $g_{\theta_2}$ \\ \hline 
 			$\|\nbull\|_{\ell}$ (for $\ell\leq L$) & Maximum norm of convolutional patch at layer $\ell$ of $g$ i.e., \\ 
 			$\|\nbull\|_{\ell}=\|\nbull\|_{\Fr}$ & if $\ell\geq L_0$ ($\ell$'th layer is fully connected)\\
 			$\|y\|_{\ell}=\max_{j\leq n}\|y_{j,\nbull}\|_{\Fr}$ & if $L>\ell> L_0$  \\
 			$\|y\|_{\ell}=\|y\|_{\infty}$ & if $L=\ell$  \\ \hline 
 			$\|W^\ell\|_{\ell}=\|W^\ell\| $ (if $\ell\neq L_0,L$)  &  Lipschitz constant of layer $\ell$ with weights $W^\ell$ \\
 			$\|W^{\ell}\|_{\ell}= \max_{j\leq n}\left\|W^{\ell}_{j,\nbull,\nbull}\right\|$ (if $\ell=L_0$)
 			& With respect to the norms $\|\nbull\|_{\ell-1}$ and $\|\nbull\|_{\ell}$ \\
 			$\|W^{\ell}\|_{\ell}=\left\| (W^{\ell})^\top\right\|_{2,\infty}$ (if $\ell=L$)  &   \\\hline 
 			$ \|X\|_x= \max_{i\leq m} \|X_{i,\nbull}\|$ & Max norm on set of encodings \\ \hline
 			$\|\theta\|_w=\sum_{\ell\leq L} \|W^\ell-M^\ell\|_\ell$ &  Norm on parameter space \\ \hline
 			%	$\|\nbull\|_{\infty-}$ &  $ \|(X,F)\|_{\infty-}=\max(\max_i\|X_{i,\nbull}\|,\|F\|_{\infty})$ \\ \hline
 			$\widehat{p}_{\nbull,\nbull,\nbull}$&$=\frac{\widehat{G}_{\nbull,\nbull,\nbull}}{\mathcal{G}}$\\ & normalized distribution over $[m]\times [n]\times [k]$ obtained from $\widehat{G}$. \\ \hline 
 		\end{tabular}
 		\caption{Table of common notations in the paper}
 	\end{table*}
 \end{center}
 
 \twocolumn

 \onecolumn
 %Table of Notation for bounds space bounds: 
 
 \begin{center}
 	\begin{table}
 		\centering
 		\begin{tabular}{|c|c|} \hline
 			Notation & Meaning \\ \thickhline 
 			$\beta$ & Upper bound on $\sum_{\ell=1}^L \left\| W^\ell-M^\ell\right\|_{\ell}$ \\ \hline
 			$1+\nu$ & Upper Bound on $\|M^\ell\|_{\ell}$ for all $\ell\leq L$\\ \hline 
 			$\chi$ & Upper bound on $\phi_{\theta_1}(U_i)$ for any $i\leq n$\\ \hline 
 			$a_\ell$ & Upper bound on $\|(W^\ell-M^\ell)^\top\|_{2,1}$ \\ \hline 
 			$s_\ell$ & Upper bound on $\|W^\ell\|_{\ell}$ \\ \hline 
 			$\mathcal{B}_{\chi}$  & $\{x\in\R^r:\|x\|\leq \chi \}$ \\ \hline 
 			$\mathcal{B}^m_{\chi}$ &$\left\{X\in\R^{m\times r}:\|X_{i,\nbull}\|\leq \chi  \quad \forall i\leq m \right\}$  \\ \hline 
 			$\mathcal{W}_{\beta,\nu}$ & $\theta_2$s satisfying~\eqref{cond:M} and~\eqref{cond:W} \\ \hline 
 			$\mathcal{F}_{\chi,\beta,\nu}:=$ &$ \left\{ F\in\R^{m\times n}: \exists \theta_2\in R^{D_2}, X\in\R^{m\times r} \quad \text{s.t.} \quad X\in \mathcal{B}_{\chi}^m  \quad \land \quad \theta_2 \in \mathcal{W}_{\beta,\nu}\quad \land \quad F_{i,j}=   \sum_{\kappa =1}^k \frac{\fix_{\kappa} \exp(g_{\theta_2}(x_i)_{j,\kappa})  }{\sum_{\kappa'=1}^k \exp(g_{\theta_2}(x_i)_{j,\kappa'})}  \right\}$ \\ \hline 
 			$\mathcal{W}_{a_\nbull ,s_\nbull }$ & Set of weights $\theta_2$ s.t.  \\ &
 			$\|W^\ell\|_\ell\leq s_\ell \quad \forall \ell\leq L$  \quad and \\ 
 			& $\|W^\ell-M^\ell\|_{2,1}\leq a_\ell  \quad \forall \ell\leq L.$  \\ \hline 
 			$\mathcal{F}_{\chi, a_{\nbull},s_\nbull}:=$  & {\small $ \left\{F\in\R^{m\times n}: \exists \theta_2\in R^{D_2}, X\in\R^{m\times r} \quad \text{s.t.} \quad X\in \mathcal{B}^m_{\chi}  \quad \land \quad \theta_2 \in \mathcal{W}_{a_\nbull ,s_\nbull }\quad \land \quad F_{i,j}=   F_{i,j}=   \sum_{\kappa =1}^k \frac{\fix_{\kappa} \exp(g_{\theta_2}(x_i)_{j,\kappa})  }{\sum_{\kappa'=1}^k \exp(g_{\theta_2}(x_i)_{j,\kappa'})}   \right\}$ } \\ \hline 
 			$	\widehat{	\mathcal{L}}(g)$ & $ \frac{1}{N}\sum_{(i,j)\in\Omega} -\left[\log(G_{i,j,	[R_{\Omega}]_{i,j}})-\log(G_{i,j,0})  \right] -\textcolor{black}{\con}\sum_{(i,j)\in[m]\times [n]}\log(G_{i,j,0}).$\\ \hline 
 			$\mathcal{L}(g)$ & $\E_{(i,j)\sim \mathcal{D}} -\left[\log(G_{i,j,	y_{i,j}})-\log(G_{i,j,0})  \right] -\textcolor{black}{\con}\sum_{(i,j)\in[m]\times [n]}\log(G_{i,j,0}).$ \\ \hline 
 			$\hat{g}$ & Minimizer of $\widehat{\mathcal{L}}$ (subject to relevant constraints, Eq~\eqref{cond:MWimpl})\\ \hline 
 			$g^*$ & Minimizer of $\mathcal{L}$ (subject to relevant constraints, Eq~\eqref{cond:MWimpl})\\ \hline 
 			
 			$\mathcal{G}_{\chi,\beta,\nu}$&  Set of $g\in\R^{m\times n\times (k+1)}$satisfying Eq.~\eqref{cond:MWimpl} \\ \hline
 			$B$ & Upper bound on $	\|g_{\theta_2}(\phi_{\theta_1}(U_i))_{j,\kappa}\|$ \\  
 			& (cf. also Eq.~\eqref{eq:defineparamG}) \\ \hline 
 			$	\mathcal{G}_{\chi, a_{\nbull},s_\nbull}:=$&  $	\left\{g\in\R^{m\times n\times (k+1)}: \exists \theta_2\in R^{D_2}, X\in\R^{m\times r} \quad \text{s.t.} \quad X\in \mathcal{B}^m_{\chi}  \quad \land \quad \theta_2 \in \mathcal{W}_{a_\nbull ,s_\nbull }\quad \land \quad  g_{i,j,\kappa}=g_{\theta_2}(x_i)_{j,k}  \right\}$ \\ \hline 
 			$	\imploss(g,y)$  &$ \log\left(  \frac{\exp(g_0)}{\exp(g_y)}        \right)=g_0-g_y$ \\
 			& (cf. Equation~\eqref{eq:imploss}) \\ \hline 
 			$ S$ & $ \max_{\ell\leq L}\prod_{l=\ell}^L s_{l}\leq \prod_{\ell=1}^L \left[s_{\ell}+1\right]$ \\\hline 
 			$\uS$ & $\left[\prod_{\ell=1}^L s_\ell\right]$ \\  \thickhline
 			$p_{i,j,\kappa}$ & Probability of sampling entry $(i,j)$ with rating $\fix_\kappa$ \quad \quad (for $\kappa\neq 0$)\\
 			$p_{i,j,0}$  & $1-\sum_{\kappa=1}^k p_{i,j,\kappa}$\\\hline 
 			\textcolor{black}{$G^{\text{Bayes}}$} & Ground truth minimizer of loss (cf Eq.~\eqref{eq:expressbayes}) \\ 
 			&	$	G^{\text{Bayes}}_{i,j,\kappa} = N p_{i,j,\kappa} $ \quad for \quad $ \kappa\neq 0$\\
 			&$G^{\text{Bayes}}_{i,j,0} =1-\sum_{\kappa=1}^k Np_{i,j,\kappa}  $ \quad  otherwise\\ \hline
 			$\Gall$ & $\sum_{i=1}^m \sum_{j=1}^n \widehat{G}_{i,j}$\\ \thickhline
 			$\underline{W}^{\ell}$& $\ell$th layer weight of encoder \\ \hline
 		\end{tabular}
 		\caption{Table of notations relevant to the theoretical results}
 	\end{table}
 \end{center}
 \twocolumn

 \onecolumn
 %Table of Notation for bounds space bounds: 
 
 \begin{center}
 	\begin{table}
 		\centering
 		\begin{tabular}{|c|c|} \hline
 			Notation & Meaning \\ \thickhline 
 			$\beta$ & Upper bound on $\sum_{\ell=1}^L \left\| W^\ell-M^\ell\right\|_{\ell}$ \\ \hline
 			$1+\nu$ & Upper Bound on $\|M^\ell\|_{\ell}$ for all $\ell\leq L$\\ \hline 
 			$\chi$ & Upper bound on $\phi_{\theta_1}(U_i)$ for any $i\leq n$\\ \hline 
 			$a_\ell$ & Upper bound on $\|(W^\ell-M^\ell)^\top\|_{2,1}$ \\ \hline 
 			$s_\ell$ & Upper bound on $\|W^\ell\|_{\ell}$ \\ \hline 
 			$\mathcal{B}_{\chi}$  & $\{x\in\R^r:\|x\|\leq \chi \}$ \\ \hline 
 			$\mathcal{B}^m_{\chi}$ &$\left\{X\in\R^{m\times r}:\|X_{i,\nbull}\|\leq \chi  \quad \forall i\leq m \right\}$  \\ \hline 
 			$\mathcal{W}_{\beta,\nu}$ & $\theta_2$s satisfying~\eqref{cond:M} and~\eqref{cond:W} \\ \hline 
 			$\mathcal{F}_{\chi,\beta,\nu}:=$ &$ \left\{ F\in\R^{m\times n}: \exists \theta_2\in R^{D_2}, X\in\R^{m\times r} \quad \text{s.t.} \quad X\in \mathcal{B}_{\chi}^m  \quad \land \quad \theta_2 \in \mathcal{W}_{\beta,\nu}\quad \land \quad F_{i,j}=   \sum_{\kappa =1}^k \frac{\fix_{\kappa} \exp(g_{\theta_2}(x_i)_{j,\kappa})  }{\sum_{\kappa'=1}^k \exp(g_{\theta_2}(x_i)_{j,\kappa'})}  \right\}$ \\ \hline 
 			$\mathcal{W}_{a_\nbull ,s_\nbull }$ & Set of weights $\theta_2$ s.t.  \\ &
 			$\|W^\ell\|_\ell\leq s_\ell \quad \forall \ell\leq L$  \quad and \\ 
 			& $\|W^\ell-M^\ell\|_{2,1}\leq a_\ell  \quad \forall \ell\leq L.$  \\ \hline 
 			$\mathcal{F}_{\chi, a_{\nbull},s_\nbull}:=$  & {\small $ \left\{F\in\R^{m\times n}: \exists \theta_2\in R^{D_2}, X\in\R^{m\times r} \quad \text{s.t.} \quad X\in \mathcal{B}^m_{\chi}  \quad \land \quad \theta_2 \in \mathcal{W}_{a_\nbull ,s_\nbull }\quad \land \quad F_{i,j}=   F_{i,j}=   \sum_{\kappa =1}^k \frac{\fix_{\kappa} \exp(g_{\theta_2}(x_i)_{j,\kappa})  }{\sum_{\kappa'=1}^k \exp(g_{\theta_2}(x_i)_{j,\kappa'})}   \right\}$ } \\ \hline 
 			$	\widehat{	\mathcal{L}}(g)$ & $ \frac{1}{N}\sum_{(i,j)\in\Omega} -\left[\log(G_{i,j,	[R_{\Omega}]_{i,j}})-\log(G_{i,j,0})  \right] -\textcolor{black}{\con}\sum_{(i,j)\in[m]\times [n]}\log(G_{i,j,0}).$\\ \hline 
 			$\mathcal{L}(g)$ & $\E_{(i,j)\sim \mathcal{D}} -\left[\log(G_{i,j,	y_{i,j}})-\log(G_{i,j,0})  \right] -\textcolor{black}{\con}\sum_{(i,j)\in[m]\times [n]}\log(G_{i,j,0}).$ \\ \hline 
 			$\mathcal{K}$ & constant in definition of the loss, $\mathcal{K}=\frac{1}{N}$ in our case\\
 			$	\mathcal{G}_{\chi, a_{\nbull},s_\nbull}:=$&  $	\left\{g\in\R^{m\times n\times (k+1)}: \exists \theta_2\in R^{D_2}, X\in\R^{m\times r} \quad \text{s.t.} \quad X\in \mathcal{B}^m_{\chi}  \quad \land \quad \theta_2 \in \mathcal{W}_{a_\nbull ,s_\nbull }\quad \land \quad  g_{i,j,\kappa}=g_{\theta_2}(x_i)_{j,k}  \right\}$ \\ \hline 
 			$	\imploss(g,y)$  &$ \log\left(  \frac{\exp(g_0)}{\exp(g_y)}        \right)=g_0-g_y$ \\
 			& (cf. Equation~\eqref{eq:imploss}) \\ 	 \thickhline
 			$p_{i,j,\kappa}$ & Probability of sampling entry $(i,j)$ with rating $\fix_\kappa$ \quad \quad (for $\kappa\neq 0$)\\
 			$p_{i,j,0}$  & $1-\sum_{\kappa=1}^k p_{i,j,\kappa}$\\\hline 
 			$G^{\text{Bayes}}$ & Ground truth minimizer of loss (cf Eq.~\eqref{eq:expressbayes}) \\ 
 			&	$	G^{\text{Bayes}}_{i,j,\kappa} = \frac{p_{i,j,\kappa}}{\mathcal{K}} $ \quad for \quad $ \kappa\neq 0$\\
 			&$G^{\text{Bayes}}_{i,j,0} =1-\sum_{\kappa=1}^k \frac{p_{i,j,\kappa}}{\mathcal{K}}  = 1- \frac{1-p_{i,j,0}}{\mathcal{K}}$ \quad  otherwise\\
 			\hline
 		\end{tabular}
 		\caption{Table of notations relevant to the theoretical results}
 	\end{table}
 \end{center}
 \twocolumn
 
 \section{Notations and General Considerations}
 \label{ap:notation}
 Our decoder network $g_{\theta_2}$, which takes as input a feature vector $\phi_{\theta_1}(U_i)\in\R^r$ corresponding to a user $i\leq m$, is a Convolutional Neural Network: the first $L_0$ layers are fully connected, and the last $L-L_0$ layers are convolutional, with each convolutional patch corresponding to an item. Let us write $g^1_{\theta_2}:\R^r\rightarrow \R^{n\times K_{L_0}}$ and $g^2_{\theta_2}:\R^{n\times K_{L_0}}\rightarrow \R^{n\times (k+1)}$ for the fully-connected and convolutional parts of the network respectively. Thus $g_{\theta_2}=g^2_{\theta_2}\circ g^1_{\theta_2}$, and for all $X\in \R^{n\times \ktwo}$ and $j\leq m$, $g^2_{\theta_2}(X)_{j,\nbull}=\tilde{g}^2_{\theta_2}(X)_{j,\nbull}$ for some single function $\tilde{g}_{\theta_2}:\R^{\ktwo} \rightarrow \R^{k+1}$ which is defined by a fully-connected neural network with weights $W^{L_0+1},W^{L_0+2},\ldots,W^{L}$. 
 
 %Note that $g^2_{\theta_2}$ (resp.$g^1_{\theta_2}$ ) only depends on parts of the components of $\theta_2$, however, to alleviate the notation, we omit to strictly distinguish those dependencies. This is because we will only need to count the total number of parameters in $g_{\theta_2}$ (i.e. the dimension $D_2$ of $\theta_2$). 
 It is assumed throughout the theoretical results that all activation functions are $1$-Lipschitz. As explained in the main paper, the ground truth matrix is denoted by $R\in\R^{m\times n}$. In the theoretical results, we make our results more general by allowing the observations to be corrupted by some i.i.d. noise $\zeta_{i,j}$. This means that the matrix $R_{\Omega}$ and tensor $U\in\R^{m\times n\times (k+1)}$ are constructed based on the values of $R_{i,j}+\zeta_{i,j}$. This has no influence on the proof techniques. As much of the comparable literature~\cite{spectre,LongSed,ledent2021normbased,graf}, our theoretical results involve the 'distance to initialization': throughout the paper and appendix $M^1, \ldots, M^L$ are weights of the same dimension as $W^1,\ldots,W^L$ which are fixed in advance.

 We aim to use Rademacher complexity type arguments to  show generalisation bounds for our model. In the case of explicit feedback, that corresponds to bounds on $\rmsepop-\rmseemp$ where 	$\rmseemp=\frac{1}{N}\sum_{(i,j)\in\Omega}\elr(R_{i,j}+\zeta_{i,j},F_{i,j})$ is the empirical RMSE and 
 $\rmsepop=\mathbb{E}_{\mathcal{D}}\left(\elr(R_{i,j}+\zeta_{i,j},F_{i,j})\right)$ is the population RMSE (here, $\elr$ denotes the square loss).  In the case of implicit feedback, we will instead be proving excess risk bounds defined with our loss function (cf. Section~\ref{sec:implicitbounds}). 
 In both cases, this requires care when handling the i.i.d. assumption: when Matrix Completion is viewed as a supervised machine learning problem for the purpose of Rademacher analysis, we must view matrices as functions from the set of all possible entries $[m]\times [n]$ to the set of possible predictions $[\fix_1,\fix_k]$. However, in our model, the neural network takes as input a whole slice $U_i$ of $U$ (which is a one-hot encoding of $R_{\Omega}$), which depends on the whole training set $\Omega$. Our solution to this problem is to ignore the specific structure of the encoding network $\phi_{\theta_1}$ and to focus on the fact that $\phi_{\theta_1}(R_{\Omega})\in \R^{m\times r}$ is a low-rank deep represention of our available information: we simply compute the covering number of a large ball in the space $\R^{m\times r}$ and rely on the fact that this space has a comparatively small number of dimensions. We then bound the covering number of the decoder network $g_{\theta_2}$ and compose the two covers to obtain a covering number of a function class $\mathcal{F}$ to which the functions (i.e. matrices) learned by our model must belong.  Note that the explicit constraint that our embedding space for each user has dimension $r$ (as opposed to using, for instance, nuclear norm regularisation) allows us to obtain $L^\infty$ covers, which leads to bounds which are valid regardless of the sampling distribution $\mathcal{D}$.
 
 \textbf{Remark:} It is worth noting that some impressive recent work shows that Neural Networks tend to learn functions of low `bottleneck rank', i.e. functions where some of the activations of the deeper layers lie in low-dimensional subspaces. Thus, it may not be necessary to restrict $r$ explicitly: if the ground truth has the suitable low-rank structure, a DNN will naturally converge towards a low bottleneck rank solution~\cite{BottleNeckRank}.

 \textbf{Remark:} As is standard in the literature, we  greatly simplify the exposition by using notation that assumes that each entry cannot be sampled more than once. However, since the entries are i.i.d., the same entry can be sampled several times, in which case the loss must be computed independently and summed for all occurrences of the entry. More rigorously, the observations are i.i.d. of the form $(\xi^o,\bar{\xi}^o)$ with $\xi^o\in \{1,2,\ldots,m\}\times \{1,2,\ldots,n\}$  and $\bar{\xi}^o\in [\fix_1,\ldots,\fix_k]$  and we write $\sum_{(i,j)\in \Omega} \Theta(R_{i,j} +\zeta_{i,j})$ instead of   $\sum_{o=1}^N \Theta(\bar{\xi}_o)$ for any function $\Theta$, and it should be assumed that the "ground truth" values $R$ are defined to satisfy $\E(\zeta_{i,j})=0$.
 
 \section{Parameter-count Based Bounds}
 \label{sec:paracount}
 We will write $\|\nbull\|_{\ell}$ for an architecture-dependent norm at layer $\ell$ which is equal to the maximum Frobenius norm of a convolutional patch, except at the last layer where we use an $L^\infty$ norm. More precisely, for any $x^\ell$ in the activation space of layer $\ell$:
 \begin{itemize}
 	\item $\|x^\ell\|_{\ell}=\|x^\ell\|_{\Fr}$  if $\ell\leq  L_0-1$ ($\ell+1$'th layer is fully connected);
 	\item $\|x^\ell\|_{\ell}=\max_{j\leq n}\|x^\ell_{j,\nbull}\|_{\Fr}$  if $ L_0\leq \ell\leq L-1$ (i.e. the $\ell+1$th layer is convolutional) and;
 	\item $\|x^\ell\|_{L}=\|x^\ell\|_{\infty}$ at the last layer.
 \end{itemize} 
 Similarly, we will write $\|W^\ell\|_{\ell}$ for the Lipschitzness constant of layer $\ell$ with the weights $W^\ell$ with respect to the norms $\|\nbull\|_{\ell-1}$ and $\|\nbull\|_{\ell}$. Hence, we trivially have $\|W^\ell\|_{\ell}=\|W^\ell\|$ if  $\ell\leq L_0-1$. In addition, we also have $ \|W^\ell\|_{\ell}=\|W^\ell\|$ if $L-1\geq \ell\geq  L_0+1$, since the convolutional operation consists in applying $W^\ell$ to each patch independently and both norms $\|\nbull\|_{\ell-1}$ and $\|\nbull\|_{\ell}$ are maxima over the patches (items). Interestingly,  for $\ell=L_0$ we have $\|W^{\ell}\|_{\ell}= \max_{j\leq n}\left\|W^{\ell}_{j,\nbull,\nbull}\right\|$, i.e. $\|W\|_\ell$ is equal to the maximum of the spectral norms of the $n$ slices of $W^\ell$ corresponding to each item.  Indeed, we have for all $x\in \R^{K_{\ell-1}}$ and for all $W^\ell\in\R^{n\times K_\ell\times K_{\ell-1}}$ we have 
 \begin{align}
 	&\sup_{\|x\|\leq 1}	\|	\op(W^\ell) x \|_{\ell}=\sup_{\|x\|\leq 1} \max_{j\leq n} \left\|\left(	\op(W^\ell) x\right)_{j,\nbull}\right \|\nonumber\\&= \sup_{\|x\|\leq 1}\max_{j\leq n} \left\|\left(W^\ell \right)_{j,\nbull,\nbull}x\right \|= \max_{j\leq n} \left\|W^{\ell}_{j,\nbull,\nbull}\right\|,  \nonumber 
 \end{align}
 as expected.
 In addition, at the last layer, we have for any $W^L\in \R^{(k+1)\times K_{L-1}}$, $\|W^L\|_L=\|(W^L)^\top\|_{2,\infty}$. Indeed, for all $(x_1,\ldots,x_n)\in\R^{n\times K_{L-1}}$ with $\|x_j\|\leq 1$ for all $j$, we have for all $j\leq n$ and for all $\kappa\leq k$:
 \begin{align*}
 	\left|\left(\op(W^L)\text{vec}(x_1,\ldots,x_n)\right)_{j,\kappa}\right|=\left|\left\langle  W^L_{\kappa,\nbull}, x_j\right\rangle  \right| \leq \left\|W^L_{\kappa,\nbull}\right\|\|x_j\|,
 \end{align*}
 with equality when $x_j=\frac{W^L_{\kappa,\nbull}}{\|W^L_{\kappa,\nbull}\|}$. Hence $\|W^\ell\|_\ell=\max_{\kappa} \|W_{\kappa,\nbull}\|=\|(W^L)^\top\|_{2,\infty}$, as expected. 
 
 Now that we have the properly defined norms, we can formulate the following result regarding the Lipschitz constant of the model w.r.t. the weights: 
 
 \begin{proposition}
 	\label{prop:littleLip}
 	Let $x,\tilde{x}\in \mathcal{B}_\chi:= B_{\R^r}(0,\chi)=\{\zeta\in\R^r:\|\zeta\|\leq \chi\}$ be two points in the ball of radius $\chi$ in $\R^r$ (those are user embeddings to be fed to the network $g_{\theta_2}$).  Let $\theta=(W^1,\ldots,W^L),\tilde{\theta}=(\widetilde{W}^1,\widetilde{W}^2,\ldots,\widetilde{W}^L)\in\R^{D_2}$ be two possible values of the parameter $\theta_2$ with the property that \begin{align}
 		\sum_{\ell\leq L}\|W^\ell-M^\ell\|_\ell \leq  \beta
 		\label{cond:W}
 	\end{align} and $\sum_{\ell\leq L}\|\widetilde{W}^\ell-M^\ell\|_\ell \leq \beta$.  We define the following norm on the space $\R^r\otimes \R^{D_2}$: 
 	
 	$$\|(x,\theta_2)\|_{\nn} =\|x\|+\sum_{\ell\leq  L}\|W^\ell-M^\ell\|_\ell.$$
 	
 	Assume that the initialization $M^1,\ldots,M^L$ satisfies \begin{align}
 		\left\|M^\ell\right\|_\ell\leq 1+\nu
 		\label{cond:M}
 	\end{align}  for all $\ell\leq L$. We have the following inequality:
 	\begin{align}
 		\label{eq:lipl}
 		&\left\| g_{\theta}(x)-g_{\tilde{\theta}}(\tilde{x})    \right\|_{\infty} \\ &\leq (\chi+1) \exp(\beta +L\nu ) \left\|(x,\theta)-(\tilde{x},\tilde{\theta})\right\|_{\nn}. \nonumber 
 	\end{align}
 \end{proposition}

 \begin{proof}[Proof of Proposition~\ref{prop:littleLip}]
 	Although the proof is somewhat analogous to those of Lemmas 2.5 and 3.4 in~\cite{LongSed}, we have following main differences: 
 	\begin{itemize}
 		\item We need to account for not only a different parameter $\theta_2$ but also a different input $x$.
 		\item The output of our network is a matrix in $\R^{n\times (k+1)}$ rather than a real number. The Lipschitz properties of our loss function will be taken into account later in Proposition~\ref{prop:fulllip}. 
 		\item Most significantly, need to rely on some architecture-specific norms, rather than the spectral standard $L^2$ norm.  Indeed, it is not enough to work with Frobenius norms at each layer since we require $L^\infty$ covering numbers of the matrices $g_{\theta_2}\circ \phi_{\theta_1}(U)$ with respect not only to each user but also to each item. This means that when viewing $g_{\theta_2}\circ \phi_{\theta_1}$ as a function from $\R^r$ to $\R^{n\times r}$,  the notion of distance at the last layer with respect to which the Lipschitzness must hold needs to involve a maximum over the $n$ rows of the output. 
 	\end{itemize}

 	Now, we proceed to estimate the sensitivity of the output with respect to changes in the parameters at any layer, including the input layer. 
 	
 	\textbf{Claim 1:}  Assume that $x=\tilde{x}$ and $\theta$ and $\tilde{\theta}$ only differ at layer $\ell$, then 
 	$$\left\| g_{\theta}(x)-g_{\tilde{\theta}}(\tilde{x})    \right\|_{L}\leq \chi \exp(\nu L+\beta) \left\| W^\ell -\widetilde{W}^\ell\right\|_{\ell}. $$
 	
 	\textbf{Proof of Claim 1:}
 	Since $W^l=\widetilde{W}^l$ for all $l\neq \ell$, we can write $g_{\theta}= g_{\textcolor{black}{\down}}\circ g_{W^\ell}\circ g_{\textcolor{black}{\upp}}$ and $g_{\tilde{\theta}}= g_{\textcolor{black}{\down}}\circ g_{\widetilde{W}^\ell}\circ g_{\textcolor{black}{\upp}}$ for two functions $g_{\textcolor{black}{\upp}}$ and $g_{\textcolor{black}{\down}}$ (depending on $W^l=\widetilde{W}^l$ for all the $l\neq \ell$) where $g_{\widetilde{W}^\ell}$ represents the convolution operation at layer $\ell$ composed with the non linearity.
 	We then have 
 	\begin{align}
 		&\left\| g_{\theta}(x)-g_{\tilde{\theta}}(\tilde{x})    \right\|_{L} \nonumber \\&=\left\|g_{\textcolor{black}{\down}}\circ g_{W^\ell}\circ g_{\textcolor{black}{\upp}}(x)- g_{\textcolor{black}{\down}}\circ g_{\widetilde{W}^\ell}\circ g_{\textcolor{black}{\upp}}(\tilde{x})\right\|_{\ell} \nonumber \\
 		&=\left\|g_{\textcolor{black}{\down}}\circ g_{W^\ell}\circ g_{\textcolor{black}{\upp}}(x)- g_{\textcolor{black}{\down}}\circ g_{\widetilde{W}^\ell}\circ g_{\textcolor{black}{\upp}}(x)\right\|_{\ell}\nonumber \\
 		&=\left\|g_{\textcolor{black}{\down}}\circ   \left(g_{W^\ell}-g_{\widetilde{W}^\ell}\right)\circ g_{\textcolor{black}{\upp}}(x)\right\|_{\ell}\nonumber \\
 		&\leq \|x\|   \|g_{\textcolor{black}{\down}}\|_{\ell\rightarrow L}\|g_{\textcolor{black}{\upp}}\|_{0\rightarrow \ell-1} \left\|W^\ell-\widetilde{W}^\ell\right\|_{\ell},\label{eq:3spec}
 	\end{align}
 	where $\|g_{\textcolor{black}{\down}}\|_{\ell\rightarrow L}$  (resp. $\|g_{\textcolor{black}{\upp}}\|_{0\rightarrow \ell-1}$) stands for the Lipschitz constant of $g_{\textcolor{black}{\down}}$ (resp. $g_{\textcolor{black}{\upp}}$) with respect to the norms $\|\nbull\|_\ell$ and $\|\nbull\|_L$ (resp.$\|\nbull\|_{\Fr}$ and $\|\nbull\|_{\ell-1}$).  Here, we have used the fact that the activation function at layer $\ell$ is component-wise and 1-Lipschitz, allowing us to deduce that $\left\|\left(g_{W^\ell}-g_{\widetilde{W}^\ell}\right)\right\|\leq \left\|W^\ell-\widetilde{W}^\ell\right\|_{\ell}$. Now, recall that by assumption we have for all $l\leq L$, $\|M^l\|\leq 1+\nu$. Therefore,  it holds that 
 	\begin{align*}
 		\left\| W^l\right\|_{\ell} \leq \left \|M^l\right\|_{\ell}+ \left\| W^l-M^l\right\|_{\ell} \leq 1+\nu+ \beta_{\ell},
 	\end{align*}
 	where $\beta_{\ell}$ stands for $ \left\| W^l-M^l\right\|_\ell$. 
 	Thus, continuing after Equation~\eqref{eq:3spec}, we have 
 	\begin{align}
 		\left\| g_{\theta}(x)-g_{\tilde{\theta}}(\tilde{x})    \right\|_{L} &  \leq \|x\| \|g_{\textcolor{black}{\upp}}\| \|g_{\textcolor{black}{\down}}\| \left\|W^\ell-\widetilde{W}^\ell\right\|_{\ell} \nonumber \\
 		&\leq \chi \prod_{l\geq \ell+1} \|W^l\|_l \prod_{l\leq  \ell-1} \|W^l\|_l   \left\| W^\ell -\widetilde{W}^\ell\right\|_{\ell}\nonumber \\
 		&\leq \chi \prod_{l\neq \ell} \|W^l\|_l  \left\| W^\ell -\widetilde{W}^\ell\right\|_{\ell} \nonumber \\
 		&\leq \chi \prod_{l\neq \ell}  (1+\nu+\beta_l) \left\| W^\ell -\widetilde{W}^\ell\right\|_{\ell}\nonumber \\
 		&\leq \chi \prod_{l=1}^L  (1+\nu+\beta_l) \left\| W^\ell -\widetilde{W}^\ell\right\|_{\ell}. \label{eq:beforelagrange}
 	\end{align}
 	
 	Recall also that we have  $ \sum_{\ell=1}^L\|W^\ell-M^\ell\|_{\ell}\leq \beta$,  and therefore  $\sum_{\ell=1}^L\beta_\ell\leq \beta$. Therefore, after optimizing over the $\beta_l$s, Equation~\eqref{eq:beforelagrange} becomes: 
 	\begin{align}
 		\left\| g_{\theta}(x)-g_{\tilde{\theta}}(\tilde{x})    \right\|_{L}  \nonumber & \leq \chi\left (1+\nu+\frac{\beta}{L}\right)^L \left\| W^\ell -\widetilde{W}^\ell\right\|_{\ell}\nonumber \\
 		&\leq \chi \exp(\beta +\nu L)  \left\| W^\ell -\widetilde{W}^\ell\right\|_{\ell}, \nonumber
 	\end{align}
 	which concludes the proof of Claim 1. 
 	
 	\textbf{Claim 2:} Assume that $\theta=\tilde{\theta}$, then we have  
 	$$\left\| g_{\theta}(x)-g_{\tilde{\theta}}(\tilde{x})    \right\|_{L}\leq \left\|x-\tilde{x}\right\|\exp(\nu L+\beta).$$
 	\textbf{Proof of Claim 2:}
 	Similarly to Equation~\eqref{eq:beforelagrange} we have:
 	\begin{align}
 		\left\| g_{\theta}(x)-g_{\tilde{\theta}}(\tilde{x})  \right\|_{L}&= 	\left\| g_{\theta}(x)-g_{\theta}(\tilde{x})  \right\|_{L} \nonumber \\
 		&\leq \left\|x-\tilde{x}\right\| \prod_{l=1}^L  (1+\nu+\beta_l)\label{eq:prodagain}.
 	\end{align}
 	Then, again, optimizing over the $\beta_l$s we can continue to obtain: 
 	\begin{align}
 		\left\| g_{\theta}(x)-g_{\tilde{\theta}}(\tilde{x})  \right\|_{L}
 		&\leq \left\|x-\tilde{x}\right\| \prod_{l=1}^L  (1+\nu+\beta_l) \nonumber \\
 		&\leq  \left\|x-\tilde{x}\right\| \left(1+\nu+\frac{\beta}{L}\right)^L \nonumber \\
 		&\leq \left\|x-\tilde{x}\right\|\exp(\nu L+\beta), \nonumber
 	\end{align}
 	which concludes the proof of Claim 2. 
 	
 	We can now finish the \textbf{proof of Proposition~\ref{prop:littleLip}}: 
 	
 	Note that since $\tilde{\theta}$ also satisfies $\|\widetilde{W}^\ell-M^\ell\|\leq 1+\nu$ for all $\ell$, Claims 1 and 2 both also hold with $\theta$ and $\tilde{\theta}$ reversed. 
 	Hence, transforming $\theta$ into $\tilde{\theta}$ one layer at a time and applying Claim 1 at each step, we can obtain: 
 	\begin{align}
 		&\left\| g_{\theta}(x)-g_{\tilde{\theta}}(x)  \right\|_L   \nonumber \\ & \leq      \chi \exp(L \nu +\beta )\sum_{\ell=1}^L \left \|W^{\ell} -\widetilde{W}^{\ell}\right\|_{\ell}. \label{eq:1stpiece}
 	\end{align}
 	
 	Together with Claim 2, this allows us to obtain: 
 	\begin{align}
 		&\left\| g_{\theta}(x)-g_{\tilde{\theta}}(\tilde{x})  \right\|_L \nonumber \\ &\leq  	\left\| g_{\theta}(x)-g_{\tilde{\theta}}(x)  \right\|_L+	\left\|g_{\tilde{\theta}}(\tilde{x}) -g_{\tilde{\theta}}(x) \right\|_L \nonumber  \\
 		&\leq    \chi \exp(L \nu +\beta )\sum_{\ell=1}^L \left \|W^{\ell} -\widetilde{W}^{\ell}\right\|_{\ell} +\left\|g_{\tilde{\theta}}(\tilde{x}) -g_{\tilde{\theta}}(x) \right\|_L \label{eq:by1stpiece}\\
 		&\leq  \chi \exp(L \nu +\beta )\sum_{\ell=1}^L \left \|W^{\ell} -\widetilde{W}^{\ell}\right\|_{\ell}+\left\|x-\tilde{x}\right\|\exp(\nu L+\beta) \label{eq:byclaim2}\\
 		&\leq  (\chi +1) \exp(L \nu +\beta ) \left(\left\|x-\tilde{x}\right\| + \sum_{\ell=1}^L \left \|W^{\ell} -\widetilde{W}^{\ell}\right\|_{\ell}  \right) \nonumber \\
 		&= (\chi +1) \exp(L \nu +\beta ) \left\| (x,\theta)-(\tilde{x},\tilde{\theta})\right\|_{\nn},  \nonumber 
 	\end{align}
 	where at Equation~\eqref{eq:by1stpiece} we have used Equation~\eqref{eq:1stpiece} and at line~\eqref{eq:byclaim2} we have used Claim 2. This concludes the proof of Proposition~\ref{prop:littleLip}.
 \end{proof}

 \noindent \textbf{Remark 1:}
 Since the norm-based quantities involved in the parameter counting bounds only show up in logarithmic terms in Theorem~\ref{thm:generalization_bound} (which is the equivalent of Theorem 1 in the main paper), it would also be possible to avoid working with the architecture-specific norms $\|\nbull\|_\ell$ in the proof of Proposition~\ref{prop:littleLip}. However, this would result in additional factors (including an additional factor of $\sqrt{n}$ in Proposition~\ref{prop:littleLip}) and a corresponding extra logarithmic  factors in the final result. In the norm-based bounds (Theorem~\ref{thm:NormBasedappendix}, equivalent to Theorem 2 in the main paper), using our architecture-specific norms is more crucial, since the norms of the weights show up outside the logarithms.

 \begin{lemma}
 	\label{Liploss}
 	
 	Let $u_1,\ldots,u_k$ be real numbers satisfying $u_1<u_2<\ldots<u_k$ and suppose that the variables $g_1,\ldots,g_k$ are defined as functions of $f_1,\ldots,f_k$ via the following: 
 	\begin{align*}
 		g_i=\text{softmax}(f_1,\ldots,f_k)_i = \frac{\exp(f_i)}{\sum_{j=1 }^k\exp(f_j)}.
 	\end{align*}
 	Define also the variable 
 	\begin{align*}
 		\rho=\sum_{i=1}^k u_i g_i
 	\end{align*}
 	We have the following bound on the $L^1$ norm of the derivative of $\rho$ with respect to the vector $(f_1,\ldots,f_k)$: 
 	\begin{align*}
 		\left\|\frac{\partial \rho}{\partial f}\right\|_1 \leq \fixx:=u_k-u_1.
 	\end{align*}
 	
 	In particular, this means that $\fixx$ is also an upper bound on the Lipschitz constant of $\rho$ with respect to the $L^\infty$ norm. 
 \end{lemma}
 \begin{proof}
 	
 	By standard formulae for softmax derivatives we have for all $i,j\leq k$:
 	\begin{align*}
 		\frac{\partial g_i}{\partial f_j}= g_i(1-g_i)\delta_{i,j} -g_ig_j (1-\delta_{i,j}),
 	\end{align*}
 	where $\delta_{i,j}$ denotes the Chroenecker delta function (i.e. $\delta_{i,j}=1$ of $i=j$ and $0$ otherwise).

 	Plugging in this formula into the definition of $\rho$ we obtain: 
 	\begin{align*}
 		\frac{\partial \rho }{\partial f_j}&= u_jg_j(1-g_j)=\sum_{i\neq j} u_ig_ig_j\nonumber \\&=u_jg_j-\sum_{i=1}^ku_i g_ig_j=g_j\left( u_j- \sum_{i=1}^ku_i g_i\right).
 	\end{align*}
 	
 	It follows that 
 	\begin{align*}
 		\left\|\frac{\partial \rho }{\partial f}\right\|_1\leq \sum_{j=1}^k g_j\left|u_j- \sum_{i=1}^ku_i g_i\right|\leq u_k-u_1=\fixx,
 	\end{align*}
 	as expected.
 	The final statement follows from classic arguments. Indeed, for all $f^1,f^2$, we have 
 	\begin{align*}
 		\left| \rho(f^1) -\rho(f^2) \right|&=\int_{0}^1 \frac{\partial \rho}{\partial t}(t f^1+(1-t)f^2)   dt \nonumber \\&=\int_{0}^1 \sum_{i=1}^k\frac{\partial \rho}{\partial f_i}(t f^1+(1-t)f^2)  \left(f^1_i-f^2_i\right) dt \nonumber \\ &\leq \left \| f^1-f^2\right\|_\infty 	\left\|\frac{\partial \rho }{\partial f}\right\|_1 = \fixx \left \| f^1-f^2\right\|_\infty ,
 	\end{align*}
 	
 	as expected.
 	
 \end{proof}
 
 The following lemma is classic~\cite{LongSed,bookhighprob}. This version is taken from Lemma A.8 in~\cite{LongSed}.
 \begin{lemma}
 	\label{lem:coversphere}
 	Let $d$ be a positive integer, $\|\nbull\|$ be a norm and $\rho$ be the metric induced by $\|\nbull\|$. Let $\kappa,\epsilon>0$. Assume that $\epsilon\leq \kappa$. A ball of radius $\kappa$ in $\R^d$ w.r.t. $\rho$ can be covered by $\left(\frac{3\kappa}{\epsilon}\right)^d$ balls of radius $\epsilon$. In particular, regardless of the value of $\epsilon$, the ball can be covered by $\left\lceil\frac{3\kappa}{\epsilon}\right\rceil^d$ balls of radius $\epsilon$.
 \end{lemma}

 \begin{proposition}
 	\label{prop:fulllip}

 	Let $x,\tilde{x}\in \mathcal{B}_\chi:= B_{\R^r}(0,\chi)=\{\zeta\in\R^r:\|\zeta\|\leq \chi\}$ be two points in the ball of radius $\chi$ in $\R^r$ (those are user embeddings to be fed to the network $g_{\theta_2}$). Let $\theta=(W^1,\ldots,W^L),\tilde{\theta}=(\widetilde{W}^1,\widetilde{W}^2,\ldots,\widetilde{W}^L)\in\R^{D_2}$ be two possible values of the parameter $\theta_2$ with the property that $\sum_{\ell\leq L}\|W^\ell-M^\ell\|_{\ell} \leq  \beta$ and $\sum_{\ell\leq  L}\|\widetilde{W}^\ell-M^\ell\|_{\ell} \leq \beta$ .  We define the following norm on the space $\R^r\otimes \R^{D_2}$: 
 	
 	$$\|(x,\theta_2)\|_{\nn} =\|x\|+\sum_{\ell\leq  L}\|\widetilde{W}^\ell-M^\ell\|_{\ell}.$$
 	
 	Assume that the initialization $M^1,\ldots,M^L$ satisfies $\left\|M^\ell\right\|\leq 1+\nu$  for all $\ell\leq L$.  We have the following inequality:
 	\begin{align}
 		\label{eq:liplful}
 		&\left\| f(g_{\theta}(x))-f(g_{\tilde{\theta}}(\tilde{x}) )  \right\|_{\infty} \\\nonumber&\leq \fixx (\chi+1) \exp(\beta+\nu L) \left\|(x,\theta)-(\tilde{x},\tilde{\theta})\right\|_{\nn}.
 	\end{align}
 	
 \end{proposition}

 \begin{proof}
 	Follows directly from Lemma~\ref{Liploss} and Proposition~\ref{prop:littleLip}.

 \end{proof}

 \begin{proposition}
 	\label{prop:coverF}
 	We write $\mathcal{W}_{\beta,\nu}$ for the set of parameters $\theta_2$ for the decoder $g$ which satisfy conditions~\eqref{cond:M} and~\eqref{cond:W}. Let us also write $\mathcal{B}^m_\chi$ for the set of $X\in\R^{m\times r}$ such that 
 	\begin{align}
 		\label{cond:X}
 		\max_{i\leq m} \|X_{i,\nbull}\|\leq \chi.
 	\end{align} 
 	
 	Let $\mathcal{F}_{\chi,\beta,\nu}\subset \R^{m\times n}$ denote a set of matrices which can be obtained from intermediate user encodings of norms less than $\chi$ and decoder network satisfying the constraints~\eqref{cond:M} and~\eqref{cond:W}: 
 	\begin{align}
 		&	\mathcal{F}_{\chi,\beta,\nu}:=\\&\Bigg\{ F\in\R^{m\times n}: \exists \theta_2\in R^{D_2}, X\in\R^{m\times r} \quad \text{s.t.} \quad  X\in \mathcal{B}_{\chi}^m  \quad \land \nonumber \\& \quad \quad \quad  \theta_2 \in \mathcal{W}_{\beta,\nu}\quad \land \quad F_{i,j}= \sum_{\kappa =1}^k \frac{\fix_{\kappa} \exp(g_{\theta_2}(x_i)_{j,\kappa})  }{\sum_{\kappa'=1}^k \exp(g_{\theta_2}(x_i)_{j,\kappa'})}   \Bigg\}\textcolor{black}{.}\nonumber
 	\end{align}
 	
 	For any training set $\Omega\subset [m]\times [n]$ of size $N$, we have the following bound on the $L^\infty$ covering number of $\mathcal{F}$ for any $\varepsilon\leq \min(\beta,\chi)$:
 	\begin{align*}
 		&\log\left(\mathcal{N}_{\Omega}\left( \mathcal{F}_{\chi,\beta,\nu},  L^\infty,\varepsilon  \right)\right)\\ \nonumber & \leq [mr+D_2]\left[ [\beta+\nu L] +\log\left( \frac{6 \fixx (\chi +\beta) (\chi+1)  }{\varepsilon}+1 \right)  \right].
 	\end{align*}

 \end{proposition}
 
 \begin{proof}

 	Let $\epsilon>0$. By Lemma~\ref{lem:coversphere}, there exists an $\epsilon/2$ cover $\mathcal{C}_1\subset \R^{m\times r}$ of the $\chi$-ball in  $\R^{m\times r}$ with respect to the norm $\|X\|_x= \max_{i\leq m} \|X_{i,\nbull}\|$ satisfying: 
 	\begin{align*}
 		\left|\mathcal{C}_1 \right| \leq \left\lceil\frac{6\chi}{\epsilon}\right\rceil^{mr}.
 	\end{align*}

 	Similarly, there exists an $\epsilon/2$ cover $\mathcal{C}_2\subset \R^{D_2}$ with respect to the norm $\|\theta\|_w= \sum_{\ell\leq L }\|W^\ell-M^\ell\|_\ell$ satisfying: 
 	\begin{align*}
 		\left| \mathcal{C}_2\right| \leq \left\lceil \frac{6\beta }{\epsilon}\right\rceil^{D_2}.
 	\end{align*}
 	
 	Consider the cover $\mathcal{C}:=\mathcal{C}_1\times \mathcal{C}_2$. For any $(X,\theta)\in\R^{m\times r}\times \R^{D_2}$ satisfying $\|X_{i,\null}\|\leq \chi$ ($\forall i$) and $\|\theta\|_{w}\leq 1+\beta$, if we define $\tilde{x}$ to be the closest element to $X$ in $\R^{m\times n}$ w.r.t. the norm $\max_{i\leq m} \|X_{i,\nbull}\|$ and $\tilde{\theta}$ to be the closest element to $\theta$ in $\R^{D_2}$ in terms of the norm $\|\theta\|_w= \sum_{\ell\leq L }\|W^\ell-M^\ell\|_\ell$, it is easy to see that $\|(X,\theta)-(\tilde{X},\tilde{\theta})\|_{\nn} \leq \epsilon/2+\epsilon/2=\epsilon$. 
 	
 	Furthermore, by Proposition~\ref{prop:fulllip}, we can even deduce the following. 
 	Let $F,\tilde{F}\in\R^{m\times n}$ be the matrices defined by 
 	\begin{align}
 		F_{i,j}&= f(g_{\theta}(X_{i,\nbull}))_j  \quad \quad \text{and} \nonumber \\
 		\tilde{F}_{i,j}&= f(g_{\tilde{\theta}}(\tilde{X}_{i,\nbull}))_j. \nonumber 
 	\end{align}
 	We have 
 	\begin{align}
 		\label{eq:almostthere}
 		\left|F_{i,j}-\tilde{F}_{i,j}\right|\leq \epsilon
 	\end{align}
 	for all $i,j$. 
 	
 	Define the set 	$\widebar{\mathcal{C}}:=$
 	\begin{align*}
 		&\Bigg\{F\in\R^{m\times n}: \exists X\in \mathcal{B}_{\chi}^m, \theta\in \mathcal{W}_{\beta,\nu} \quad \text{s.t.} \\\nonumber &\quad\quad \quad \quad \quad \quad  \quad \quad   \quad \quad \quad  \quad \quad F_{i,j}= f(g_\theta)(X_{i,\nbull})_j \quad \forall i,j \Bigg\}.
 	\end{align*}

 	By Proposition~\ref{prop:fulllip} any sample $\Omega$, we have that $\widebar{\mathcal{C}}$ is an $\epsilon \fixx (\chi+1) \exp(\beta+\nu L) $ cover of $\mathcal{F}_{\chi,\beta,\nu}$. 
 	
 	Therefore, setting $\epsilon=\varepsilon \left[  \fixx (\chi+1) \exp(\beta+\nu L)  \right]^{-1}$, we obtain an $\varepsilon$ cover of $\mathcal{F}_{\chi,\beta,\nu}$ with cardinality bounded by $\mathcal{N}$ where 	$\log(\mathcal{N})=$
 	\begin{align}
 		&\log \left(\left\lceil\frac{6\chi}{\epsilon}\right\rceil^{mr} \times  \left\lceil \frac{6\beta }{\epsilon}\right\rceil^{D_2} \right) \\
 		&  \leq mr \log\left(\left\lceil\frac{6\chi}{\epsilon}\right\rceil\right)+ D_2 \log\left( \left\lceil \frac{6\beta }{\epsilon}\right\rceil \right)\nonumber \\
 		&	\leq [mr+D_2] \log\left( \left\lceil\frac{6 (\chi +\beta)}{\epsilon}\right \rceil\right)\nonumber \\
 		&\leq  [mr+D_2] \log\left( \left\lceil\frac{6 \fixx (\chi +\beta) (\chi+1) \exp(\beta+\nu L) }{\varepsilon}\right \rceil\right)\nonumber \\
 		&\leq  [mr+D_2]\left[ [\beta+\nu L] +\log\left( \frac{6 \fixx (\chi +\beta) (\chi+1)  }{\varepsilon}+1 \right)  \right], \nonumber
 	\end{align}
 	as expected. 
 	
 \end{proof}

 The following is the equivalent of Theorem 1 in the main paper. 
 \begin{theorem}
 	\label{thm:generalization_bound}
 	Define the empirical RMSE	$\rmseemp=\frac{1}{N}\sum_{(i,j)\in\Omega}\elr(R_{i,j}+\zeta_{i,j},F_{i,j})$  and the population RMSE by 
 	$\rmsepop=\mathbb{E}_{\mathcal{D}}\left(\elr(R_{i,j}+\zeta_{i,j},F_{i,j})\right)$\footnote{See also the first remark in section~\ref{ap:notation}.}.

 	Fix $\beta,\nu,\chi>0$ and initialized weights $M^1,\ldots,M^L$. With probability greater than $1-\delta$ over the draw of the training set $\Omega$, for every set of parameters $\theta_1,\theta_2$ satisfying the following conditions:
 	\begin{align}
 		\label{cond:MW}
 		\sum_{\ell\leq L}\|W^\ell-M^\ell\|_\ell \leq  \beta  \nonumber \\
 		\left\|M^\ell\right\|_\ell\leq 1+\nu,
 	\end{align}
 	if the condition 
 	\begin{align*}
 		\max_{i\leq m} \|\phi_{\theta_1}(U_i)\|\leq \chi  
 	\end{align*}
 	is satisfied, then we also have $\left | \rmsepop -\rmseemp \right|\leq$
 	\begin{align*}
 		&3[\fixx^2] \sqrt{\frac{\log(2/\delta)}{2N}} + \frac{16\fixx^2 }{N} +  \fixx^2 \sqrt{ \frac{48[mr+D_2] [\beta+\nu L]}{N}} \nonumber  \\&   \>  +\fixx^2\sqrt{\frac{[mr+D_2]\log\left(72 N (\chi +\beta) (\chi+1) +1 \right)  }{N}}.
 	\end{align*}
 	
 \end{theorem}

 \begin{remark}
 	The theoretical results do not explicitly model how the encoder network learns the feature representations, only the function class restriction implied by the existence of such a low dimensional embedding. In practice, the encoder network is also able to capture useful information and aid training, though it is not easy to characterize this phenomenon mathematically and the problem is left to future work. A strong indication that a low-rank representation can be learned without explicitly forcing $r$ to be small can be gleaned from the recent work~\cite{BottleNeckRank}.
 \end{remark}
 
 \begin{proof}
 	
 	%Recall that $F_{i,j}=f(g^{\theta_2}(\phi_{\theta_1}(U_i))_{j,\nbull})$, and by assumption in the Theorem, we have that 
 	By the first condition, regardless of the values $\Omega$, we have that
 	$	\max_{i\leq m} \|\phi_{\theta_1}(U_i)\|\leq \chi $. Therefore, regardless of the precise values of $U_{i}$ and $\theta_1$, as long as the first condition of the theorem holds, the matrix $F$ belongs to the function class $\mathcal{F}_{\chi,\beta,\nu}$. In addition, our loss function $\elr$ is bounded by $\fixx^2$.  Therefore, by Theorem~\ref{rademachh}, we immediately have, with probability greater than $1-\delta$ over the draw of the training set: 
 	\begin{align*}
 		&\left|\rmsepop -\rmseemp \right|=\\&\left| \mathbb{E}_{\mathcal{D}}\left(\elr(R_{i,j}+\zeta_{i,j},F_{i,j})\right)- \frac{1}{N}\sum_{(i,j)\in\Omega}\elr(R_{i,j}+\zeta_{i,j},F_{i,j})   \right|\nonumber \\
 		&\leq \sup_{\bar{F}\in\mathcal{F}_{\chi,\beta,\nu}} \Bigg| \mathbb{E}_{\mathcal{D}}\left(\elr(R_{i,j}+\zeta_{i,j},\bar{F}_{i,j})\right)-\\& \quad \quad \quad \quad \quad  \quad \quad \quad \quad \quad \quad \quad  \frac{1}{N}\sum_{(i,j)\in\Omega}\elr(R_{i,j}+\zeta_{i,j},\bar{F}_{i,j})   \Bigg|\nonumber \\
 		&\leq 2\rad_{S}(\elc\circ \mathcal{F}_{\chi,\beta,\nu })+3[\fixx^2] \sqrt{\frac{\log(2/\delta)}{2N}}.
 	\end{align*}

 	Now, since our loss function $\elc$ is also $2\fixx$-Lipschitz, we can apply the classic Talagrand's lemma~\ref{lem:talagrand} to obtain (w.p. $1-\delta$):
 	\begin{align*}
 		&\left| \rmsepop -\rmseemp \right|\nonumber \\
 		&\leq 2\rad_{S}(\elr\circ \mathcal{F}_{\chi,\beta,\nu })+3[\fixx^2] \sqrt{\frac{\log(2/\delta)}{2N}}\nonumber \\
 		&	\leq 4\fixx \rad_{S}(\mathcal{F}_{\chi,\beta,\nu })+3[\fixx^2] \sqrt{\frac{\log(2/\delta)}{2N}}.
 	\end{align*}

 	Now, we can use Proposition~\eqref{prop:dudley} with $\alpha=\frac{1}{N}$ to continue: 
 	\begin{align*}
 		&\left| \rmsepop -\rmseemp \right|\nonumber \\
 		&\leq 4\fixx^2 \rad_{S}(\frac{1}{\fixx}\mathcal{F}_{\chi,\beta,\nu })+3[\fixx^2] \sqrt{\frac{\log(2/\delta)}{2N}}\nonumber \\
 		&\leq 3[\fixx^2] \sqrt{\frac{\log(2/\delta)}{2N}} + \frac{16\fixx^2 }{N}+ \frac{48\fixx^2}{\sqrt{N}} \times \\ & \int_{\frac{1}{N}}^1 \sqrt{ [mr+D_2]\left[ [\beta+\nu L] +\log\left( \frac{6  (\chi +\beta) (\chi+1)  }{\varepsilon}+1 \right)  \right]} d\epsilon  \nonumber  \\
 		&\leq  3[\fixx^2] \sqrt{\frac{\log(2/\delta)}{2N}} + \frac{16\fixx^2 }{N} + \frac{48\fixx^2}{\sqrt{N}}  \times \\&\sqrt{ [mr+D_2]\left[ [\beta+\nu L] +\log\left(6 N (\chi +\beta) (\chi+1) +1 \right)  \right]},\nonumber 
 	\end{align*}
 	a expected.
 	
 \end{proof}

 \section{Norm-Based bounds}
 Recall the following results.
 \begin{proposition}[\cite{ledent2021normbased}, Proposition 6]
 	\label{suplin}
 	Let positive reals $(a,b,\epsilon)$ and positive integer $m$ be given. Let the tensor $X\in \mathbb{R}^{N\times U\times d}$ be given with $\forall i\in \{1,2,\ldots,N\},\forall u \in \{1,2,\ldots,U\},  \>  \|X_{i,u,\nbull}\|_{2}\leq b $. For any fixed $M$:
 	%	\vspace{-0.1cm}
 	\begin{align*}
 		&\log_2 \mathcal{N}  \left(     \{XA : A\in \mathbb{R}^{d\times m}, \|A-M\|_{2,1}\leq a\},\epsilon,  \|\nbull\|_{*}   \right) \nonumber \\& \leq  \frac{64a^2b^2}{\epsilon^2} \log_2\left[\left(\frac{8ab}{\epsilon}+7\right)mNU\right],
 	\end{align*}
 	%	\vspace{-0.11cm}
 	where %$XA$ is defined by $(XA)_{u,i,j}=\sum_{o=1}^dX_{u,i,o}A_{o,j}$ and
 	the norm $\|\nbull\|_{*}$  over the space $ \mathbb{R}^{n\times U\times m}$  is defined by \small $\| Y  \|_{*}= \max_{i\leq n}\max_{j\leq U} \left[\sum_{k=1}^m Y_{i,j,k}^2\right]^{\frac{1}{2}}. $
 \end{proposition}
 \begin{proposition}[\cite{ledent2021normbased}, Proposition 5]
 	\label{suplininf}
 	Let positive reals $(a,b,\epsilon)$ and positive integer $m$ be given. Let the tensor $X\in \mathbb{R}^{N\times U\times d}$ be given with $\forall i\in \{1,2,\ldots,N\},\forall u \in \{1,2,\ldots,U\},  \>  \|X_{i,u,\nbull}\|_{2}\leq b $. For any fixed $M$:
 	%	\vspace{-0.1cm}
 	\begin{align*}
 		&\log_2 \mathcal{N}  \left(     \{XA : A\in \mathbb{R}^{d\times m}, \|A-M\|_{\Fr}\leq a\},\epsilon,  \|\nbull\|_{\inf}   \right)\\ &\leq  \frac{64a^2b^2}{\epsilon^2} \log_2\left[\left(\frac{8ab}{\epsilon}+7\right)mNU\right].
 	\end{align*}
 \end{proposition}
 \textbf{Remark:} The second dimension with size $U$ in the above theorem should be interpreted as the number of convolutional patches. In the case of our architecture, this will be $n$, the number of items. The proof strategy consists in applying a result of~\cite{ZhangCovering} (see Proposition~\ref{MaureySup}) for the $L^\infty$ covering number of linear classes to a modified problem where each (convolutional patch, input sample, output channel) combination is treated as a different sample. 
 
 As an almost immediate application, we obtain the following corollary for our setting: 
 \begin{corollary} [Covering number for layer $ \ell$]
 	\label{suplincor}
 	Consider the following function class associated to a layer in the decoder network: 
 	\begin{align*}
 		\mathcal{H}_\ell&:=\Bigg\{ \R^{ K_{\ell-1}} \rightarrow \R^{ K_{\ell}}: x\mapsto \op(W^\ell )x \quad \text{s.t.} \\& \nonumber  \quad \quad\quad\quad  \quad\left\|W^\ell-M^\ell\right\|_{2,1} \leq a   \> \land  \>\|W^{\ell}\|_{\ell} \leq s  \Bigg\}\quad \quad ( \ell\leq L_0-1)\\
 		\mathcal{H}_\ell&:=\Bigg\{ \R^{n\times K_{\ell-1}} \rightarrow \R^{n\times K_{\ell}}: x\mapsto \op(W^\ell )x \quad  \text{s.t.}\\&\nonumber \quad \quad \quad    \left\|W^\ell-M^\ell\right\|_{2,1} \leq a   \> \land \>\|W^{\ell}\|_{\ell} \leq s  \Bigg\}\quad \quad (L_0\leq \ell\leq L-1)\\
 		\mathcal{H}_\ell&:=\Bigg\{ \R^{n\times K_{\ell-1}} \rightarrow \R^{n\times K_{\ell}}: x\mapsto \op(W^\ell )x \quad \text{s.t.}\\& \nonumber \quad\quad\quad\quad\quad\quad  \left\|W^\ell-M^\ell\right\|_{\Fr} \leq a   \> \land \>\|W^{\ell}\|_{\ell} \leq s  \Bigg\}\quad \quad (\ell=L)
 	\end{align*}
 	Suppose we have a given "sample set"  $x_1,\ldots,x_{\widetilde{N}} \in \R^{n\times k_1}$ such that for all $i\leq \widetilde{N}$ we have $\|x_i\|_{\ell-1}\leq b$. For any $\epsilon>0$, there exists a cover $\mathcal{C}_\ell \subset\mathcal{H}_\ell$ with the following properties: 
 	\begin{align*}
 		\forall h\in \mathcal{F}_\ell, \quad \exists \tilde{h}\in \mathcal{C}_\ell  \quad \text{s.t.}  \quad \forall i\leq \widetilde{N}, \quad   \left\| h(x_i)-\tilde{h}(x_i) \right \|_{\ell}\leq \epsilon 
 	\end{align*}
 	and 
 	\begin{align*}
 		\log\left(|\mathcal{C}_\ell|\right) \leq \frac{256a^2b^2}{\epsilon^2} \log\left[\left(\frac{16ab}{\epsilon}+7\right)\widetilde{N}D_2n\right].
 	\end{align*}
 \end{corollary}

 \begin{proof}
 	Without the condition $\|W^{\ell}\|_{\ell} \leq s$, the corollary follows immediately from Propositions~\ref{suplin} (for layers $\ell\leq L-1$) and~\ref{suplininf} (for $\ell=L$) by replacing the product of the input spacial dimensions ($n$ or $1$ depending on the layer) and output channels ($K_{\ell}$) by $D_2n$, which strongly majorates the former since we certainly have $D_2=\sum_{\ell\neq L_0} K_{\ell-1}K_\ell +nK_{L_0-1}K_{L_0}\leq  nD_2$ . The lemma, with the condition $\|W^{\ell}\|_{\ell} \leq s \quad \forall l\leq L$ and a multiplicative factor of $2^{-1}$ added to $\epsilon$, then follows from classic doubling arguments between covering numbers and packing numbers. See also~\cite{graf,ledent2021normbased} for other variations of this argument for more general architectures. 
 \end{proof}
 \textbf{Remark:} This bound holds for both convolutional and fully-connected layers. The factor of $n$ inside the logarithmic term is not necessary in the case of the fully-connected layer. However, we leave it in to alleviate the notation with only a modest impact on tightness.

 Recall also the following simplification of Proposition 10 in~\cite{ledent2021normbased}. The main differences are: 
 \begin{itemize}
 	\item We work directly with the class including the restrictions on the $\|\nbull\|_\ell$ norms defined in the Appendix~\ref{ap:notation}  (this make little difference in the proof)
 	\item We fix  norms one layer at a time incurring products of $\|\nbull\|_{\ell}$ norms, rather than making a finer analysis including loss function augmentation and empirical estimates of the norms of the intermediary layers as is done in~\cite{ledent2021normbased}. This also simplifies the proof in~\cite{ledent2021normbased}.
 \end{itemize}

 \begin{proposition}[\cite{ledent2021normbased} (proposition 10), cf. also~\cite{spectre}(Lemma A.7) and~\cite{graf} (Section C.4)]
 	\label{Chainingprop}
 	Let $L$ be a natural number and $a_1,\ldots,a_L>0$ and $s_1,\ldots,s_L>0$ be real numbers. Let $\mathcal{V}_0,\mathcal{V}_1,\ldots,\mathcal{V}_{L}$ be $L+1$ vector spaces each endowed with a norm $|\nbull|_{\ell}$ for $0\leq l\leq L$. 
 	Let $B_1,B_2,\ldots,B_L$ be $L$ vector spaces  each endowed with two $\|\nbull\|_{\ell}$ and $\|\nbull\|_{\ell,\sigma}$. For each $1\leq \ell\leq L$, let $\mathcal{B}_{\ell}:=\left\{ W\in B_\ell: \|W\|_{\ell}\leq a_{\ell} \quad \land \quad   \|W\|_{\ell,\sigma}\leq s_{\ell}\right\}$.
 	% and $\mathcal{B}_1,\mathcal{B}_2,\ldots,\mathcal{B}_L$ be the balls of radii $a_1,a_2,\ldots,a_L$ in the spaces $B_1,B_2,\ldots,B_L$ with the norms $|\nbull|_{1},|\nbull|_{2},\ldots,|\nbull|_{L}$ respectively.
 	For each $\ell\in \{1,2,\ldots,L\}$ we are given an operator $F^\ell: \mathcal{V}_{\ell-1}\times B_{\ell}\rightarrow \mathcal{V}_{\ell}: (x,A)\rightarrow F^\ell_{A}(x)$. For each $\ell_1,\ell_2$ with $\ell_2>\ell_1$ and each  $\mathcal{A}^{\ell_1,\ell_2}=(A^{\ell_1+1},\ldots,A^{\ell_2})\in \mathcal{B}^{\ell_1,\ell_2}:=\mathcal{B}_{l_1+1}\times  \mathcal{B}_{\ell_1+2} \times \ldots  \mathcal{B}_{\ell_2}$, let us define $$F^{l_1\rightarrow \ell_2}_{\mathcal{A}^{\ell_1,\ell_2}}:\mathcal{V}_{l_1}\rightarrow \mathcal{V}_{l_2}: x\rightarrow F^{\ell_1\rightarrow \ell_2}_{\mathcal{A}^{\ell_1,\ell_2}}(x)= F^{\ell_2}_{A^{\ell_2}}\circ \ldots\circ  F^{\ell_1}_{A^{\ell_1}}(x),$$
 	for all $\ell$, $\mathcal{F}^\ell_{\mathcal{A}^L}=F^{0\rightarrow \ell}_{\mathcal{A}^L}$ and $\mathcal{F}_{\mathcal{A}}=\mathcal{F}^L_{\mathcal{A}}.$

 	Suppose that the following \textbf{assumptions} are satisfied:
 	\begin{enumerate}
 		\item For all $A\in\mathcal{B}_{\ell}$, $F^\ell_{A}(\nbull):\mathcal{V}_{\ell-1}\rightarrow \mathcal{V}_{\ell}$ is $s_{\ell}$-Lipschitz with respect to the norms $|\nbull|_{\ell-1}$ and $|\nbull|_{\ell}$ on $\mathcal{V}_{\ell-1}$ and $ \mathcal{V}_{\ell}$ respectively.
 		\item 	For all $\ell\in \{1,2,\ldots,L\}$, all $b>0$,  all $Z_1,Z_2,\ldots,Z_N\in \mathcal{V}_{\ell-1}$ such that $|Z_i|_{\ell-1}\leq b \quad \forall i$ and all $\epsilon>0$, there exists a subset $\mathcal{C}_\ell(b,Z,\epsilon)\subset \mathcal{B}_\ell$ such that
 		\begin{align}
 			\log(\#\left(\mathcal{C}_{\ell,\epsilon,n}(b,Z,\epsilon)\right))\leq \frac{C_{\ell,\epsilon,n}a_l^2b^2}{\epsilon^2},
 		\end{align}
 		where $C_{\ell,\epsilon,n}$ is some function of $\ell,\epsilon,n$ (and $a_\nbull, b_\nbull$)
 		and, for all $A\in \mathcal{B}_\ell$, there exists an $\bar{A}\in \mathcal{C}_\ell(b,\epsilon)$ such that:
 		\begin{align}
 			\left|F^\ell_{A}(Z_i)-F^\ell_{\widebar{A}}(Z_i)\right|_{\ell}\leq  \epsilon \quad \quad \forall i,
 		\end{align}
 		and for all $\widebar{A}=(\widebar{A}^1,\widebar{A}^2,\ldots,\widebar{A}^L)\in \mathcal{C}_\ell(b,Z,\epsilon)$, $\|\widebar{A}^\ell\|_\ell\leq s_\ell$.
 	\end{enumerate}

Then,  for any $0<\epsilon<1$ and $b>0$, and for any $x_1,\ldots,x_N\in \mathcal{V}_0$ such that $|x_i|_{0}\leq b \quad \forall i$, there exists a subset $\mathcal{C}_{\epsilon,b,N}$ of $\mathcal{B}^L$ such that for all $\mathcal{A}=(A^1,A^2,\ldots,A^L)\in \mathcal{B}:=\mathcal{B}^1\times \ldots\times \mathcal{B}^L$ there exists a $\bar{\mathcal{A}}\in \mathcal{C}_{\epsilon,b,X}$ such that, for any $i$ such that 
 	\begin{align}
 		\label{toolll}
 		\left|  F^{0\rightarrow \ell}_{\mathcal{A}^\ell}(x_i)-   F^{0\rightarrow \ell}_{\bar{\mathcal{A}}^\ell}(x_i) \right|_{\ell}&\leq \frac{\epsilon}{\prod_{l=\ell+1}^L s_l}\quad \quad (\forall \ell\leq L, \forall i\leq \widetilde{N}).
 	\end{align}

 	Furthermore, we have
 	\begin{align*}
 		\log \#(\mathcal{C}_{\epsilon,b,X})&\leq 4b^2 C^*_{\epsilon,n} \left[\prod_{\ell=1}^L s_\ell^2\right] \left[  \sum_{\ell =1}^L \left(\frac{a_\ell }{s_\ell \epsilon}\right)^{\frac{2}{3}} \right]^3\nonumber  \\&
 		\leq 4b^2C^*_{\epsilon,n}\frac{L^2}{\epsilon^2}  \left[\prod_{\ell=1}^L s_\ell^2\right]\sum_{\ell=1}^L \left(\frac{a_\ell}{s_\ell } \right)^{2},
 	\end{align*}
 	where 
 	$C^*_{\epsilon,n}:=\max_{\ell\leq L} C_{\ell,\epsilon\prod_{l=\ell}^{L}s_l^{-1},n}$.
 	
 \end{proposition}

 We can now use Proposition~\ref{Chainingprop} and Lemma~\ref{suplin} to prove the covering number bound below.
 
 \begin{theorem}
 	\label{thm:covernorm}
 	Let $a_1,\ldots,a_L,\chi ,s_1,\ldots,s_L>0$.  Suppose we are also given some initialization weights $M^1,\ldots,M^L$. 
 	
 	Define $\mathcal{W}_{a_\nbull ,s_\nbull }$ to be the set of weights $\theta_2$ with the following properties: 
 	\begin{align}
 		\label{condnorm:W}
 		&\|W^\ell\|_\ell\leq s_\ell \quad \quad \forall \ell\leq L \\
 		& \|W^\ell-M^\ell\|_{2,1}\leq a_\ell (\forall \ell<L)\\
 		&\|W^L-M^L\|_{\Fr}\leq a_L.
 	\end{align}

 	Consider the following function class $	\mathcal{F}_{\chi, a_{\nbull},s_\nbull}:=$
 	\begin{align*}
 		&	\Bigg\{F\in\R^{m\times n}: \exists \theta_2\in R^{D_2}, X\in\R^{m\times r} \quad \text{s.t.}  \quad X\in \mathcal{B}^m_{\chi}  \land \quad \nonumber \\& \> \> \quad \theta_2 \in \mathcal{W}_{a_\nbull ,s_\nbull }\quad \land \quad F_{i,j}= \sum_{\kappa =1}^k \frac{\fix_{\kappa} \exp(g_{\theta_2}(x_i)_{j,\kappa})  }{\sum_{\kappa'=1}^k \exp(g_{\theta_2}(x_i)_{j,\kappa'})}   \Bigg\}. \nonumber 
 	\end{align*}
 	For any $1>\epsilon>0$ there exists a cover $\mathcal{C}_\epsilon \subset 	\mathcal{F}_{\chi, a_{\nbull},s_\nbull}$ with the following properties. 
 	For all $F\in\mathcal{F}_{\chi, a_{\nbull},s_\nbull}$, there exists 
 	$\widebar{F}\in \mathcal{C}_\epsilon $ such that 
 	\begin{align*}
 		\left| \widebar{F}_{i,j}-F_{i,j}\right| \leq  \epsilon  \quad \quad \quad (\forall i\leq m, \forall j\leq n)
 	\end{align*}
 	and  
 	\begin{align*}
 		&	\log\left( |\mathcal{C}| \right) \leq mr \log\left( \frac{600\chi \fixx}{\epsilon} \left[\prod_{\ell=1}^L s_\ell \right] +1\right)   \\&\quad \>+ \frac{1050 r\fixx^2}{\epsilon^2} \left[\prod_{\ell=1}^L s_\ell^2\right]   \left[  \sum_{\ell =1}^L \left(\frac{a_\ell }{s_\ell }\right)^{\frac{2}{3}} \right]^3   \times \\& \quad \quad \>\log\left( D_2n\left[\frac{17a_{\max} S \fixx }{\epsilon}+7\right] \left[ \frac{600\chi\fixx}{\epsilon}\left[\prod_{\ell=1}^L s_\ell \right]+1\right]   \right),
 	\end{align*}
 	where $S:=\max_{\ell\leq L}\prod_{l=\ell}^L s_{l}\leq \prod_{\ell=1}^L \left[s_{\ell}+1\right]$.
 	
 \end{theorem}
 \begin{proof}
 	By Lemma~\ref{lem:coversphere} and similarly to the beginning of the proof of Proposition~\ref{prop:coverF}, for any $\epsilon_1$,  there exists a cover $\mathcal{C}_0 \subset \R^r$ of the ball of radius $\chi$ in $\R^r$ (w.r.t. the $L^2$ norm)  with cardinality satisfying: 
 	\begin{align*}
 		\left|\mathcal{C}_0  \right|\leq \left\lceil\frac{6\chi}{\epsilon_1}\right\rceil^{r}\leq  \left( \frac{6\chi}{\epsilon_1}+1\right)^{r}.
 	\end{align*}

 	We will now \textit{use the cover} $\mathcal{C}_0$ \textit{as a sample set} in an application of Proposition~\ref{Chainingprop}. Note that this will only help us obtain a satisfactory cover of our function class thanks to the fact that Proposition~\ref{Chainingprop} returns an $L^\infty$ cover (this is only possible with a norm-based bound as above). The assumption $|x_i|_{0}\leq b$ from Proposition~\ref{Chainingprop} is satisfied with $b=\chi$ by construction of the internal cover $\mathcal{C}_0$. In addition, Assumption 1 from Proposition~\ref{Chainingprop} is satisfied by definition of the our norms $\|\nbull\|_{\ell}$, which play the role of $\|\nbull\|_{\ell,\sigma}$ in Proposition~\ref{Chainingprop} (the functions $F^\ell_{A}$). Finally, Assumption 2 from Proposition~\ref{Chainingprop} is satisfied by Proposition~\ref{suplincor}  with $U=1$  if $\ell\leq L_0-1$, $U=n$ if $L_0\leq \ell\leq L$, $\|\nbull\|_{\ell,\sigma}=\|\nbull\|$ (for $\ell\leq L-1$), $\|\nbull\|_{\ell,\sigma}=\|\nbull^\top\|_{2,\infty}$ (for $\ell=L$), with $|\nbull|_{\ell}$ being our layer specific norm $\|\nbull\|_\ell$ (cf. beginning of Section ~\ref{sec:paracount}) and with $C_{\epsilon,\ell,n}=256 \log\left[\left(\frac{16ab}{\epsilon}+7\right)\widetilde{N}D_2n\right].$

 	Thus, we obtain, for each $\epsilon_2>0$, a cover $\mathcal{C}_1\subset \mathcal{W}_{a_\nbull,s_\nbull}$ with the following properties: 
 	
 	1. 	For all  $\theta_2\in\mathcal{W}_{a_\nbull,s_\nbull}$ there exists a $\widebar{\theta}_2\in  \mathcal{W}_{a_\nbull,s_\nbull}$ such that for any $\bar{x}\in \mathcal{C}_1$, we have 
 	\begin{align*}
 		\left| g_{\theta_2}(\bar{x})_{j,k}-f g_{\bar{\theta}_2}(\bar{x})_{j,k}\right| \leq \epsilon_2. \quad \quad \forall k
 	\end{align*}
 	
 	2. The cardinality of the cover $\mathcal{C}_0$ is bounded by 
 	\begin{align}
 		\label{eq:thisstepisearly}
 		&\log(|\mathcal{C}_1|)\leq  4  \chi^2	C^*_{\epsilon_2,n} \left[\prod_{\ell=1}^L s_\ell^2\right] \left[  \sum_{\ell =1}^L \left(\frac{a_\ell  }{s_\ell \epsilon}\right)^{\frac{2}{3}} \right]^3 \nonumber \\& 
 		\leq 4\chi^2 	C^*_{\epsilon_2,n}\frac{L^2}{\epsilon_2^2}  \left[\prod_{\ell=1}^L s_\ell^2\right]\sum_{\ell=1}^L \left(\frac{a_\ell}{s_\ell } \right)^{2}.
 	\end{align}
 	where, writing $a_{\max}$ for $\max(a_1,\ldots,_L)$, 
 	\begin{align*}
 		C^*_{\epsilon_2,n}&:=\max_{\ell\leq L} C_{\ell,\epsilon_2}\\
 		&\leq 256  \log\left[\left(\frac{16a_{\max}  S}{\epsilon_2}+7\right)\widetilde{N}D_2n\right] \nonumber \\
 		&=256\left[\log\left[\left(\frac{16a_\ell S}{\epsilon_2}+7\right)D_2n\right] +r\log \left( \frac{6\chi}{\epsilon_1}+1\right) \right]\nonumber \\
 		&\leq 256  r\log\left( D_2n\left[\frac{16a_{\max} S}{\epsilon_2}+7\right] \left[ \frac{6\chi}{\epsilon_1}+1\right]   \right).
 	\end{align*}
 	
 	Continuing the calculation in equation~\eqref{eq:thisstepisearly}, we therefore have   
 	\begin{align*}
 		&	\log(|\mathcal{C}_1|)\leq  \frac{1024  r\chi^2}{\epsilon_2^2} \left[\prod_{\ell=1}^L s_\ell^2\right]   \left[  \sum_{\ell =1}^L \left(\frac{a_\ell }{s_\ell }\right)^{\frac{2}{3}} \right]^3    \times \\& \quad \quad \quad \quad \quad    \log\left( D_2n\left[\frac{16a_{\max} S}{\epsilon_2}+7\right] \left[ \frac{6\chi}{\epsilon_1}+1\right]   \right).
 	\end{align*}
 	
 	Now, define the following cover of $\mathcal{B}_{\chi}^m \times \mathcal{W}_{a_\nbull,s_\nbull }$: 
 	\begin{align*}
 		\widebar{\mathcal{C}} :=\mathcal{C}_0^{\times m} \times \mathcal{C}_1.
 	\end{align*}
 	This induces a "cover" $\mathcal{C}\subset\mathcal{F}_{a_\nbull,s_\nbull}$ via  the map 
 	\begin{align*}
 		(X,\theta_2)\rightarrow F \quad \text{s.t.} \quad F_{i,j}=f(g_{\theta_2}(X_{i,\nbull})_{j,\nbull}).
 	\end{align*}
 	Furthermore, for any pair $(X,\theta_2)\in \mathcal{B}_\chi^m \times \mathcal{W}
 	_{a_\nbull,s_\nbull}$, defining $\widebar{\theta}_2$ to be the closest cover element to $\theta_2$ and  defining   $\widebar{X}$ ensuring that for all $i\leq m$, $\widebar{X}_{i,\nbull}$ is the closest cover element to $X_{i,\nbull}$, we have the following inequality for any $i\leq m$ and $ j\leq n$
 	\begin{align*}
 		&\nonumber 	\left| F^{(X,\theta_2)}_{i,j} -F^{(\widebar{X},\widebar{\theta}_2)}_{i,j}  \right| \\ \nonumber  & \leq 	\left| F^{(X,\theta_2)}_{i,j} -F^{(X,\widebar{\theta}_2)}_{i,j}  \right| +	\left| F^{(X,\widebar{\theta}_2)}_{i,j}-F^{(\widebar{X},\widebar{\theta}_2)}_{i,j}   \right|  \nonumber  \\
 		&\leq \fixx\left[\prod_{\ell=1}^L s_\ell \right] \epsilon_1 +\fixx\epsilon_2.
 	\end{align*}

 	Thus, setting 
 	\begin{align*}
 		\epsilon_1 &=\frac{\epsilon}{100\fixx\left[\prod_{\ell=1}^L s_\ell \right]} \quad \quad \text{and} \\
 		\epsilon_2&= \frac{99}{100\fixx} \epsilon, 
 	\end{align*}
 	we obtain an $\epsilon$ cover $ \mathcal{C}$ of $\mathcal{F}_{a_\nbull,s_\nbull}$ such that

 	\begin{align*}
 		&\log\left(\left|\mathcal{C}\right|\right )\\
 		&\leq  \nonumber mr \log\left( \frac{6\chi \fixx}{\epsilon_1}+1\right)+  \frac{1024   \chi^2r \fixx^2}{\epsilon_2^2} \left[\prod_{\ell=1}^L s_\ell^2\right]   \left[  \sum_{\ell =1}^L \left(\frac{a_\ell }{s_\ell }\right)^{\frac{2}{3}} \right]^3  \times \\&\quad \quad \quad \quad \quad \quad \quad \> \nonumber       \log\left( D_2n\left[\frac{16a_{\max} S \fixx}{\epsilon_2}+7\right] \left[ \frac{6\chi\fixx}{\epsilon_1}+1\right]   \right) \nonumber  \\
 		&\leq mr \log\left( \frac{600\chi \fixx }{\epsilon} \left[\prod_{\ell=1}^L s_\ell \right] +1\right) + \\& \quad \quad \quad \quad\quad \quad\quad \quad  \frac{1050 r\chi^2\fixx^2}{\epsilon^2} \left[\prod_{\ell=1}^L s_\ell^2\right]   \left[  \sum_{\ell =1}^L \left(\frac{a_\ell }{s_\ell }\right)^{\frac{2}{3}} \right]^3 \\& \nonumber   \quad \> \times       \log\left( D_2n\left[\frac{17a_{\max}S\fixx }{\epsilon}+7\right] \left[ \frac{600\chi\fixx}{\epsilon}\left[\prod_{\ell=1}^L s_\ell \right]+1\right]   \right),\nonumber 
 	\end{align*}
 	
 	as expected. 
 \end{proof}
 
 Using this, we can finally obtain our norm-based bounds by using Dudley's entropy formula to bound the Rademacher Complexity of $\mathcal{F}_{a_\nbull,s_\nbull}$ (the following is the equivalent of Theorem 2 in the main paper). 
 \begin{theorem}
 	\label{thm:NormBasedappendix}
 	
 	Fix $a_1,\ldots,a_L>0$, $\chi>0$ and $s_1,\ldots,s_L>0$. With probability greater than $1-\delta$ over the draw of the training set $\Omega$, for every set of parameters $\theta_1,\theta_2$ satisfying the following conditions: 
 	\begin{align*}
 		\|W^\ell\|_\ell & \leq s_\ell \quad \quad \forall \ell \leq L \\ 
 		\left\|\left(W^\ell-M^\ell\right)^\top\right \|_{2,1} &\leq a_\ell    \quad \quad \forall \ell \leq L-1 \\
 		\left\|\left(W^L-M^L\right)^\top\right \|_{\Fr}&\leq a_L \\
 		\max_{i\leq m} \|\phi_{\theta_1}(U_i)\|&\leq \chi,
 	\end{align*}
 	we  have: 
 	\begin{align*}
 		&\left| \rmsepop-\rmseemp\right| \\\nonumber &  \leq  3\fixx^2\sqrt{\frac{\log(2/\delta)}{2N}} + \frac{16\fixx^2 }{N}  \\ &\nonumber+ 48\fixx^2 \sqrt{\frac{mr}{N} }     \sqrt{\log\left( 600N\chi  \left[\prod_{\ell=1}^L s_\ell \right] +1\right) }   \\
 		&   \quad \quad + 1584 \fixx^2 \left[\prod_{\ell=1}^L s_\ell\right]   \left[  \sum_{\ell =1}^L \left(\frac{a_\ell }{s_\ell }\right)^{\frac{2}{3}} \right]^\frac{3}{2}\sqrt{\frac{r  }{N}} \times \\& \quad  \sqrt{\log \left( D_2n\left[17Na_{\max} S+7\right] \left[ 600N\chi\left[\prod_{\ell=1}^L s_\ell \right]+1\right]   \right)},\nonumber 
 	\end{align*}
 	where $S:=\max_{\ell\leq L}\prod_{l=\ell}^L s_{l}\leq \prod_{\ell=1}^L \left[s_{\ell}+1\right]$.
 \end{theorem}
 
 \begin{proof}
 	As in the proof of Theorem~\ref{thm:generalization_bound}, we have by Theorem~\ref{rademachh}, we immediately have, with probability greater than $1-\delta$ over the draw of the training set: 
 	\begin{align*}
 		&\left| \rmsepop -\rmseemp \right|\\&=\left| \mathbb{E}_{\mathcal{D}}\left(\elr(R_{i,j}+\zeta_{i,j},F_{i,j})\right)- \frac{1}{N}\sum_{(i,j)\in\Omega}\elr(R_{i,j}+\zeta_{i,j},F_{i,j})   \right|\nonumber \\
 		&\leq \sup_{\bar{F}\in\mathcal{F}_{\chi, a_\nbull,s_\nbull}} \Bigg| \mathbb{E}_{\mathcal{D}}\left(\elr(R_{i,j}+\zeta_{i,j},\bar{F}_{i,j})\right)\\  \nonumber & \quad \quad \quad \quad  \quad \quad \quad \quad  \quad \quad \quad \quad -  \frac{1}{N}\sum_{(i,j)\in\Omega}\elr(R_{i,j}+\zeta_{i,j},\bar{F}_{i,j})   \Bigg|\nonumber \\
 		&\leq 2\rad_{S}(\elr\circ \mathcal{F}_{\chi, a_\nbull,s_\nbull})+3\fixx^2 \sqrt{\frac{\log(2/\delta)}{2N}}\nonumber \\
 		&\leq 4\fixx\rad_{S}( \mathcal{F}_{\chi, a_\nbull,s_\nbull})+3\fixx^2 \sqrt{\frac{\log(2/\delta)}{2N}},
 	\end{align*}
 	where we have used the classic Talagrand contraction lemma~\ref{lem:talagrand} at the last line. 
 	
 	Now, using Dudley's entropy theorem~\ref{prop:dudley} as well as Theorem~\ref{thm:covernorm} (for $\fixx=1$): 
 	
 	\begin{align*}
 		&\left| \rmsepop -\rmseemp \right| \\&\leq 4\fixx^2\rad_{S}\left(\frac{1}{\fixx} \mathcal{F}_{\chi, a_\nbull,s_\nbull}\right)+3\fixx^2\sqrt{\frac{\log(2/\delta)}{2N}}\nonumber \\
 		&\leq  3\fixx^2 \sqrt{\frac{\log(2/\delta)}{2N}} + \frac{16\fixx^2 }{N}+ \nonumber \\&\frac{48\fixx^2}{\sqrt{N}} \int_{\frac{1}{N}}^1  \Bigg[mr \log\left( \frac{600\chi}{\epsilon} \left[\prod_{\ell=1}^L s_\ell \right] +1\right) \nonumber \\& \quad \quad \quad \quad  \quad \quad \quad \quad  +  \frac{1050 r\chi^2}{\epsilon^2} \left[\prod_{\ell=1}^L s_\ell^2\right]   \left[  \sum_{\ell =1}^L \left(\frac{a_\ell }{s_\ell }\right)^{\frac{2}{3}} \right]^3    \times \\&  \quad   \log\Bigg( D_2n\Bigg[\frac{17a_{\max} S}{\epsilon}+7\Bigg] \left[ \frac{600\chi }{\epsilon}\left[\prod_{\ell=1}^L s_\ell \right]+1\right]   \Bigg)    \Bigg]^{\frac{1}{2}} d\epsilon \nonumber \\& \leq  3\fixx^2 \sqrt{\frac{\log(2/\delta)}{2N}} + \frac{16\fixx^2 }{N}  \\& + 48\fixx^2  \chi^2\sqrt{\frac{mr}{N} } \sqrt{\log\left( 600N\chi  \left[\prod_{\ell=1}^L s_\ell \right] +1\right) }  \\
 		& \quad \quad \quad  \quad \quad \quad + 1584 \fixx^2 \chi \left[\prod_{\ell=1}^L s_\ell\right]   \left[  \sum_{\ell =1}^L \left(\frac{a_\ell }{s_\ell }\right)^{\frac{2}{3}} \right]^\frac{3}{2}\sqrt{\frac{r  }{N}} \times  \\& \quad \quad \sqrt{\log \left( D_2n\left[17Na_{\max} S+7\right] \left[ 600N\chi \left[\prod_{\ell=1}^L s_\ell \right]+1\right]   \right)},\nonumber 
 	\end{align*}
 	as expected. 
 \end{proof}
 
 \section{On the Interpretation and Generalization Properties of our Loss Function}
 \label{sec:implicitbounds}
 In this section, we further discuss our loss function and its interpretation in terms of probabilities. 
 
 Consider the learning problem of minimizing the following function:
 \begin{align}
 	\label{eq:therealloss}
 	&\widehat{	\mathcal{L}}(g)=\frac{1}{N}\sum_{(i,j)\in\Omega} -\left[\log(G_{i,j,	[R_{\Omega}]_{i,j}})-\log(G_{i,j,0})  \right] \nonumber \\&\quad \quad \quad  \quad \quad \quad \quad \quad \quad -\textcolor{black}{\con}\sum_{(i,j)\in[m]\times [n]}\log(G_{i,j,0})
 \end{align}
 where as usual $G_{i,j,\kappa}=\frac{\exp(g_{i,j,\kappa})}{\sum_{u=0}^k \exp(g_{i,j,u})    }$ and where $\mathcal{K}$ is a constant which will be determined later.
 This can be viewed as an empirical analogue of the following population version:

 \begin{align}
 	\label{eq:thereallosspop}
 	&	\mathcal{L}(g)=\E_{(i,j,\kappa)\sim \mathcal{D}} -\left[\log(G_{i,j,\kappa	})-\log(G_{i,j,0})  \right] \nonumber \\&\quad \quad \quad  \quad \quad \quad \quad \quad \quad-\textcolor{black}{\con}\sum_{(i,j)\in[m]\times [n]}\log(G_{i,j,0}).
 \end{align}
 
 Note that when $\Omega$ doesn't contain any duplicates (as is the case in all real-life datasets we consider) and \textcolor{black}{$\con=\frac{1}{N}$}, the loss function~\eqref{eq:therealloss} is the same as the loss function~\eqref{zerobased} we use in our model (up to a constant factor).

 Below, we will also use the notation $\imploss:\R^{k+1}\otimes [k]\rightarrow \R$ for the loss function such that 
 \begin{align}
 	\imploss(F_{i,j,\nbull},y)&=-\left[\log(G_{i,j,y})-\log(G_{i,j,0})  \right] 
 \end{align}
 i.e., $	\imploss(g,y)=$
 \begin{align}
 	\label{eq:imploss}
 	&-\left[\log\left(\frac{\exp(g_y)}{\sum_{\kappa=1}^k\exp(F_{\kappa}) }\right)-\log\left(\frac{\exp(g_0)}{\sum_{\kappa'=1}^k\exp(F_{\kappa'}) }\right)  \right]\nonumber  \\
 	&=\log\left(  \frac{\exp(g_0)}{\exp(g_y)}        \right)=g_0-g_y.
 \end{align}
 Thus, equation~\eqref{eq:therealloss} can be written 	$\widehat{\mathcal{L}}(g)=$
 \begin{align*}
 	\frac{1}{N}\sum_{(i,j)\in\Omega} \imploss(F_{i,j,\nbull},y)-\textcolor{black}{\con}\sum_{(i,j)\in[m]\times [n]}\log(G_{i,j,0}).
 \end{align*}

 \subsection{On the Optimal Value of $G_{i,j,\kappa}$ }

Theorems~\ref{thm:generalization_bound_main} and~\ref{thm:NormBased_main} concern the generalization gap in terms of the square loss. They show that with high probability, the test MSE is not much higher than the training MSE. However, our model works by minimizing the cross-entropy loss, not the square loss. It is therefore interesting to wonder whether this cross entropy training provably allows the model to make predictions with a low (train and test) MSE. Further, we can wonder whether our learning procedure allows the model to learn the full distribution over entries and ratings, which implies successful capture of the \textit{implicit feedback} information. In this section, we answer both questions in the affirmative by showing that if the ground truth distribution is realizable by our model architecture, then minimizing a loss function equivalent to the $\elc$ loss provably recovers the ground truth sampling distribution up to a small deviation in total variation norm.  In particular, Corollaries~\ref{cor:firstcorrr} and~\ref{cor:explicitfromimplicit} below establish the proof of Theorem~\ref{thm:allthisnewstuff} in the main. 
 
 Consider the learning problem of minimizing the following function:
 \begin{align*}
 	&\widehat{	\mathcal{L'}}(g)=\frac{1}{N}\sum_{(i,j)\in\Omega} -\left[\log(G_{i,j,	[R_{\Omega}]_{i,j}})-\log(G_{i,j,0})  \right] \nonumber \\&\quad \quad \quad  \quad \quad \quad \quad \quad \quad -\textcolor{black}{\frac{1}{N}}\sum_{(i,j)\in[m]\times [n]}\log(G_{i,j,0}).
 \end{align*}
 where as usual $G_{i,j,\kappa}=\frac{\exp(g_{i,j,\kappa})}{\sum_{u=0}^k \exp(g_{i,j,u})    }$. 
 As mentioned above, this can be viewed as an empirical analogue of the following population version: 
 
 \begin{align*}
 	&	\mathcal{L}'(g)=\E_{(i,j,\kappa)\sim \mathcal{D}} -\left[\log(G_{i,j,\kappa	})-\log(G_{i,j,0})  \right] \nonumber \\&\quad \quad \quad  \quad \quad \quad \quad \quad \quad-\textcolor{black}{\frac{1}{N}}\sum_{(i,j)\in[m]\times [n]}\log(G_{i,j,0}).
 \end{align*}

 \textcolor{black}{Importantly, when} $\Omega$ doesn't contain any duplicates (as is the case in all real-life datasets we consider), the loss function~\eqref{eq:therealloss} is the same as the loss function~\eqref{zerobased} we use in our model (up to a constant factor). \textcolor{black}{Thus, although the results in this section apply to the minimization of the loss~\eqref{eq:therealloss}, this is equivalent to minimizing our loss~\eqref{zerobased} in practical scenarios.}

 Below, we will also use the notation $\imploss:\R^{k+1}\otimes [k]\rightarrow \R$ for the loss function such that 
 \begin{align*}
 	\imploss(G_{i,j,\nbull},y)&=-\left[\log(G_{i,j,y})-\log(G_{i,j,0})  \right] 
 \end{align*}
 i.e., $	\imploss(g,y)=$
 \begin{align}
 	\label{eq:imploss}
 	&-\left[\log\left(\frac{\exp(g_y)}{\sum_{\kappa=1}^k\exp(F_{\kappa}) }\right)-\log\left(\frac{\exp(g_0)}{\sum_{\kappa'=1}^k\exp(F_{\kappa'}) }\right)  \right]\nonumber  \\
 	&=\log\left(  \frac{\exp(g_0)}{\exp(g_y)}        \right)=g_0-g_y.
 \end{align}
 Thus, equation~\eqref{eq:therealloss} can be written $\widehat{\mathcal{L'}}(g)=$
 \begin{align*}
 	\frac{1}{N}\sum_{(i,j)\in\Omega} \imploss(G_{i,j,\nbull},y)-\textcolor{black}{\frac{1}{N}}\sum_{(i,j)\in[m]\times [n]}\log(G_{i,j,0}).
 \end{align*}

 Let us recall  a complete model of the sampling procedure:
 each entry is sampled i.i.d. from the set $[m]\times [n]\times [\fix_1,\ldots,\fix_{k}]$, with the probability of $(i,j,\fix_{\kappa})$ being equal to $p_{i,j,\kappa}$ for some $p\in\R^{m\times n\times k}$ such that $\sum_{i=1}^m\sum_{j=1}^n\sum_{\kappa=1}^k p_{i,j,\kappa}=1$. For convenience, we also write $p_{i,j,0}$ for the quantity $1-\sum_{\kappa=1}^k p_{i,j,\kappa}=\sum_{(i',j')\neq (i,j)}\sum_{\kappa=1}^k p_{i',j',\kappa}$. 
 
 %	For simplicity, we explain the following for the parameter counting version of the bounds, but the same discussion is of course also valid in the case of the norm-based bounds. 
 
 %Assume that $p\in \mathcal{G}_{\chi,\beta,\nu}$

 We have
 \begin{align}
 	\label{eq:theexplanation}
 	-\mathcal{L}'(g)&=\sum_{i=1}^m\sum_{j=1}^n \sum_{\kappa=1}^k p_{i,j,\kappa} \log(G_{i,j,\kappa}) \\& \quad \quad \quad  +\sum_{i=1}^m\sum_{j=1}^n \left[\textcolor{black}{\frac{1}{N}}-1+p_{i,j,0} \right] \log(G_{i,j,0}).
 \end{align}
 
 In particular, as long as 
 \begin{align}
 	\label{eq:awkwardCondition}
 	\textcolor{black}{\frac{1}{N}}\geq 1-p_{i,j,0},
 \end{align}
 we have that the choice of $G$ which minimizes the population loss $\mathcal{L}'(g)$ is
 \begin{align}
 	G^{\text{Bayes}}_{i,j,\kappa} &=\textcolor{black}{Np_{i,j,\kappa}}\quad \quad \text{for} \quad \quad \kappa\neq 0 \nonumber \\
 	G^{\text{Bayes}}_{i,j,0} &=1-\sum_{\kappa=1}^k \textcolor{black}{N p_{i,j,\kappa}}.\label{eq:expressbayes}
 \end{align} 
 
The condition~\eqref{eq:awkwardCondition} is equivalent to the requirement that no entry is expected to be sampled more than once if a data set of size $N$ is drawn, which is reasonable in real-life datasets.  In addition, as long as there are no duplicates in the training set $\Omega$, our loss function~\eqref{zerobased} from the experiments section will be exactly the same as the loss function ~\eqref{eq:therealloss} from above. However, it is worth noting that if the entries are drawn i.i.d., even if the condition~\eqref{eq:awkwardCondition} is satisfied (and no particular entry has a high probability of being a duplicate), some duplicate entries will occur in the full dataset with high probability. Thus, formally, Theorem~\ref{thm:implicitparam} only applies to the loss~\eqref{eq:therealloss} in the i.i.d. case. In practice, real data sets do \textit{not} involve duplicate entries because the sampling procedure isn't exactly i.i.d. (duplicate entries are naturally avoided since no one will rate the same movie twice). Nonetheless, the i.i.d. setting is widely considered a good proxy for preliminary study of recommender system datasets~\cite{ReallyUniform1,ReallyUniform2,Foygel_Max_Norm,foygel2011learning,mostrelated,mostrelatedearly,LedentIMC,ledent2024generalization}. 
 
 \textcolor{black}{Equation~\eqref{eq:expressbayes} is} in line with the intuition explained in the main paper:  $G_{i,j,0}$ is an estimate of the probability that an i.i.d. dataset of $N$ samples will not contain any observation for entry $(i,j)$. Since most entries are unobserved, this is typically a very high value (close to $1$), as is observed in practice. 
 
 \subsection{Guarantees for the Cross-Entropy Excess Risk}
 \label{subsec:B}
To understand what our generalization bounds can teach us about the model's ability to recover the ground truth distribution and make strong implicit feedback predictions, let us express the excess risk under the assumption that $\gbay$ can be represented with our model architecture. Let $\widehat{g}$ (resp. $\widehat{G}$) denote the scores (resp. probabilities) of the empirical risk minimizer. For simplicity let us  write $p_{i,j}$ for $\sum_{\kappa=1}^k p_{i,j,\kappa}=\frac{1-G^{\text{Bayes}}_{i,j,0}}{N}$, $\widehat{G}_{i,j}$ for $1-\widehat{G}_{i,j,0}$, and $\Gall$ for $\sum_{i,j=1}^{m,n} \widehat{G}_{i,j}$.
 
When $g^{*}=\gbay$, from  Equations~\eqref{eq:theexplanation} and~\eqref{eq:expressbayes} we can express the excess risk as follows: 
 	\begin{align*}
 		& \nonumber 	\mathcal{L}'(\widehat{g})-\mathcal{L}'(g^*)\\\nonumber&=\sum_{i,j=1}^{m,n}  [p_{i,j,\kappa}]\log\left[\frac{p_{i,j,\kappa}}{\frac{\widehat{G}_{i,j,\kappa}}{N}}\right] \\\nonumber& \quad  \quad \quad \nonumber \quad \quad  +\frac{1}{N}\sum_{i,j} [1-Np_{i,j}]\log\left[\frac{ [1-Np_{i,j}]}{1-\widehat{G}_{i,j}}\right]\\
 		&\nonumber =\sum_{i,j=1}^{m,n}  [p_{i,j,\kappa}]\log\left[\frac{p_{i,j,\kappa}}{\frac{\widehat{G}_{i,j,\kappa}}{\Gall}}\right] +\log\left(\frac{N}{\Gall}\right) \\&\nonumber  \quad \quad \quad \quad \quad + \frac{mn-N}{N}\sum_{i,j} \frac{[1-Np_{i,j}]}{mn-N}\log\left[\frac{ \frac{1-Np_{i,j}}{mn-N}}{\frac{1-\widehat{G}_{i,j}}{mn-N}}\right]\\
 		&\nonumber=\sum_{i,j=1}^{m,n}  [p_{i,j,\kappa}]\log\left[\frac{p_{i,j,\kappa}}{\frac{\widehat{G}_{i,j,\kappa}}{\Gall}}\right]   \\&\nonumber  \quad \quad  + \frac{mn-N}{N}\sum_{i,j} \frac{[1-Np_{i,j}]}{mn-N}\log\left[\frac{ \frac{1-Np_{i,j}}{mn-N}}{\frac{1-\widehat{G}_{i,j}}{mn-\Gall}}\right] \\ & \quad \quad \quad \quad \nonumber   + \frac{mn-N}{N}\log\left(\frac{mn-N}{mn-\Gall}\right) +\log\left(\frac{N}{\Gall}\right) \\
 		&\nonumber=\text{KL}\left( p_{\nbull,\nbull,\nbull}, \frac{\widehat{G}_{\nbull,\nbull,\nbull}}{\Gall} \right)\nonumber \\ & \quad  +\frac{mn-N}{N}\text{KL}\left((1-Np_{\nbull,\nbull})/(mn-N), \frac{1-\widehat{G}_{\nbull,\nbull}}{mn-\Gall}\right)\nonumber \\ & \quad \quad +\frac{1}{N} \log\left(  \frac{(mn-N)^{mn-N} N^N}{(mn-\Gall)^{mn-N} \Gall^N} \right)\nonumber \\
 		&=T_1+T_2+T_3.\nonumber
 	\end{align*}

In the above: 
 	% $\mathcal{P}$ and $\widehat{\mathcal{P}}$
 	\begin{itemize}
 		\item The first term $T_1$ is the KL divergence of the probability distributions over $[m]\times [n]\times [k]$ defined by  (for all $i\leq m,j\leq n, \kappa \leq k$) $\P(i,j,\kappa)=p_{i,j,\kappa}$ and $\P(i,j,\kappa)=\widehat{G}_{i,j,\kappa}/\Gall$ respectively. 
 		\item The second term $T_2$ is equal to  $\frac{mn-N}{N}$ times the KL divergence between the distributions over $[m]\times [n]$ defined by $\P(i,j)=1-N(p_{i,j,0})$ (i.e., $\P(i,j)$ approximates the probability that a dataset of size $N$ doesn't contain a sample at entry $(i,j)$) and $\P(i,j)=\frac{1-\widehat{G}_{i,j}}{mn-\Gall}$  (i.e., the model's estimate of the same quantity).
 		\item The last term $T_3$ regulates the normalization constants $\Gall$ and $mn-\Gall$. 
 	\end{itemize} 
 	Note that the quantity inside the denominator of the logarithm in the definition of $T_3$ is maximized for $\Gall=N$, which implies that $T_3\geq 0$ in all cases. In addition, it is trivially the case that $T_1\geq 0$ and $T_2 \geq 0$.

From this, we can obtain an analogue of Theorems~\ref{thm:generalization_bound} for the excess risk defined in terms of KL divergence expressed in the term $T_1$ above. Note that it is also possible to obtain an analogue of Theorem~\ref{thm:NormBasedappendix}, we simply focus on the parameter counting example to illustrate the technique.
 	\begin{theorem}
 		\label{thm:implicitparam}
 		Fix $\beta,\nu,\chi>0$ and initialized weights $M^1,\ldots,M^L$.  Let $g^*$ (resp $\widehat{g}$ ) be the minimizer of the loss~\eqref{eq:thereallosspop} (resp~\eqref{eq:therealloss}) subject to the following conditions: 
 		\begin{align}
 			\label{cond:MWimpl}
 			\sum_{\ell\leq L}\|W^\ell-M^\ell\|_\ell &\leq  \beta  \nonumber \\
 			\left\|M^\ell\right\|_\ell&\leq 1+\nu   \nonumber \\
 			\max_{i\leq m} \|\phi_{\theta_1}(U_i)\| &\leq \chi  \nonumber   \\
 			\|g_{\theta_2}(\phi_{\theta_1}(U_i))_{j,\kappa}\|&\leq  B \quad \forall i,j,\kappa 
 		\end{align}
 		Then, if $g^*=\gbay$,  we have with probability greater than $1-\delta$ over the draw of the training set $\Omega$, 
 		\begin{align}
 			\label{eq:calQU}
 			&	 	T_1=	\text{KL}\left( p_{\nbull,\nbull,\nbull}, \frac{\widehat{G}_{\nbull,\nbull,\nbull}}{\Gall} \right)\leq \left|\mathcal{L}'(\hat{g})-\mathcal{L}'(g^*)\right| \leq \mathcal{Q},
 		\end{align}
 		where
 		\begin{align*}
 			\nonumber
 			\mathcal{Q}&:= 12B \sqrt{\frac{\log(2/\delta)}{2N}} + \frac{32 B }{N}  +\frac{96 B}{\sqrt{N}}  \times \\&   \nonumber \sqrt{ [mr+D_2]\left[ [\beta+\nu L] +\log\left(6 N (\chi +\beta) (\chi+1)B +1 \right)\right]}.
 		\end{align*}	
 \end{theorem}

 	\begin{proof}
 		First, note that from the same argument as in the proof of Proposition~\ref{prop:coverF}, we have for any fixed $B>1$ \footnote{For convenience, we define the function class $\mathcal{G}_{\chi,\beta,\nu}$ in Equation~\eqref{eq:defineparamG} using the extra constraint  $\|g_{\theta_2}(\phi_{\theta_1}(U_i))_{j,\kappa}\|\leq B$. However, the value of $B$ only becomes relevant later in equation~\eqref{eq:1123}.}:
 		\begin{align}
 			\label{eq:elemcovercounting}
 			&	\log\left(\mathcal{N}_{\Omega}\left( \mathcal{G}_{\chi,\beta,\nu},  L^\infty,\varepsilon  \right)\right)\leq [mr+D_2]\times \\& \quad \quad \left[ [\beta+\nu L] +\log\left( \frac{6  (\chi +\beta) (\chi+1)  }{\varepsilon}+1 \right)  \right],\nonumber 
 		\end{align}
 		where 	$\mathcal{G}_{\chi,\beta,\nu}  :=$
 		\begin{align}
 			&\Bigg\{ g\in\R^{m\times n\times (k+1)}: \exists \theta_2\in R^{D_2}, X\in\R^{m\times r} \quad \text{s.t.}\nonumber  \\  & \quad \quad \quad  \quad X\in \mathcal{B}_{\chi}^m  \quad \land \quad \theta_2 \in \mathcal{W}_{\beta,\nu}  \quad \land  \label{eq:defineparamG}\\& \quad \quad \quad  \quad \quad  \quad 	\|g_{\theta_2}(\phi_{\theta_1}(U_i))_{j,\kappa}\|\leq B \quad  \forall i,j,\kappa  \quad \quad  \land \nonumber  \\ & \quad  \quad\quad  \quad  \quad\quad    \quad  \quad\quad \quad\quad  \quad \quad  \quad g_{i,j,k}= (g_{\theta_2}(x_i)_{j,\kappa})   \Bigg\}.\nonumber
 		\end{align}
 		Thus, similarly to the proof of Proposition~\ref{thm:generalization_bound}, we have for any $g$,  using Theorem~\ref{rademachh} and Proposition~\ref{prop:dudley}: 
 		\begin{align}
 			& \left|\widehat{\mathcal{L}}(g)-\mathcal{L}(g)\right| \nonumber \\&\leq  2\rad_{S}(\imploss\circ   \mathcal{G}_{\chi,\beta,\nu})+6B \sqrt{\frac{\log(2/\delta)}{2N}}\label{eq:1123}\\
 			&\nonumber \leq 4B\rad_{S}\left[  \frac{1}{B} \mathcal{G}_{\chi,\beta,\nu} \right]+6B \sqrt{\frac{\log(2/\delta)}{2N}} \nonumber \\
 			&\leq    6B \sqrt{\frac{\log(2/\delta)}{2N}} +\frac{16B}{N}+ \frac{48B}{\sqrt{N}} \times \\ \nonumber & \int_{\frac{1}{N}}^1 \sqrt{ [mr+D_2]\left[ [\beta+\nu L] +\log\left( \frac{6  (\chi +\beta) (\chi+1)  }{\varepsilon}+1 \right)  \right]} d\epsilon  \label{eq:rescale}  \\
 			&\leq 6B \sqrt{\frac{\log(2/\delta)}{2N}} + \frac{16 B}{N} + \frac{48B}{\sqrt{N}}\times \\&  \sqrt{ [mr+D_2]\left[ [\beta+\nu L] +\log\left(6 N (\chi +\beta) (\chi+1) +1 \right)  \right]},\nonumber 
 		\end{align}
 		where at Line~\eqref{eq:rescale} we have used Equation~\eqref{eq:elemcovercounting} with $\varepsilon \leftarrow \varepsilon B\geq \varepsilon$.  Equation~\eqref{eq:calQU} now follows directly from the standard use of the triangle inequality.  
 \end{proof}
 
We have the following easy corollary: 
 	\begin{corollary}
 		\label{cor:firstcorrr}
 		Let $\widehat{p}_{\nbull,\nbull,\nbull}=\frac{\widehat{G}_{\nbull,\nbull,\nbull}}{\mathcal{G}}$ denote the normalized distribution over $[m]\times [n]\times [k]$ obtained from $\widehat{G}$.  Under the assumptions of Theorem~\ref{thm:implicitparam} we have the following high probability upper bound on the L1 distance between $ \widehat{p}_{\nbull,\nbull,\nbull}$ and $ p_{\nbull,\nbull,\nbull}$:
 		\begin{align}
 			\label{eq:coupdegrace}
 			\left| p_{\nbull,\nbull,\nbull}- \widehat{p}_{\nbull,\nbull,\nbull}\right|_1\leq \sqrt{\frac{1}{2}\mathcal{Q}}.
 		\end{align}
 		Furthermore, we also have the following bound on the L1 distance between the estimated and ground truth distributions over the entries $(i,j)$:
 		\begin{align}
 			\label{eq:postprocessedcoupdegrace}
 			\left| p_{\nbull,\nbull}- \widehat{p}_{\nbull,\nbull}\right|_1=	\sum_{i,j=1}^{m,n} \left|p_{i,j}-\widehat{p}_{i,j}\right|\leq \sqrt{\frac{1}{2}\mathcal{Q}}.
 		\end{align}
 	\end{corollary}
 	\begin{proof}
 		Equation~\eqref{eq:coupdegrace} follows from Theorem~\ref{thm:implicitparam} and an application of Pinskers inequality (Lemma~\ref{lem:pinsker}).
 		Finally, Equation~\eqref{eq:postprocessedcoupdegrace} can be obtained as follows: 
 		\begin{align*}
 			\sum_{i,j=1}^{m,n} \left|p_{i,j}-\widehat{p}_{i,j}\right|& = 		\sum_{i,j=1}^{m,n} \left|\sum_{\kappa=1}^kp_{i,j,\kappa}-\sum_{\kappa=1}^k\widehat{p}_{i,j,\kappa}\right|\\
 			&\leq \sum_{i,j=1}^{m,n}\sum_{\kappa=1}^k\sum_{\kappa=1}^k\left|p_{i,j,\kappa}-\widehat{p}_{i,j,\kappa}\right|.
 		\end{align*}
 \end{proof}

From the above, if the ground truth distribution is realizable ($g^*=\gbay$), the normalized probability distribution $\frac{\widehat{G}}{\mathcal{G}}$ approaches the true distribution $p$ when $\frac{B}{\sqrt{N}} \sqrt{ [mr+D_2]\left[ [\beta+\nu L] +\log\left(6 N (\chi +\beta) (\chi+1) +1 \right)  \right]}\rightarrow 0$. This implies that training with our cross entropy loss provably allows the model to recover the precise sampling distribution, and to make accurate implicit feedback prediction. 
Interestingly, it also applies that training with our cross entropy loss also guarantees good performance \textit{in terms of MSE}:
 	\begin{corollary}
 		\label{cor:explicitfromimplicit}
 		Assume that there is no noise: for each $i,j$, there exists a single $\kappa_{i,j}\leq k$ such that $p_{i,j,\kappa_{i,j}}\neq 0$ and $p_{i,j,\kappa'}=0$ for all $\kappa'\neq \kappa_{i,j}$. Under the assumptions of Theorem~\ref{thm:implicitparam}, we have w.p. $\geq 1-\delta$: 
 		\begin{align*}
 			\rmsepop\leq 2\fixx^2 \sqrt{\frac{1}{2}\mathcal{Q}}.
 		\end{align*}
 	\end{corollary}
 	\begin{proof}
 		We have that the prediction $\widetilde{G}_{i,j}$ made for entry $i,j$ is given by 
 		\begin{align*}
 			\widetilde{G}_{i,j}=\sum_{\kappa=1}^k \frac{\widehat{p}_{i,j,\kappa}}{\widehat{p}_{i,j}}\fix_\kappa=\sum_{\kappa=1}^k \frac{\widehat{G}_{i,j,\kappa}}{\sum_{\kappa'=1}^k\widehat{G}_{i,j,\kappa'}}\fix_\kappa.
 		\end{align*}
 		From which it follows (w.h.p.)
 		\begin{align*}
 			\rmsepop &=\sum_{i,j} p_{i,j} \left[ \fix_{\kappa_{i,j}}-\sum_{\kappa\neq \kappa_{i,j}}\frac{\widehat{p}_{i,j,\kappa} }{\widehat{p}_{i,j,\kappa}}\right]^2\nonumber\\
 			&\leq \fixx^2\sum_{i,j,\kappa\atop \kappa\neq \kappa_{i,j}}(p_{i,j}-\widehat{p}_{i,j}+\widehat{p}_{i,j}) \frac{\widehat{p}_{i,j,\kappa}}{\widehat{p}_{i,j}} \nonumber  \\
 			&	\leq  \fixx^2 \sum_{i,j,\kappa\atop \kappa\neq \kappa_{i,j}}  \widehat{p}_{i,j,\kappa} + \fixx^2 \sum_{i,j,\kappa\atop \kappa\neq \kappa_{i,j}}  \left|  p_{i,j}-\widehat{p}_{i,j}  \right|\frac{\widehat{p}_{i,j,\kappa}}{\widehat{p}_{i,j}} \nonumber \\
 			&\leq \fixx^2\left[ \left|\widehat{p}_{\nbull,\nbull,\nbull}-p_{\nbull,\nbull,\nbull}\right|_{1} + \left|\widehat{p}_{\nbull,\nbull}-p_{\nbull,\nbull}\right|_{1}\right]\nonumber \\
 			&\leq 2\fixx^2 \sqrt{\frac{1}{2}\mathcal{Q}},\nonumber 
 		\end{align*}
 		as expected.
 \end{proof}

 \section{Classic Lemmas}

 Recall the following Lemma from~\cite{spectre,ledent2021normbased}:
 \begin{proposition}
 	\label{prop:dudley}
 	Let $\mathcal{F}$ be a real-valued function class taking values in $[0,1]$, and assume that $0\in \mathcal{F}$. Let $S$ be a finite sample of size $N$. For any $2\leq p\leq \infty$, we have the following relationship between the Rademacher complexity  $\rad(\mathcal{F}|_{S})$ and the covering number $\mathcal{N}(\mathcal{F}|S,\epsilon,\|\nbull\|_{p})$.
 	\begin{align*}
 		\rad(\mathcal{F}|_{S})\leq \inf_{\alpha>0} \left(     4\alpha+\frac{12}{\sqrt{N}}  \int_{\alpha}^1      \sqrt{\log \mathcal{N}(\mathcal{F}|S,\epsilon,\|\nbull\|_{p})  } d\epsilon     \right),
 	\end{align*}
 	where the norm $\|\nbull\|_{p}$ on $\mathbb{R}^m$ is defined by $\|x\|_{p}^p=\frac{1}{n}(\sum_{i=1}^m|x_i|^p)$.
 \end{proposition}

 Recall also the following classic theorem~\cite{rademach}:
 \begin{theorem}
 	\label{rademachh}
 	Let $Z,Z_1,\ldots,Z_N$ be iid random variables taking values in a set $\mathcal{Z}$. Consider a set of functions $\mathcal{F}\in[0,1]^{\mathcal{Z}}$. $\forall \delta>0$, we have with probability $\geq 1-\delta$ over the draw of the sample $S$ that $$\forall f \in \mathcal{F}, \quad \mathbb{E}(f(Z))\leq \frac{1}{N}\sum_{i=1}^Nf(z_i)+2\rad_{S}(\mathcal{F})+3\sqrt{\frac{\log(2/\delta)}{2N}}. $$
 \end{theorem}

 Covering numbers of linear function classes:
 \begin{proposition}
 	\label{MaureySup}
 	Let $n,d\in\mathbb{N}$, $a,b>0$. Suppose we are given $n$ data points collected as the rows of a matrix $X\in \mathbb{R}^{n\times d}$, with $\|X_{i,\nbull}\|_2\leq b,\forall i=1,\ldots,n$. For $U_{a,b}(X)=\big\{X\alpha:\|\alpha\|_2\leq a,\alpha\in\rbb^d\big\}$, we have
 	\small %\vspace{-0.2cm}
 	\begin{align*}
 		\log_2\mathcal{N}\left(U_{a,b}(X),\epsilon,\|\nbull\|_{\infty}  \right)\leq \frac{36a^2b^2}{\epsilon^2}\log_2\left(\frac{8abn}{\epsilon}+6n+1\right).
 	\end{align*}
 \end{proposition}

 Recall also the following classic Lemma (cf.~\cite{ledoux91probability} see also~\cite{meir2003generalization} page 846).

 \begin{lemma}[Talagrand Contraction Lemma]
 	\label{lem:talagrand}
 	Let $\mathcal{F}$ be a hypothesis space, let $x_1,\ldots,x_N$ be a sample set and let $\phi_1,\ldots,\phi_N$ be $N$ $\lambda$-Lipschitz functions. We have the following: 
 	\begin{align*}
 		\frac{1}{N} \E_\sigma\left[  \sup_{f\in\mathcal{F}}\sigma_i (\phi_i\circ f)(x_i)  \right ]\leq \frac{\lambda}{N}\E_\sigma\left[  \sup_{f\in\mathcal{F}}\sigma_i  f(x_i)  \right ],\nonumber 
 	\end{align*}
 	where the $\sigma_1,\ldots,\sigma_N$ are i.i.d. Rademacher variables. 
 \end{lemma}

 	\begin{lemma}[Pinsker's inequality]
 		\label{lem:pinsker}
 		Let $P$ and $\widehat{P}$ be two probability distributions over   some set $\mathcal{X}$ (e.g. $\mathcal{X}=[m]\times [n]\times [k]$). We have the following inequality: 
 		\begin{align*}
 			\left\|P-\widehat{P}\right\|_{1} \leq \sqrt{\frac{1}{2}\text{KL}(P,Q)}.
 		\end{align*}
 \end{lemma}

 	\begin{figure*}
 		\centering
 		\includegraphics[width=0.75\textwidth]{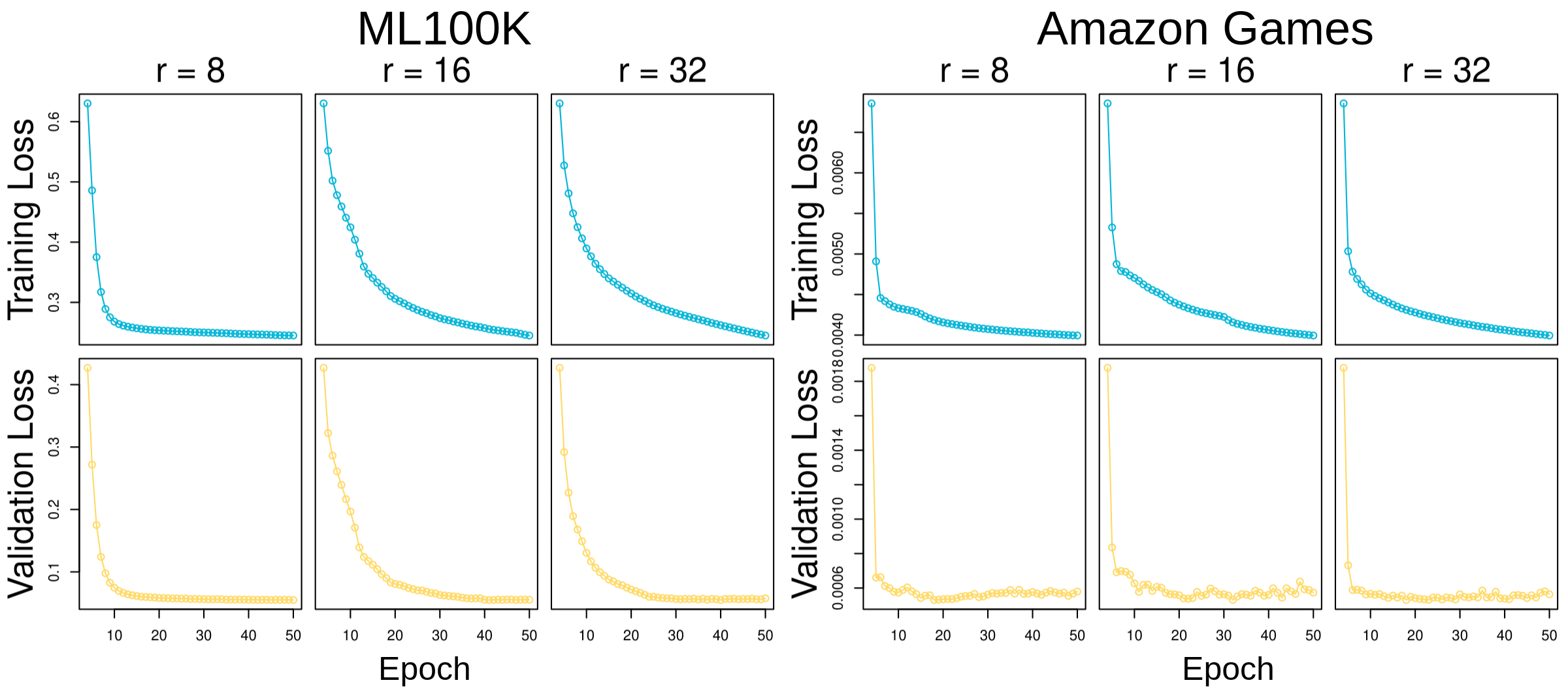}
 		\caption{Analysis of the convergence of Conv4Rec with varying sizes of the autoencoder bottleneck $r$ on the \textbf{ML100K} and \textbf{Amazon Games} datasets. The top graphs (in blue) show epochs versus training loss, while the bottom graphs (in yellow) display epochs versus validation loss.}
 		\label{fig:convergence}
 	\end{figure*}

 	\section{Empirical Convergence Analysis}
Conv4Rec is a deep neural network that employs an autoencoder architecture. A thorough theoretical analysis of the convergence of our model is out of the scope of this paper, and ultimately dependent on exiting theories regarding gradient descent and its commonly used extensions such as Adam, etc. However, one can verify experimentally that the method converges stably for many different hyper parameter configurations, which we do in this section. 
 	
Figure~\ref{fig:convergence} illustrates our analysis of the convergence of Conv4Rec with varying sizes of the autoencoder bottleneck $r$ on the \textbf{ML100K} and \textbf{Amazon Games} datasets. The top graphs (depicted in blue) represent epochs versus training loss, whereas the bottom graphs (shown in yellow) illustrate epochs versus validation loss. Our observations indicate that the algorithm exhibits favorable convergence properties, as evidenced by the stabilization of validation accuracy across both datasets. Furthermore, as expected, a smaller bottleneck dimension correlates with a more rapid convergence on the training set.

 \end{document}